\documentclass{article} 
\usepackage{iclr2021_conference,times}

\usepackage{amsmath,amsfonts,bm}

\def\figref#1{figure~\ref{#1}}
\def\Figref#1{Figure~\ref{#1}}

\def\Secref#1{Section~\ref{#1}}

\def\eqref#1{equation~\ref{#1}}

\def\1{\bm{1}}

\def\vs{{\bm{s}}}

\DeclareMathAlphabet{\mathsfit}{\encodingdefault}{\sfdefault}{m}{sl}
\SetMathAlphabet{\mathsfit}{bold}{\encodingdefault}{\sfdefault}{bx}{n}

\def\gL{{\mathcal{L}}}

\def\gT{{\mathcal{T}}}

\newcommand{\E}{\mathbb{E}}

\newcommand{\gradlossktj}{\frac{\partial L_{k+1}(\bW_{N:1}(k, t))}{\partial \bW_j(k, t)}}

\newcommand{\x}{\mathbf{x}}

\newcommand{\eg}{\emph{e.g.},\ }
\newcommand{\ie}{\emph{i.e.},\ }

\newcommand{\states}{\mathcal{S}}
\newcommand{\actions}{\mathcal{A}}

\newcommand{\bu}{\mathbf{u}}
\newcommand{\bv}{\mathbf{v}}
\newcommand{\be}{\mathbf{e}}
\newcommand{\by}{\mathbf{y}}
\newcommand{\bx}{\mathbf{x}}

\newcommand{\bw}{\mathbf{w}}

\newcommand{\bs}{\mathbf{s}}
\newcommand{\ba}{\mathbf{a}}
\newcommand{\bM}{\mathbf{M}}

\newcommand{\bG}{\mathbf{G}}

\newcommand{\bI}{\mathbf{I}}
\newcommand{\bg}{\mathbf{g}}
\newcommand{\bV}{\mathbf{V}}

\newcommand{\bR}{\mathbf{R}}
\newcommand{\bQ}{\mathbf{Q}}
\newcommand{\bA}{\mathbf{A}}
\newcommand{\bN}{\mathbf{N}}
\newcommand{\bS}{\mathbf{S}}
\newcommand{\bW}{\mathbf{W}}
\newcommand{\bU}{\mathbf{U}}
\newcommand{\bO}{\mathbf{O}}

\newcommand{\srank}{\text{srank}}
\newcommand{\rank}{\text{rank}}
\newcommand{\deepnet}{\bW_N(k, t) \bW_\phi(k, t)}

\newcommand{\stateactioni}{[\bs_i; \ba_i]}

\def\vs{\emph{vs}.\ }
\def\etc{\emph{etc.}\ }

\usepackage{hyperref}
\hypersetup{colorlinks=true,linkcolor=[RGB]{21, 100, 220},citecolor=brown}

\usepackage{url}
\usepackage{amsmath}
\usepackage{amssymb}
\usepackage{amsthm}
\usepackage{wrapfig}
\usepackage{graphicx}
\usepackage{algorithm}
\usepackage{algorithmic}
\usepackage{bbm}
\usepackage[skins,theorems]{tcolorbox}
\usepackage{wrapfig,lipsum,booktabs}
\usepackage{subcaption}
\usepackage[font=small]{caption}
\newtheorem{theorem}{Theorem}[section]
\newtheorem{assumption}{Assumption}[section]
\newtheorem{proposition}{Proposition}[section]
\newtheorem{definition}{Definition}
\newtheorem{corollary}{Corollary}[theorem]
\newtheorem{lemma}{Lemma}[theorem]

\title{Implicit Under-Parameterization Inhibits \\ Data-Efficient Deep Reinforcement Learning}

\author{Aviral Kumar\thanks{Equal Contribution. Correspondence to Aviral Kumar $<\texttt{aviralk@berkeley.edu}>$ and Rishabh Agarwal $<\texttt{rishabhagarwal@google.com}>$.}$^{\ \ 1, 2}$, Rishabh Agarwal$^{*\ 2, 3}$, Dibya Ghosh$^{1}$, Sergey Levine$^{1, 2}$\\
$^1$UC Berkeley, $^2$Google Research, $^3$MILA, Universit\'e de Montr\'eal  \\
}

\iclrfinalcopy 
\begin{document}

\maketitle

\begin{abstract}
We identify an implicit under-parameterization phenomenon in value-based deep RL methods that use bootstrapping: when value functions, approximated using deep neural networks, are trained with gradient descent using iterated regression onto target values generated by previous instances of the value network, more gradient updates decrease the expressivity of the current value network.
We characterize this loss of expressivity via a drop in the rank of the learned value network features, and show that this typically corresponds to a performance drop. 
We demonstrate this phenomenon on Atari and Gym benchmarks, in both offline and online RL settings.
We formally analyze this phenomenon and show that it results from a pathological interaction between bootstrapping and gradient-based optimization. We further show that mitigating implicit under-parameterization by
controlling rank collapse can improve performance.
\end{abstract}

\section{Introduction}
\vspace{-0.2cm}
\label{sec:intro}
Many commonly used deep reinforcement learning~(RL) algorithms estimate value functions using bootstrapping, which corresponds to sequentially fitting value functions to target value estimates generated from the value function learned in the previous iteration. Despite high-profile achievements~\citep{silver2017mastering}, these algorithms are highly unreliable due to poorly understood optimization issues. Although a number of hypotheses have been proposed to explain these issues~\citep{Achiam2019TowardsCD,bengio2020interference, fu19diagnosing, igl2020impact, martha2018sparse, kumar2020discor}, a complete understanding remains elusive.

We identify an ``implicit under-parameterization'' phenomenon that emerges when value networks are trained using gradient descent combined with bootstrapping.  
This phenomenon manifests as an excessive aliasing of features learned by the value network across states, which is exacerbated with more gradient updates. While the supervised deep learning literature suggests that some feature aliasing is desirable for generalization~\citep[\eg][]{gunasekar2017implicit, arora2019implicit}, implicit under-parameterization results in more pronounced aliasing than in supervised learning. 
This over-aliasing causes an otherwise expressive value network to \emph{implicitly} behave as an \emph{under-parameterized} network, often resulting in poor performance. 

Implicit under-parameterization becomes aggravated when the rate of data re-use is increased, restricting the sample efficiency of deep RL methods. In online RL,
increasing the number of gradient steps in between data collection steps for \textit{data-efficient} RL~\citep{fu19diagnosing, fedus2020revisiting} causes the problem to emerge more frequently. In the extreme case when no additional data is collected, referred to as \textit{offline} RL~\citep{lange2012batch, agarwal2020optimistic, levine2020offline}, implicit under-parameterization manifests consistently, limiting the viability of offline methods. 

We demonstrate the existence of implicit under-parameterization in common value-based deep RL methods, including Q-learning~\citep{Mnih2015, hessel2018rainbow} and actor-critic~\citep{Haarnoja18}, as well as neural fitted-Q iteration~\citep{Riedmiller2005, Ernst05}. To isolate the issue, we study 
the effective rank of the features in the penultimate layer of the value network~(\Secref{sec:problem}).
We observe that after an initial learning period, the rank of the learned features drops steeply.
As the rank decreases, the ability of the features to fit subsequent target values and the optimal value function generally deteriorates and results in a sharp decrease in performance~(\Secref{sec:understanding_iup}).

To better understand the emergence of implicit under-parameterization, we formally study the dynamics of Q-learning under two distinct models of neural net behavior~(\Secref{sec:theory}): kernel regression~\citep{ntk,mobahi2020self} and deep linear networks~\citep{arora2018optimization}. We corroborate the existence of this phenomenon in both models, and show that implicit under-parameterization stems from a pathological interaction between bootstrapping and the implicit regularization of gradient descent. Since value networks are trained to regress towards targets generated by a previous version of the same model, this leads to a sequence of value networks of potentially decreasing expressivity, which can result in degenerate behavior and a drop in performance.

\begin{figure}
    \centering
    \vspace{-5pt}
    \includegraphics[width=0.93\textwidth]{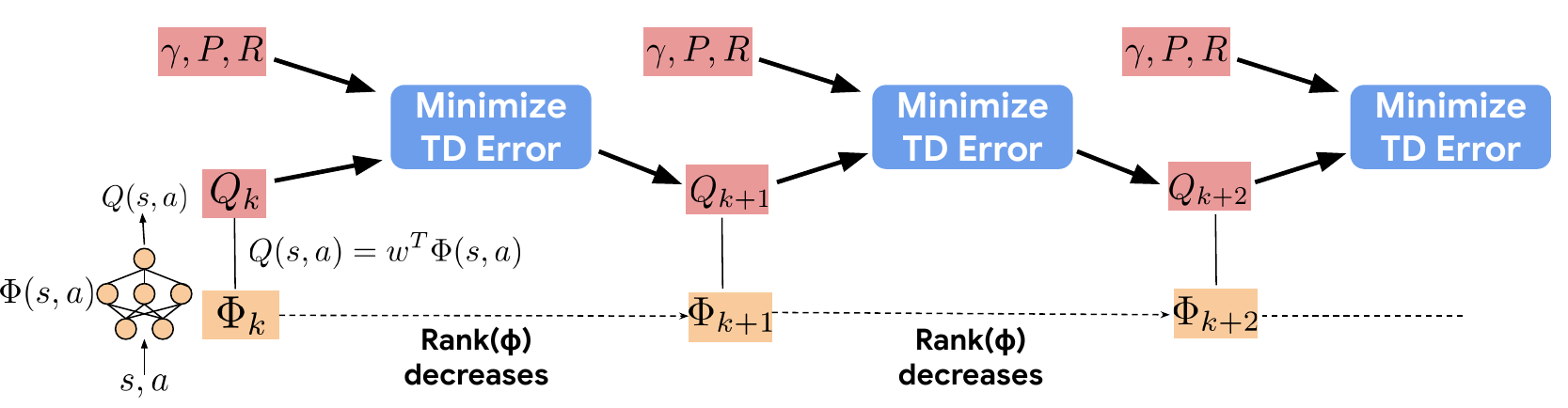}
    \vspace{-5pt}
    \caption{\textbf{Implicit under-parameterization.} Schematic diagram depicting the emergence of an \textit{effective rank} collapse in deep Q-learning. Minimizing TD errors using gradient descent with deep neural network Q-function leads to a collapse in the effective rank of the learned features $\Phi$, which is exacerbated with further training.}
    \label{fig:teaser}
    \vspace{-15pt}
\end{figure}

The main contribution of this work is the identification of implicit under-parameterization in deep RL methods that use bootstrapping. Empirically, we demonstrate a collapse in the rank of the learned features during training, and show it typically corresponds to a drop in performance in the Atari~\citep{bellemare2013ale} and continuous control Gym~\citep{gym} benchmarks in both the offline and data-efficient online RL settings. We verify the emergence of this phenomenon theoretically and characterize settings where implicit under-parameterization can emerge. We then show that mitigating this phenomenon via a simple penalty on the singular values of the learned features improves performance of value-based RL methods in the offline setting on Atari.

\vspace{-8pt}
\section{Preliminaries}
\label{sec:background}
\vspace{-8pt}
The goal in RL is to maximize long-term discounted reward in a Markov decision process (MDP), defined as a tuple $(\states, \actions, R, P, \gamma)$ \citep{puterman1994markov}, with state space $\states$, action space $\actions$, a reward function $R(\bs, \ba)$, transition dynamics $P(\bs' | \bs, \ba)$ and a discount factor $\gamma \in [0, 1)$. 
The Q-function $Q^\pi(\bs, \ba)$ for a policy $\pi(\ba|\bs)$, is the expected long-term discounted reward obtained by executing action $\ba$ at state $\bs$ and following $\pi(\ba|\bs)$ thereafter, 
$Q^\pi(\bs, \ba) := \E\left[ \sum\nolimits_{t=0}^\infty \gamma^t R(\bs_t, \ba_t) \right]$.
$Q^\pi(\bs, \ba)$  is the fixed point of the Bellman operator $\gT^{\pi}$, \mbox{$\forall \bs, \ba$: $\gT^{\pi} Q(\bs, \ba) := 
R(\bs, \ba) + \gamma \E_{\bs' \sim P(\cdot|\bs, \ba), \ba' \sim \pi(\cdot|\bs')} \left[ Q(\bs', \ba')\right]$}, which can be written in vector form as: $\bQ^\pi = \bR + \gamma P^\pi \bQ^\pi$.
The optimal Q-function, $Q^*(\bs, \ba)$,  is the fixed point of the Bellman optimality operator $\gT$ : $\gT Q(\bs, \ba) := R(\bs, \ba) + \gamma \E_{\bs' \sim P(\cdot|\bs, \ba)} \left[\max_{\ba'} Q(\bs', \ba')\right]$. 

Practical Q-learning methods~\citep[\eg][]{Mnih2015,hessel2018rainbow,Haarnoja18} convert the Bellman equation into a bootstrapping-based objective for training a Q-network, $Q_\theta$, via gradient descent. This objective, known as mean-squared temporal difference~(TD) error, is given by: $L(\theta) = \sum_{\bs, \ba} \left(R(\bs, \ba) + \gamma \bar{Q}_\theta(\bs', \ba') - Q(\bs, \ba) \right)^2$, where $\bar{Q}_\theta$ is a delayed copy of the Q-function, typically referred to as the \emph{target network}. These methods train Q-networks via gradient descent and slowly update the target network via Polyak averaging on its parameters. We refer the output of the penultimate layer of the deep Q-network as the learned \emph{feature matrix} $\Phi$, such that $Q(\bs, \ba) = \bw^T \Phi(\bs, \ba)$, 
where $\bw \in \mathbb{R}^{d}$ and $\Phi \in \mathbb{R}^{|\states| |\actions| \times d}$. 

\begin{wrapfigure}{r}{0.4\textwidth}
\begin{small}
\vspace{-38pt}
\begin{minipage}[t]{0.99\linewidth}
\begin{algorithm}[H]
\small
\caption{\textbf{Fitted Q-Iteration (FQI)}}
\label{alg:fqi}
\begin{algorithmic}[1]
\small{
    \STATE Initialize Q-network $\bQ_\theta$, buffer $\mu$.
    \FOR{fitting iteration $k$ in \{1, \dots, N\}}
        \STATE Compute $\bQ_\theta(\bs,\ba)$ and target values
        \mbox{${y}_k(\bs, \ba) = r + \gamma \max_{\ba'} \bQ_{k-1}(\bs', \ba')$}\\ 
        on $\{(\bs, \ba)\} \sim \mu$ for training
        \STATE Minimize TD error for $\bQ_\theta$ via $t = 1, \cdots,$ T gradient descent updates,\\
        \mbox{$\min_{\theta}~(Q_\theta(\bs,\ba) - \by_k)^2$}
    \ENDFOR}
\end{algorithmic}
\end{algorithm}
\end{minipage}
\vspace{-15pt}
\end{small}
\end{wrapfigure}
For simplicity of analysis, we abstract  deep Q-learning methods into a generic fitted Q-iteration (\textbf{FQI}) framework~\citep{Ernst05}. We refer to FQI with neural nets as neural FQI~\citep{Riedmiller2005}. In the $k$-th \emph{fitting} iteration, FQI trains the Q-function, $\bQ_k$, to match the target values, $\by_k = \bR + \gamma P^\pi \bQ_{k-1}$ generated using previous Q-function, $\bQ_{k-1}$ (Algorithm~\ref{alg:fqi}).  Practical methods can be instantiated as variants of FQI, with different target update styles, different optimizers, etc.

\vspace{-5pt}
\section{Implicit Under-Parameterization in Deep Q-Learning}
\label{sec:problem}
\vspace{-7pt}

In this section, we empirically demonstrate the existence of implicit under-parameterization in deep RL methods that use bootstrapping. We characterize implicit under-parameterization in terms of the \emph{effective rank}~\citep{yang2019harnessing} of the features learned by a Q-network.
The effective rank of the feature matrix $\Phi$, for a threshold $\delta$ (we choose $\delta = 0.01$), denoted as $\srank_\delta(\Phi)$, is given by \mbox{$\srank_{\delta}(\Phi) = \min \left\{ k: \frac{\sum_{i=1}^k \sigma_{i} (\Phi)}{\sum_{i=1}^{d} \sigma_i(\Phi)} \geq 1 - \delta \right\}$},
where $\{\sigma_i(\Phi)\}$ are the singular values of $\Phi$ in decreasing order, i.e., $\sigma_1 \geq \cdots \geq \sigma_{d} \geq 0$. Intuitively, $\srank_{\delta}(\Phi)$ represents the number of ``effective'' unique components of the feature matrix $\Phi$ that form the basis for linearly approximating the Q-values. 
When the network maps different states to orthogonal feature vectors, then $\srank_{\delta}(\Phi)$ has high values close to $d$. When the network ``aliases'' state-action pairs by mapping them to a smaller subspace,
$\Phi$ has only a few active singular directions, and $\srank_{\delta}(\Phi)$ takes on a small value. 

\vspace{-1pt}
\begin{definition}[]
\textnormal{Implicit under-parameterization} refers to a reduction in the effective rank of the features, $\srank_\delta(\Phi)$, that occurs implicitly as a by-product of learning deep neural network Q-functions.  
\end{definition}
\vspace{-6pt}

While rank decrease also occurs in supervised learning, it is usually beneficial for obtaining generalizable solutions~\citep{gunasekar2017implicit,arora2019implicit}. However, we will show that in deep Q-learning, an interaction between bootstrapping and gradient descent can lead to more aggressive rank reduction~(or rank collapse), which can hurt performance.

\begin{figure}[t]
\centering
\vspace{-10pt}
    \includegraphics[width=0.95\linewidth]{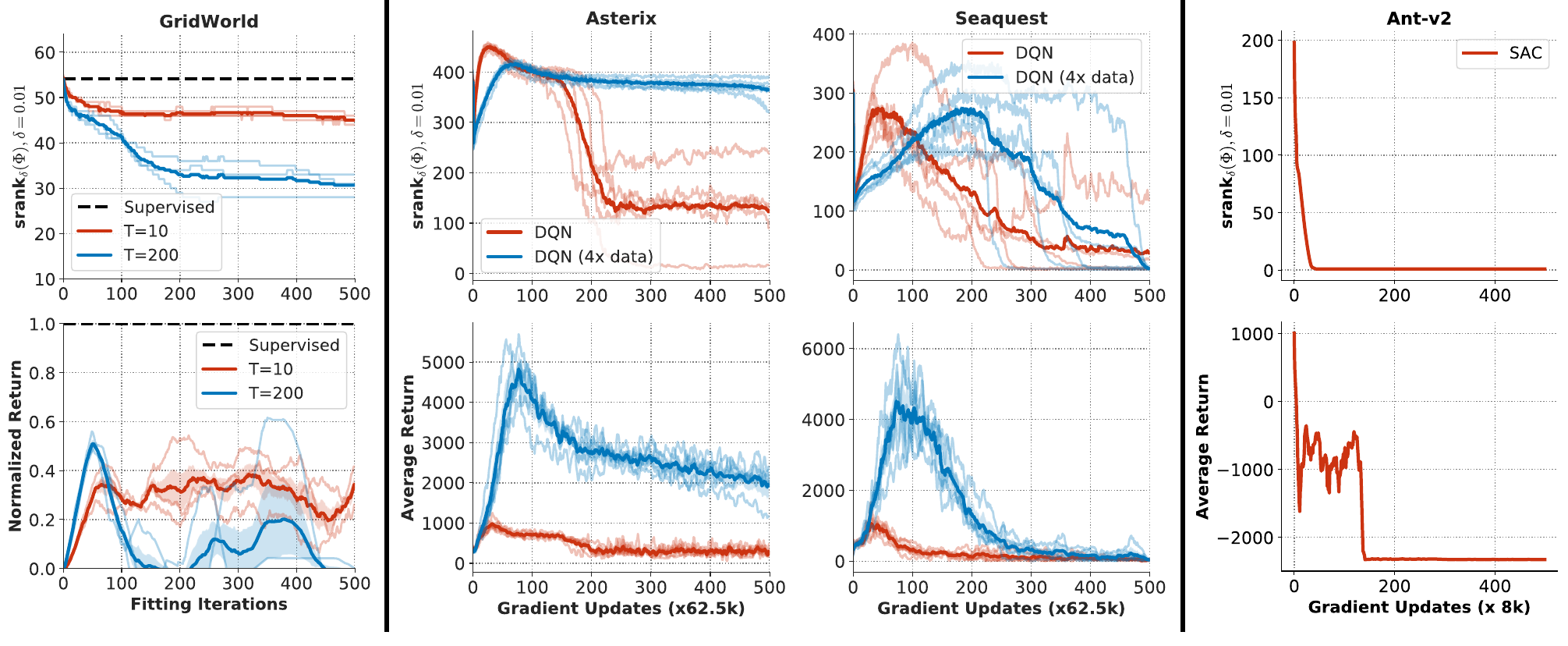}
    \vspace{-6pt}
    \caption{\textbf{Offline RL}. {$\srank_\delta(\Phi)$ and performance of neural FQI on gridworld, DQN on Atari and SAC on Gym environments in the offline RL setting. Note that low rank (top row) generally corresponds to worse policy performance (bottom row). Rank collapse is worse with more gradient steps per fitting iteration (T$=10$ \vs $200$ on gridworld). Even when a larger, high coverage dataset is used, marked as DQN (4x data), rank collapse occurs (for Asterix also see Figure~\ref{fig:offline_problem_app_20k} for a complete figure with a larger number of gradient updates).}}
    \label{fig:offline_problem}
    \vspace{-15pt}
\end{figure}

\vspace{-1pt}
\textbf{Experimental setup.} To study implicit under-parameterization empirically, we compute $\srank_\delta(\Phi)$ on a minibatch of state-action pairs sampled i.i.d. from the training data (\ie the dataset in the offline setting, and the replay buffer in the online setting). We investigate offline and online RL settings on benchmarks including Atari games~\citep{bellemare2013ale} and Gym environments~\citep{gym}. We also utilize gridworlds described by \citet{fu19diagnosing} to compare the learned Q-function against the oracle solution computed using tabular value iteration. We evaluate DQN~\citep{Mnih2015} on gridworld and Atari and SAC~\citep{Haarnoja18} on Gym domains.


\vspace{-1pt}
\textbf{Data-efficient offline RL.}~ In offline RL, our goal is to learn effective policies by performing Q-learning on a fixed dataset of transitions. 
We investigate the presence of rank collapse when deep Q-learning is used with broad state coverage offline datasets from \citet{agarwal2020optimistic}.
In the top row of Figure~\ref{fig:offline_problem}, we show that after an initial learning period, $\srank_\delta(\Phi)$ decreases in all domains~(Atari, Gym and the gridworld). The final value of $\srank_\delta(\Phi)$ is often quite small -- \eg in Atari, only 20-100 singular components are active for $512$-dimensional features, implying significant underutilization of network capacity. 
Since under-parameterization is \emph{implicitly} induced by the learning process, even high-capacity value networks behave as low-capacity networks as more training is performed with a bootstrapped objective~(\eg mean squared TD error).

On the gridworld environment, regressing to $Q^*$ using supervised regression results in a much higher $\srank_\delta(\Phi)$~(black dashed line in Figure~\ref{fig:offline_problem}(left)) than when using neural FQI. On Atari, even when a {4x} larger offline dataset with much broader coverage is used (blue line in Figure~\ref{fig:offline_problem}), rank collapse still persists, indicating that implicit under-parameterization is not due to limited offline dataset size. Figure~\ref{fig:offline_problem}~($2^\mathrm{nd}$ row) illustrates that policy performance generally deteriorates as $\srank(\Phi)$ drops, and eventually collapses simultaneously with the rank collapse. {While we do not claim that implicit under-parameterization is the only issue in deep Q-learning, the results in Figure~\ref{fig:offline_problem} show that the emergence of this under-parameterization is \textit{strongly} associated with poor performance.}  

To prevent confounding effects from the distribution mismatch between the learned policy and the offline dataset, which often affects the performance of Q-learning methods, we also study CQL~\citep{kumar2020conservative}, an offline RL algorithm designed to handle distribution mismatch. 
We find a similar degradation in effective rank and performance for CQL~(\Figref{fig:offline_problem_cql_app}), implying that under-parameterization does not stem from distribution mismatch and arises even when the resulting policy is within the behavior distribution (though the policy may not be exactly pick actions observed in the dataset). We provide more evidence in  Atari and Gym domains in Appendix~\ref{more_evidence_offline}.


\begin{figure*}[h]
\centering
\vspace{-9pt}
    \includegraphics[width=0.93\linewidth]{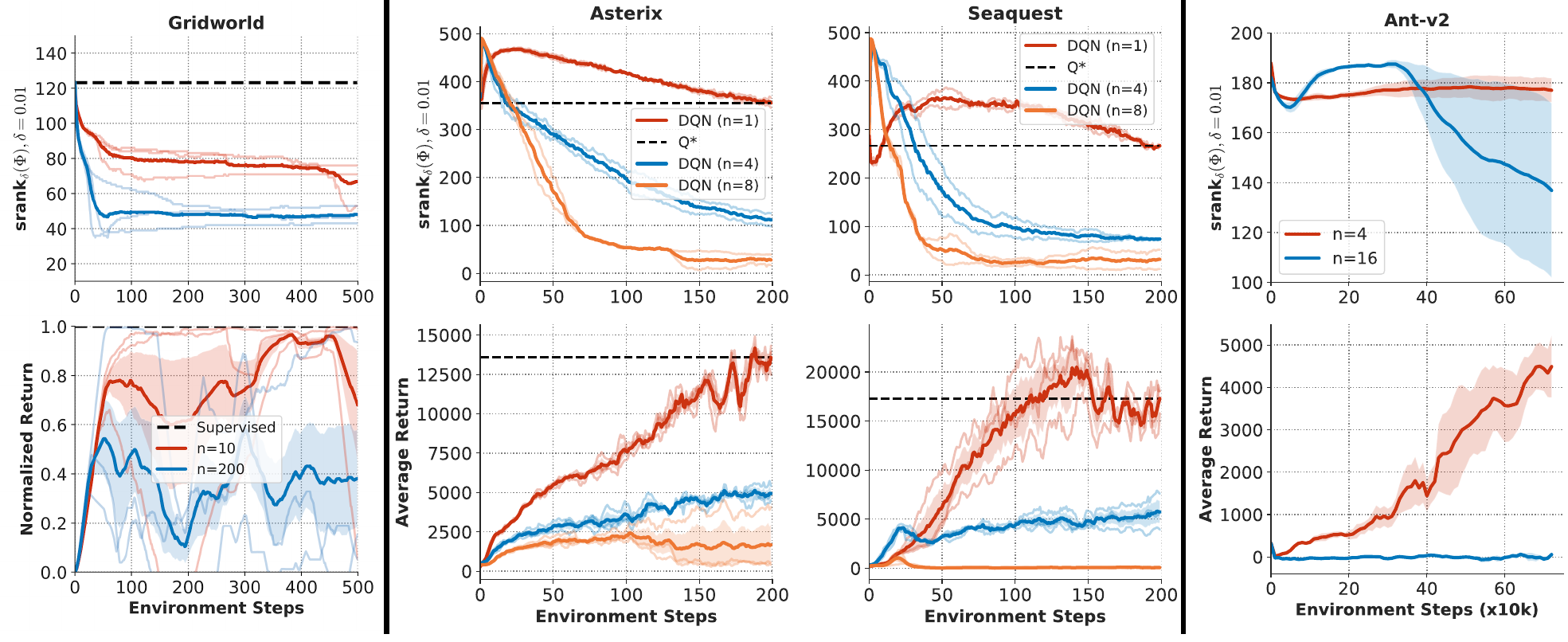}
    \vspace{-7pt}
    \caption{\textbf{Data Efficient Online RL}.
    $\srank_\delta(\Phi)$ and performance of neural FQI on gridworld, DQN on Atari and SAC on Gym domains in the online RL setting, with varying numbers of gradient steps per environment step ($n$). Rank collapse happens earlier with more gradient steps, and the corresponding performance is poor.}
    \label{fig:online_problem}
    \vspace{-6pt}
\end{figure*}

\vspace{-1pt}
\textbf{Data-efficient online RL.}\label{sec:data_efficient_rl}
Deep Q-learning methods typically use very few gradient updates~($n$) per environment step~(\eg DQN takes 1 update every 4 steps on Atari, $n=0.25$). Improving the sample efficiency of these methods requires increasing $n$ to utilize the replay data more effectively. However, we find that using larger values of $n$ results in higher levels of rank collapse as well as performance degradation.
In the top row of Figure~\ref{fig:online_problem}, we show that larger values of $n$ lead to a more aggressive drop in $\srank_\delta(\Phi)$ (red \vs blue/orange lines), and that rank continues to decrease with more training. Furthermore, the bottom row 
illustrates that larger values of $n$ result in worse performance, corroborating \citet{fu19diagnosing,fedus2020revisiting}. We find similar results with the Rainbow algorithm~\citep{hessel2018rainbow}~(Appendix~\ref{app:data_efficient_online_rl}).
As in the offline setting, directly regressing to $Q^*$ via supervised learning does not cause rank collapse (black line in \Figref{fig:online_problem}).

\vspace{-5pt}
\subsection{Understanding Implicit Under-parameterization and its Implications}\label{sec:understanding_iup}
\vspace{-5pt}


\begin{figure}[t]
\centering
    \begin{subfigure}{0.235\textwidth}
        \centering
        \vspace{-5pt}
        \includegraphics[width=\textwidth]{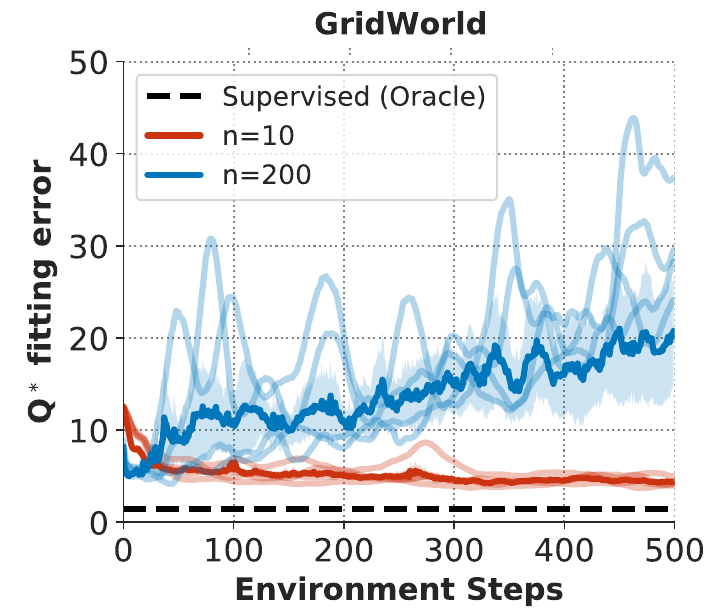}
        \vspace{-11pt}
        \caption{\textsc{Q$^*$ Fitting Error}}\label{fig:expressivity}
    \end{subfigure}
    \begin{subfigure}{0.24\textwidth}
        \centering
        \includegraphics[width=\textwidth]{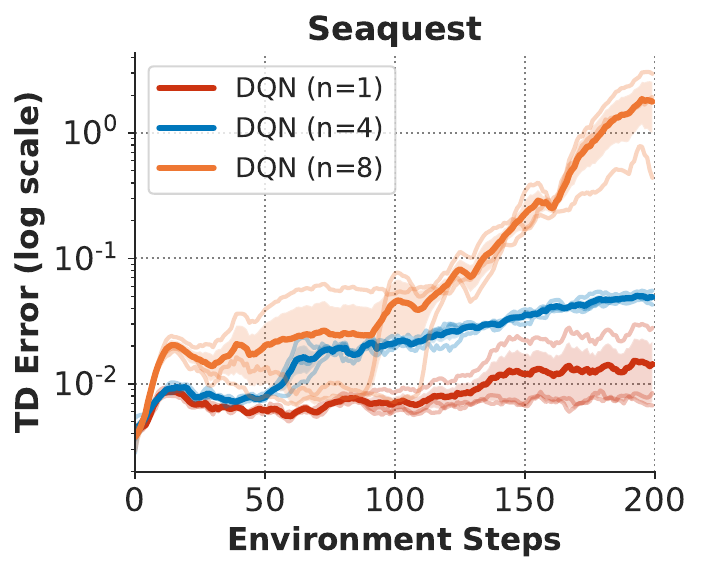}
        \vspace{-11pt}
        \caption{ \textsc{TD Error}}\label{fig:expressivity_Td}
    \end{subfigure}
    \begin{subfigure}{0.24\textwidth}
        \centering
        \includegraphics[width=\textwidth]{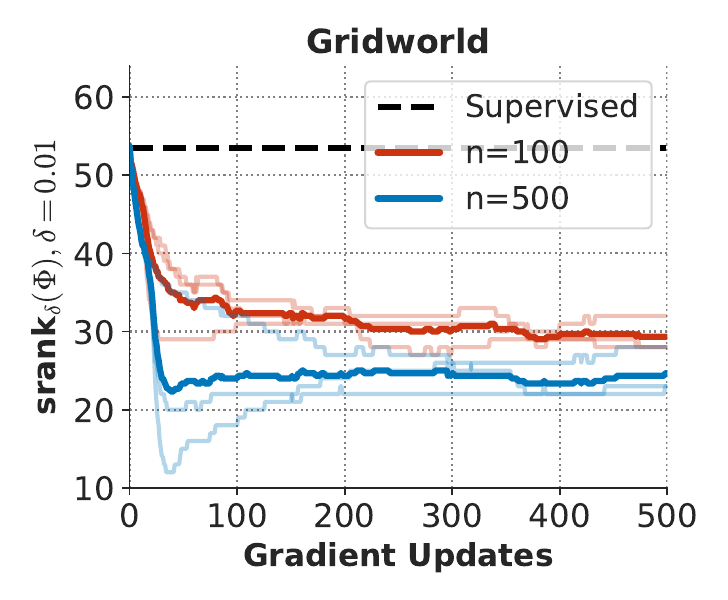}
        \caption{\textsc{Q-Reinitialization}}\label{fig:reinit}
    \end{subfigure}
    \begin{subfigure}{0.24\textwidth}
        \centering
        \includegraphics[width=\textwidth]{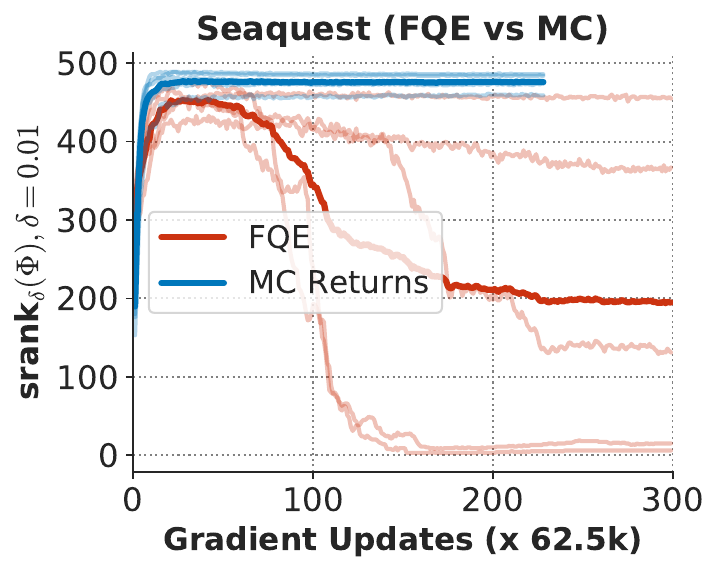}
        \caption{\textsc{FQE} \emph{vs}. \textsc{MC}}\label{fig:ablations}
    \end{subfigure}
    \vspace{-5pt}
    \caption{\textbf{(a)} Fitting error for $Q^*$ prediction for $n\!=\!10$ vs $n\!=\!200$ 
    steps in \Figref{fig:online_problem} (left). Observe that rank collapse inhibits fitting $Q^*$ as the fitting error rises over training while rank collapses. \textbf{(b)} TD error for varying values of $n$ for \textsc{Seaquest} in \Figref{fig:online_problem} (middle). TD error increases with rank degradation. \textbf{(c)} Q-network re-initialization in each fitting iteration on gridworld. \textbf{(d)} Trend of $\srank_\delta(\Phi)$ for policy evaluation based on bootstrapped updates~(FQE) vs Monte-Carlo returns (no bootstrapping). Note that rank-collapse still persists with reinitialization and FQE, but goes away in the absence of bootstrapping.}
    \vspace{-14pt}
    \label{fig:expressivity_all}
\end{figure}

\textbf{How does implicit under-parameterization degrade performance?} 
Having established the presence of rank collapse in data-efficient RL, we now discuss how it can adversely affect performance. As the effective rank of the network features $\Phi$ decreases, so does the network's ability to fit the subsequent target values, and eventually results in inability to fit $Q^*$. 
In the gridworld domain, we measure this loss of expressivity by measuring the error in fitting oracle-computed $Q^*$ values via a linear transformation of $\Phi$. When rank collapse occurs, the error in fitting $Q^*$ steadily increases during training, and the consequent network is not able to predict $Q^*$ at all by the end of training~(\Figref{fig:expressivity}) -- this entails a drop in performance. In Atari domains, we do not have access to $Q^*$, and so we instead measure TD error, that is, the error in fitting the target value estimates, $\bR + \gamma P^\pi \bQ_{k}$. In \textsc{Seaquest}, as rank decreases, the TD error increases~(\Figref{fig:expressivity_Td}) and the value function is unable to fit the target values, culminating in a performance plateau~(\Figref{fig:online_problem}). This observation is consistent across other environments; we present further supporting evidence in Appendix~\ref{app:regression}.


\textbf{Does bootstrapping cause implicit under-parameterization?} 
We perform a number of controlled experiments in the gridworld and Atari environments to isolate the connection between rank collapse and bootstrapping. We first remove confounding issues of poor network initialization \citep{fedus2020catastrophic} and non-stationarity \citep{igl2020impact} by showing that rank collapse occurs even when the Q-network is re-initialized from scratch at the start of each fitting iteration (Figure \ref{fig:reinit}). To show that the problem is not isolated to the control setting, we show evidence of rank collapse in the policy evaluation setting as well. We trained a value network using fitted Q-evaluation for a fixed policy $\pi$ (\ie using the Bellman operator $\gT^{\pi}$ instead of $\gT$), and found that rank drop still occurs~(FQE in \Figref{fig:ablations}). Finally, we show that by removing bootstrapped updates and instead regressing directly to Monte-Carlo (MC) estimates of the value, the effective rank \textit{does not} collapse (MC Returns in \Figref{fig:ablations}). These results, along with similar findings on other Atari environments~(Appendix~\ref{app:evidence_bootstrap}), our analysis indicates that bootstrapping is at the core of implicit under-parameterization.

\vspace{-7pt}
\section{Theoretical Analysis of Implicit Under-Parameterization}
\vspace{-9pt}
\label{sec:theory}
In this section, we formally analyze implicit under-parameterization and prove that training neural networks with bootstrapping reduces the effective rank of the Q-network, corroborating the empirical observations in the previous section. We focus on policy evaluation (Figure~\ref{fig:ablations} and Figure~\ref{fig:offline_policy_eval_5_games}), where we aim to learn a Q-function that satisfies $\bQ = \bR + \gamma P^\pi \bQ$ for a fixed $\pi$, for ease of analysis.
We also presume a fixed dataset of transitions, $\mathcal{D}$, to learn the Q-function.

\vspace{-7pt}
\subsection{Analysis via Kernel Regression}
\label{sec:self_distill}
\vspace{-6pt}
We first study bootstrapping with neural networks through a mathematical abstraction that treats the Q-network as a kernel machine, following the neural tangent kernel (NTK) formalism~\citep{ntk}. Building on prior analysis of self-distillation~\citep{mobahi2020self}, we assume that each iteration of bootstrapping, the Q-function optimizes the squared TD error to target labels $\by_k$ with a kernel regularizer. This regularizer captures the inductive bias from gradient-based optimization of TD error and resembles the regularization imposed by gradient descent under NTK~\citep{mobahi2020self}. The error is computed on $(\bs_i, \ba_i) \in \mathcal{D}$ whereas the regularization imposed by a universal kernel $u$ with a coefficient of $c \geq 0$ is applied to the Q-values at \emph{all} state-action pairs as shown in Equation~\ref{eqn:self_dist_problem}. We consider a setting $c > 0$ for all rounds of bootstrapping, which corresponds to the solution obtained by performing gradient descent on TD error for a small number of iterations with early stopping in each round~\citep{suggala2018connecting} and thus, resembles how the updates in Algorithm~\ref{alg:fqi} are typically implemented in practice.
\begin{equation}
\label{eqn:self_dist_problem}
    \small{\bQ_{k+1} \leftarrow \arg \min_{\bQ \in \mathcal{Q}} \sum_{\bs_i, \ba_i \in \mathcal{D}} \left(Q(\bs_i, \ba_i) - y_{k}(\bs_i, \ba_i) \right)^2 + c \sum_{(\bs, \ba)} \sum_{(\bs', \ba')} u\mathbf{(}(\bs, \ba), (\bs', \ba')\mathbf{)} Q(\bs, \ba) Q(\bs', \ba').}
\end{equation}
The solution to Equation~\ref{eqn:self_dist_problem} can be expressed as $Q_{k+1}(\bs, \ba) = \bg^{T}_{(\bs, \ba)} (c \bI + \bG)^{-1} \by_k$,
where $\bG$ is the Gram matrix for a special positive-definite kernel~\citep{duffy2015green} 
and $\bg_{(\bs, \ba)}$ denotes the row of $\bG$ corresponding to the input $(\bs, \ba)$~\citep[Proposition~1]{mobahi2020self}. A detailed proof is in Appendix~\ref{app:self_distill_proofs}. When combined with the fitted Q-iteration recursion, setting labels $\by_k = \bR + \gamma P^\pi \bQ_{k-1}$, we recover a recurrence that relates subsequent value function iterates
\vspace{-0.21in}

\begin{align}
\resizebox{.945\textwidth}{!}{$
    \bQ_{k+1}  = \bG(c \bI + \bG)^{-1} \by_k = \underbrace{ \bG(c \bI + \bG)^{-1}}_{\bA} \left[\bR + \gamma P^\pi \bQ_k\right] = \bA \left({\sum_{i=1}^{k} \gamma^{k-i} \left(P^\pi \bA \right)^{k-i}} \right) \bR := \bA \bM_k \bR.  
    \label{eqn:unfolding}
$}
\end{align}
\vspace{-0.2in}

Equation~\ref{eqn:unfolding} follows from unrolling the recurrence and setting the algorithm-agnostic initial Q-value, $\bQ_0$, to be $\mathbf{0}$. 
We now show that the sparsity of singular values of the matrix $\bM_k$ generally increases over fitting iterations, implying that the effective rank of $\bM_k$ diminishes with more iterations. 
{For this result, we assume that the matrix $\bS$ is normal, \ie the norm of the (complex) eigenvalues of $\bS$ is equal to its singular values. We will discuss how this assumption can be relaxed in Appendix~\ref{app:eigval_srank}.}   
\begin{tcolorbox}[colback=blue!6!white,colframe=black,boxsep=0pt,top=3pt,bottom=5pt]
\begin{theorem}
\label{thm:self_distillation}
Let $\bS$ be a shorthand for $\bS = \gamma P^\pi \bA$ and assume $\bS$ is a normal matrix. Then there exists an infinite, strictly increasing sequence of fitting iterations, $(k_l)_{l=1}^\infty$ starting from $k_1 = 0$,
such that, for any two singular-values $\sigma_i(\bS)$ and $\sigma_j(\bS)$ of $\bS$ with $\sigma_i(\bS) < \sigma_j(\bS)$, 
\vspace{-0.03in}
\begin{equation}
    {\forall~ l \in \mathbb{N}~~\text{and}~~l' \geq l,~~ \frac{\sigma_i(\bM_{k_{l^\prime}})}{\sigma_j(\bM_{k_{l^\prime}})} < \frac{\sigma_i(\bM_{k_l})}{\sigma_j(\bM_{k_l})} \leq \frac{\sigma_i(\bS)}{\sigma_j(\bS)}}.
\vspace{-0.03in}
\end{equation}
Hence, $\srank_{\delta}(\bM_{k_{l^\prime}}) \leq \srank_\delta(\bM_{k_l})$.
Moreover, if $\bS$ is positive semi-definite, then $(k_l)_{l=1}^\infty = \mathbb{N}$, \emph{\ie} $\srank$ continuously decreases in each fitting iteration.
\end{theorem}
\end{tcolorbox}
    

We provide a proof of the theorem above as well as present a stronger variant that shows a gradual decrease in the effective rank for fitting iterations outside this infinite sequence in Appendix \ref{app:self_distill_proofs}. 
As $k$ increases along the sequence of iterations given by $k=(k_l)_{l=1}^\infty$, the effective rank of the matrix $\bM_{k}$ drops, leading to low expressivity of this matrix. Since $\bM_{k}$ linearly maps rewards to the Q-function~(Equation~\ref{eqn:unfolding}), drop in expressivity results of $\bM_{k}$ in the inability to model the actual $\bQ^\pi$.

{\textbf{Summary of our analysis.} Our analysis of 
bootstrapping and gradient descent from the view of regularized kernel regression suggests that rank drop happens with more training (i.e., with more rounds of bootstrapping). In contrast to self-distillation~\citep{mobahi2020self}, rank drop may not happen in every iteration (and rank may increase between two consecutive iterations occasionally), 
but $\srank_{\delta}$ exhibits a generally decreasing trend.}

\vspace{-3pt}
\subsection{Analysis with Deep Linear Networks under Gradient Descent}
\label{sec:grad_descent}
\vspace{-6pt}
While Section~\ref{sec:self_distill} demonstrates rank collapse will occur in a kernel-regression model of Q-learning, it does not illustrate \textit{when} the rank collapse occurs. {To better specify a point in training when rank collapse emerges, we present a complementary derivation for the case when the Q-function is represented as a deep linear neural network~\citep{arora2019implicit}, which is a widely-studied setting for analyzing implicit regularization of gradient descent in supervised learning~\citep{gunasekar2017implicit,gunasekar2018implicit,arora2018optimization,arora2019implicit}}. Our analysis will show that rank collapse can emerge as the generated target values begin to approach the previous value estimate, in particular, when in the vicinity of the optimal Q-function.

{\textbf{Proof strategy.} Our proof consists of two steps: \textbf{(1)} We show that the effective rank of the feature matrix decreases within one fitting iteration (for a given target value) due to the low-rank affinity, \textbf{(2)} We show that this effective rank drop is ``compounded'' as we train using a bootstrapped objective. Proposition~\ref{thm:sing_val_evolution} explains \textbf{(1)} and Proposition~\ref{thm:compounding_main_paper}, Theorem~\ref{thm:near_optimality} and Appendix~\ref{app:rank_drop_compounds} discuss \textbf{(2)}.}

\textbf{Additional notation and assumptions.} We represent the Q-function as a deep linear network with at $\geq 3$  layers, such that $Q(\bs, \ba) = \bW_{N} \bW_\phi [\bs; \ba]$, where $N \geq 3$,  $\bW_{N} \in \mathbb{R}^{1 \times d_{N-1}}$ and
$\bW_\phi = \bW_{N-1}\bW_{N-2} \cdots \bW_1$ with \mbox{$\bW_i \in \mathbb{R}^{d_i \times d_{i-1}}$ for $i=1, \dots, N-1$}. 
$\bW_\phi$ maps an input $[\bs;\ba]$ to corresponding penultimate layer' features $\Phi(\bs, \ba)$. Let $\bW_j(k, t)$ denotes the weight matrix $\bW_j$ at the $t$-th step of gradient descent during the $k$-th fitting iteration (Algorithm~\ref{alg:fqi}). We define $\bW_{k, t} = \bW_N(k, t) \bW_\phi(k, t)$ and $L_{N, k+1}(\bW_{k, t})$ as the TD error objective in the $k$-th fitting iteration. We study $\srank_\delta(\bW_\phi(k, t))$ since the rank of features $\Phi = \bW_\phi(k, t) [\mathcal{S}, \mathcal{A}]$ is equal to rank of $\bW_\phi(k, t)$ provided the state-action inputs have high rank.

We assume that the evolution of the weights is governed by a continuous-time differential equation~\citep{arora2018optimization} \emph{within} each fitting iteration $k$. To simplify analysis, we also assume that all \textit{except the last-layer} weights follow a ``balancedness'' property (Equation~\ref{eqn:actual_balanced}), which suggests that the weight matrices in the consecutive layers in the deep linear network share the same singular values (but with different permutations). However, note that we do not assume balancedness for the last layer which trivially leads to rank-1 features, making our analysis strictly more general than conventionally studied deep linear networks. In this model, we can characterize the evolution of the singular values of the feature matrix $\bW_{\phi}(k, t)$, using techniques analogous to \citet{arora2019implicit}:
\begin{proposition}
\label{thm:sing_val_evolution}
 The singular values of the feature matrix $\bW_\phi(k, t)$ evolve according to:
 \begin{equation}
    \label{eqn:sing_val_diff_eqn}
     \dot{\sigma}_{r}(k, t) = -N \cdot \left( \sigma_r^2(k, t) \right)^{1 - \frac{1}{N-1}} ~\cdot~ \left\langle \bW_{N}(k, t)^T \frac{d L_{N, k+1}(\bW_{K, t})}{d \bW}, \bu_r(k, t) \bv_r(k, t)^T \right\rangle,
 \end{equation}
for $r = 1, \cdots, \min_{i=1}^{N-1} d_i$, where $\bu_r(k, t)$ and $\bv_r(k, t)$ denote the left and right singular vectors of the feature matrix, $\bW_\phi(k, t)$, respectively.
\end{proposition}

\begin{wrapfigure}{r}{0.25\textwidth}
\centering
\vspace{-5pt}
    \includegraphics[width=0.95\linewidth]{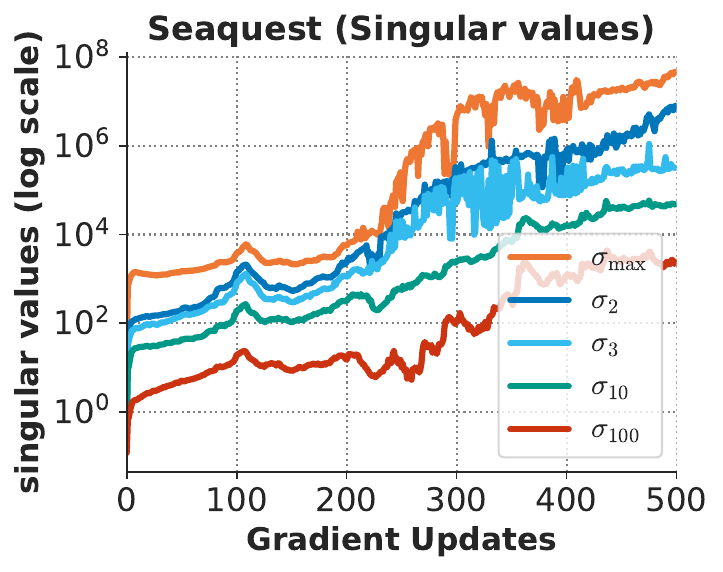}
    \vspace{-10pt}
    \caption*{\small{Evolution of singular values of $W_\phi$ on \textsc{Seaquest}}}
    \vspace{-20pt}
\end{wrapfigure}
Solving the differential equation~(\ref{eqn:sing_val_diff_eqn}) indicates that {larger singular values will evolve at a exponentially faster rate than smaller singular values (as we also formally show in Appendix~\ref{app:rank_decrease_sl}) and the difference in their magnitudes disproportionately increase with increasing $t$}. This behavior also occurs empirically, illustrated in the figure on the right~(also see \Figref{fig:sing_val_all}), where larger singular values are orders of magnitude larger than smaller singular values. Hence, the effective rank, $\srank_\delta(\bW_\phi(k, t))$, will decrease with more gradient steps \emph{within} a fitting iteration $k$. 

\textbf{Abstract optimization problem for the low-rank solution.} Building on Proposition~\ref{thm:sing_val_evolution}, we note that the final solution obtained in a bootstrapping round~(\ie fitting iteration) can be equivalently expressed as the solution that minimizes a weighted sum of the TD error and a data-dependent implicit regularizer $h_{\mathcal{D}}(\bW_\phi, \bW_N)$ that encourages disproportionate singular values of $\bW_\phi$, and hence, a low effective rank of $\bW_\phi$. While the actual form for $h$ is unknown, to facilitate our analysis of bootstrapping, we make a simplification and express this solution as the minimum of Equation~\ref{eqn:closed_form}.
\begin{align}
\label{eqn:closed_form}
\small{\min_{\bW_\phi, \bW_N \in \mathcal{M}} \vert\vert \bW_N \bW_\phi[\bs;\ba] - \by_k(\bs, \ba) \vert\vert^2 + {\lambda_k}  \srank_\delta(\bW_\phi) .}
\end{align}
Note that the entire optimization path
may not correspond to the objective in Equation~\ref{eqn:closed_form}, but the Equation~\ref{eqn:closed_form} represents the final solution of a given fitting iteration. $\mathcal{M}$ denotes the set of constraints that $\bW_N$ obtained via gradient optimization of TD error must satisfy, however we do not need to explicitly quantify $\mathcal{M}$ in our analysis. $\lambda_k$ is a constant that denotes the strength of rank regularization. Since $\srank_\delta$ is always regularized, our analysis assumes that $\lambda_k > 0$~(see Appendix~\ref{app:rank_decrease_sl}).

\textbf{Rank drop within a fitting iteration ``compounds'' due to bootstrapping.} In the RL setting, the target values are given by $y_k(\bs, \ba) = r(\bs, \ba) + \gamma P^\pi Q_{k-1}(\bs, \ba)$. First note that when $r(\bs, \ba) = 0$ and $P^\pi = \bI$, i.e., when the bootstrapping update resembles self-regression, we first note that just ``copying over weights'' from iteration $k-1$ to iteration $k$ is a feasible point for solving Equation~\ref{eqn:closed_form}, which attains \emph{zero} TD error with no increase in $\srank_\delta$. A better solution to Equation~\ref{eqn:closed_form} can thus be obtained by incurring non-zero TD error at the benefit of a decreased $\srank$, indicating that in this setting, $\srank_\delta(\bW_\phi)$ drops in each fitting iteration, leading to a compounding rank drop effect. 

We next extend this analysis to the full bootstrapping setting. Unlike the self-training setting, $y_k(\bs, \ba)$ is not directly expressible as a function of the previous $\bW_\phi(k, T)$ due to additional reward and dynamics transformations. 
Assuming closure of the function class~(Assumption~\ref{assumption:closed}) under the Bellman update~\citep{munos2008finite, chen2019information}, we reason about the compounding effect of rank drop across iterations in Proposition~\ref{thm:compounding_main_paper} (proof in Appendix~\ref{app:rank_drop_compounds}). Specifically, $\srank_\delta$ can increase in each fitting iteration due to $\bR$ and $P^\pi$ transformations, but will decrease due to low rank preference of gradient descent. This change in rank then compounds as shown below.
\begin{tcolorbox}[colback=blue!6!white,colframe=black,boxsep=0pt,top=3pt,bottom=5pt]
\begin{proposition}
\label{thm:compounding_main_paper}
Assume that the Q-function is initialized to $\bW_\phi(0)$ and $\bW_N(0)$. Let the Q-function class be closed under the backup, i.e., $\exists\ \bW_N^P, \bW_\phi^P, \text{~s.t.~} \left(\bR + \gamma P^\pi \bQ_{k-1} \right)^T =  \bW_N^P(k) \bW_\phi^P(k) [\mathcal{S}; \mathcal{A}]^T$, and assume that the change in $\srank$ due to dynamics and reward transformations is bounded: $\srank_\delta(\bW^P_\phi(k)) \leq \srank_\delta(\bW_\phi(k-1)) + c_k$. Then,   
\vspace{-0.3cm}
\begin{equation*}
    {\srank_\delta(\bW_\phi(k)) \leq {\srank_\delta (\bW_\phi(0)) + \sum_{j=1}^k c_j - \sum_{j=1}^k \frac{||\bQ_j - \by_j||}{\lambda_j}}}.
\vspace{-0.2cm}
\end{equation*}
\end{proposition}
\end{tcolorbox}
Proposition~\ref{thm:compounding_main_paper} provides a bound on the value of $\srank$ after $k$ rounds of bootstrapping. $\srank$ decreases in each iteration due to non-zero TD errors, but potentially increases due to reward and bootstrapping transformations. To instantiate a concrete case where rank clearly collapses, we investigate $c_k$ as the value function gets closer to the Bellman fixed point, which is a favourable initialization for the Q-function in Theorem~\ref{thm:near_optimality}. In this case, the learning dynamics begins to resemble the self-training regime, as the target values approach the previous value iterate $\by_k \approx \bQ_{k-1}$, and thus, as we show next, the potential increase in $\srank$ ($c_k$ in Proposition~\ref{thm:compounding_main_paper}) converges to $0$. 
\begin{tcolorbox}[colback=blue!6!white,colframe=black,boxsep=0pt,top=3pt,bottom=5pt]
\begin{theorem}
\label{thm:near_optimality}
Suppose target values $\by_{k} = \bR + \gamma P^\pi \bQ_{k-1}$ are close to the previous value estimate $\bQ_{k-1}$, i.e. $\forall~ \bs, \ba, ~y_{k}(\bs, \ba) = Q_{k-1}(\bs, \ba) + \varepsilon(\bs, \ba)$, with {$|\varepsilon(\bs, \ba)| \ll |Q_{k-1}(\bs, \ba)|$}. Then, there is a constant $\epsilon_0$ depending upon $\bW_N$ and $\bW_\phi$, such that for all $\|\varepsilon\| < \varepsilon_0$, $c_k = 0$. Thus, $\srank$ decreases in iteration $k$: 
${\srank_\delta(\bW_\phi(k)) \leq \srank_\delta(\bW_\phi(k-1)) - ||\bQ_k - y_k|| / \lambda_k}$.
\end{theorem}
\end{tcolorbox}


\begin{wrapfigure}{r}{0.47\textwidth}
\centering
\vspace{-5pt}
    \includegraphics[width=0.93\linewidth]{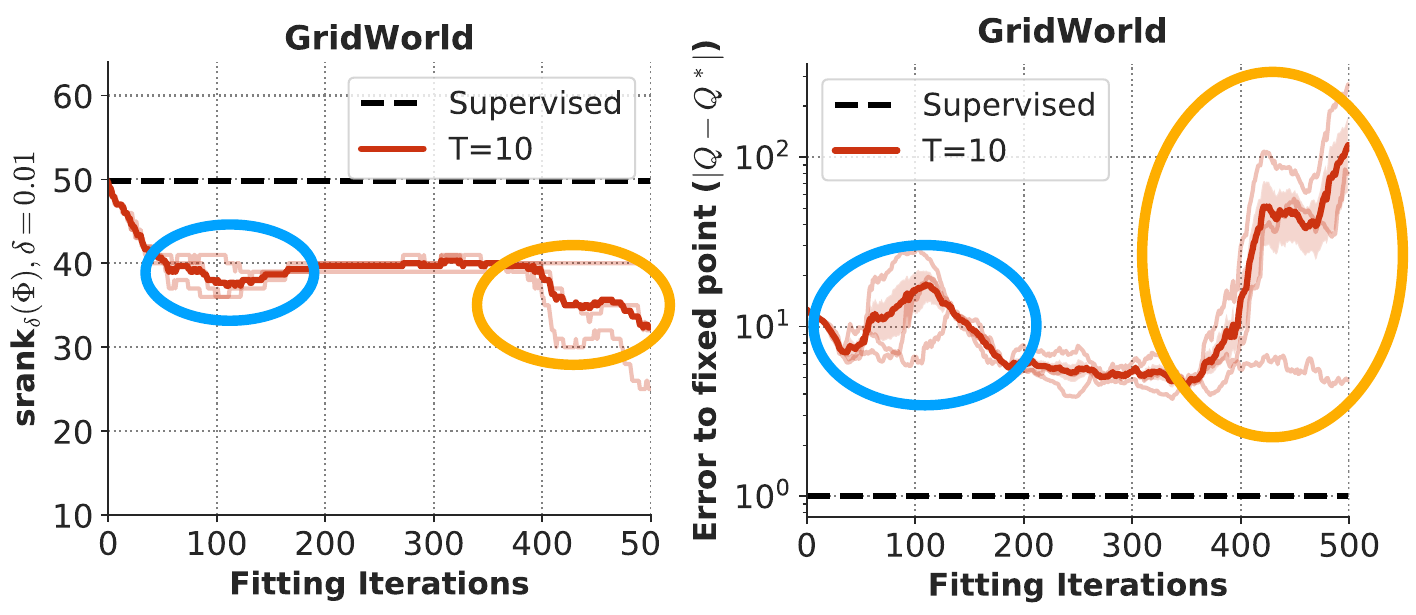}
    \vspace{-6pt}
    \caption{\small{{Trend of $\srank_\delta(\Phi)$ {\emph v.s.} error on log scale to the projected TD fixed point. A drop in $\srank_\delta(\Phi)$ (shown as blue and yellow circles) corresponds to a corresponding increase in distance to the fixed point.}}}
    \label{fig:linear_net_near_optimality}
    \vspace{-10pt}
\end{wrapfigure}
We provide a complete form, including the expression for $\epsilon_0$ and a proof in Appendix~\ref{app:proof_subopt}.
To empirically show the \textbf{consequence of Theorem~\ref{thm:near_optimality}} that a decrease in $\srank_\delta(\bW_\phi)$ values can lead to an increase in the distance to the fixed point in a neighborhood around the fixed point, we performed a controlled experiment on a deep linear net shown in Figure~\ref{fig:linear_net_near_optimality} that measures the relationship between of $\srank_\delta(\Phi)$ and the error to the projected TD fixed point $|\bQ - \bQ^*|$. Note that a drop in $\srank_{\delta}(\Phi)$  corresponds to a increased value of $|\bQ - \bQ^*|$ indicating that rank drop when $\bQ$ get close to a fixed point can affect convergence to it. 


\vspace{-5pt}
\section{Mitigating Under-Parametrization Improves Deep Q-Learning}
\label{sec:fixes}
\vspace{-5pt}

We now show that mitigating implicit under-parameterization by preventing rank collapse can improve performance. We place special emphasis on the offline RL setting in this section, since it is particularly vulnerable to the adverse effects of rank collapse.

\begin{wrapfigure}{r}{0.51\textwidth}
\begin{center}
\vspace{-27pt}
    \begin{subfigure}{0.245\textwidth}
        \includegraphics[width=\linewidth]{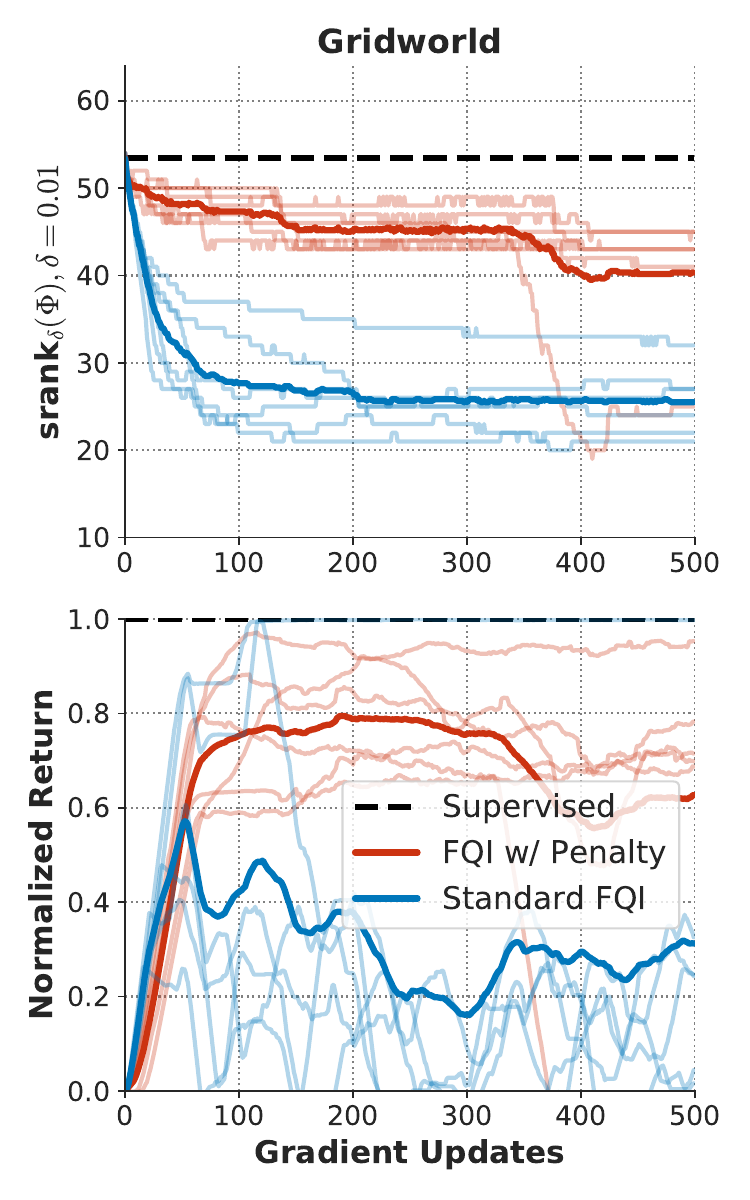}
        \vspace{-0.5cm}
        \caption{}\label{fig:gridworld_fix_results}
    \end{subfigure}
    \begin{subfigure}{0.245\textwidth}
        \includegraphics[width=\linewidth]{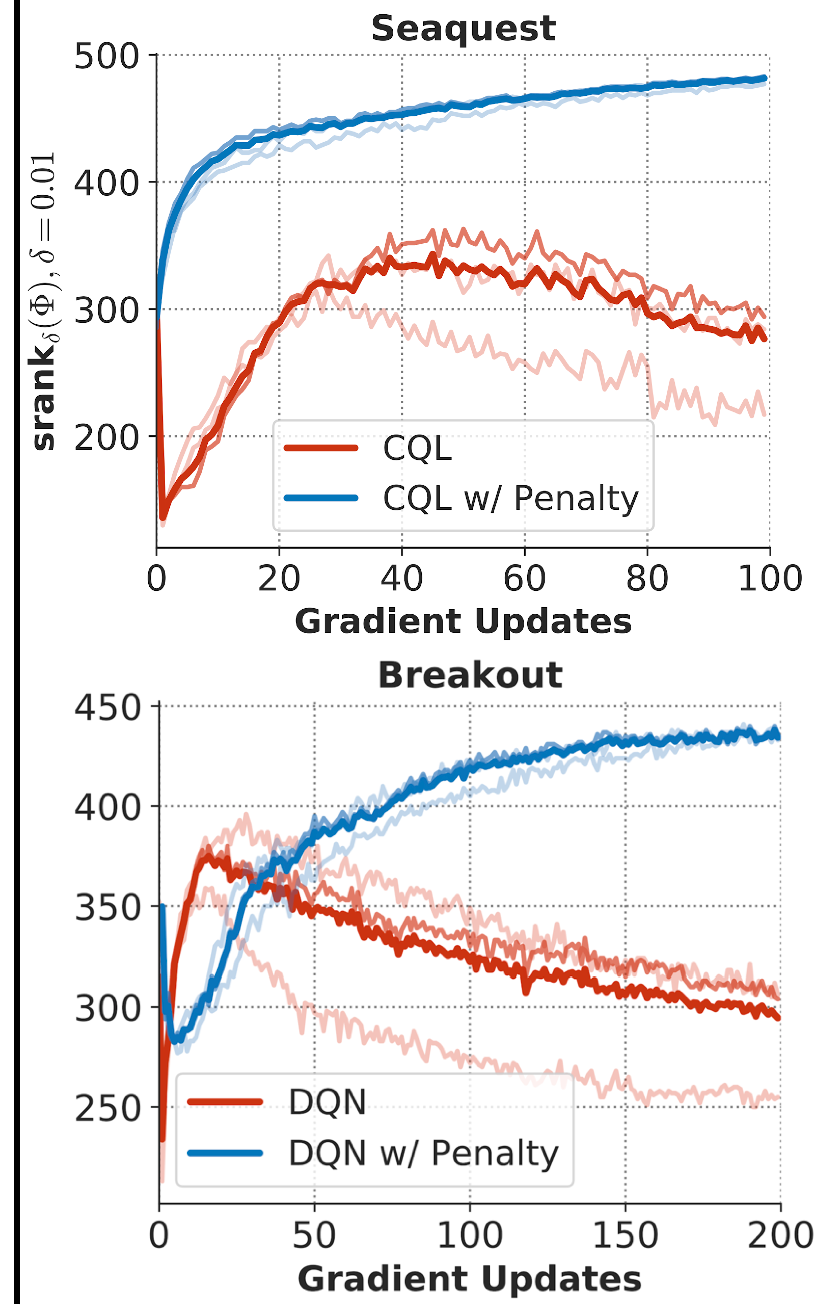}
        \vspace{-0.5cm}
        \caption{}\label{fig:atari_rank_increase_penalty}
    \end{subfigure}
    \vspace{-8pt}
    \caption{\textbf{(a)}: $\srank_\delta(\Phi)$ (top) and performance (bottom) of FQI on gridworld in the offline setting with 200 gradient updates per fitting iteration. Note reduced
    rank collapse and higher performance with the regularizer $\gL_p(\Phi)$. \textbf{(b)}: $\gL_p(\Phi)$ mitigates the rank collapse in DQN and CQL in the offline RL setting on Atari.}
    \vspace{-7pt}
\end{center}
\end{wrapfigure}


We devise a penalty (or a regularizer) $\gL_p(\Phi)$ that encourages higher effective rank of the learned features, $\srank_\delta(\Phi)$, to prevent rank collapse. The effective rank function $\srank_\delta(\Phi)$ is non-differentiable, so we choose a simple surrogate that can be optimized over deep networks. Since effective rank is maximized when the magnitude of the singular values is roughly balanced, one way to increase effective rank is to minimize the largest singular value of $\Phi$, $\sigma_{\max}(\Phi)$, while simultaneously maximizing the smallest singular value, $\sigma_{\min}(\Phi)$. We construct a simple penalty $\gL_p(\Phi)$ derived from this intuition, given by:
\begin{equation}
    \label{eqn:penalties}
    \gL_p(\Phi) = \sigma_{\max}^2(\Phi) - \sigma_{\min}^2(\Phi).
\end{equation}
$\gL_p(\Phi)$ can be computed by invoking the singular value decomposition subroutines in standard automatic differentiation frameworks~\citep{tensorflow, NIPS2019_9015}. We estimate the singular values over the feature matrix computed over a minibatch, and add the resulting value of $\gL_p$ as a penalty to the TD error objective, with a tradeoff factor $\alpha = 0.001$.

\textbf{Does $\gL_p(\Phi)$ address rank collapse?} We first verify whether controlling the minimum and maximum singular values using $\gL_p(\Phi)$ actually prevents rank collapse. When using this penalty on the gridworld problem (Figure \ref{fig:gridworld_fix_results}), the effective rank does not collapse, instead gradually decreasing at the onset and then plateauing, akin to the evolution of effective rank in supervised learning. In Figure~\ref{fig:atari_rank_increase_penalty}, we plot the evolution of effective rank on two Atari games in the {offline} setting (all games in Appendix~\ref{app:fix_results}), and observe that using $\gL_p$ also generally leads to increasing rank values.

\textbf{Does mitigating rank collapse improve performance?} 
We now evaluate the performance of the penalty using DQN~\citep{Mnih2015} and CQL \citep{kumar2020conservative} on Atari dataset from \citet{agarwal2020optimistic}~(5\% replay data), used in Section~\ref{sec:problem}. Figure \ref{fig:atari_offline_results} summarizes the relative improvement from using the penalty for 16 Atari games. Adding the penalty to DQN improves performance on all {16/16} games with a median improvement of \textbf{74.5\%}; adding it to CQL, a state-of-the-art offline algorithm, improves performance on {11/16} games with median improvement of \textbf{14.1\%}. 
Prior work has discussed that standard Q-learning methods designed for the online setting, such as DQN, are generally ineffective with small offline datasets~\citep{kumar2020conservative, agarwal2020optimistic}. Our results show that mitigating rank collapse makes even such simple methods substantially more effective in this setting, suggesting that rank collapse and the resulting implicit under-parameterization may be an crucial piece of the puzzle in explaining the challenges of offline RL.

\begin{wrapfigure}{r}{0.67\textwidth}
\centering
    \includegraphics[width=0.49\linewidth]{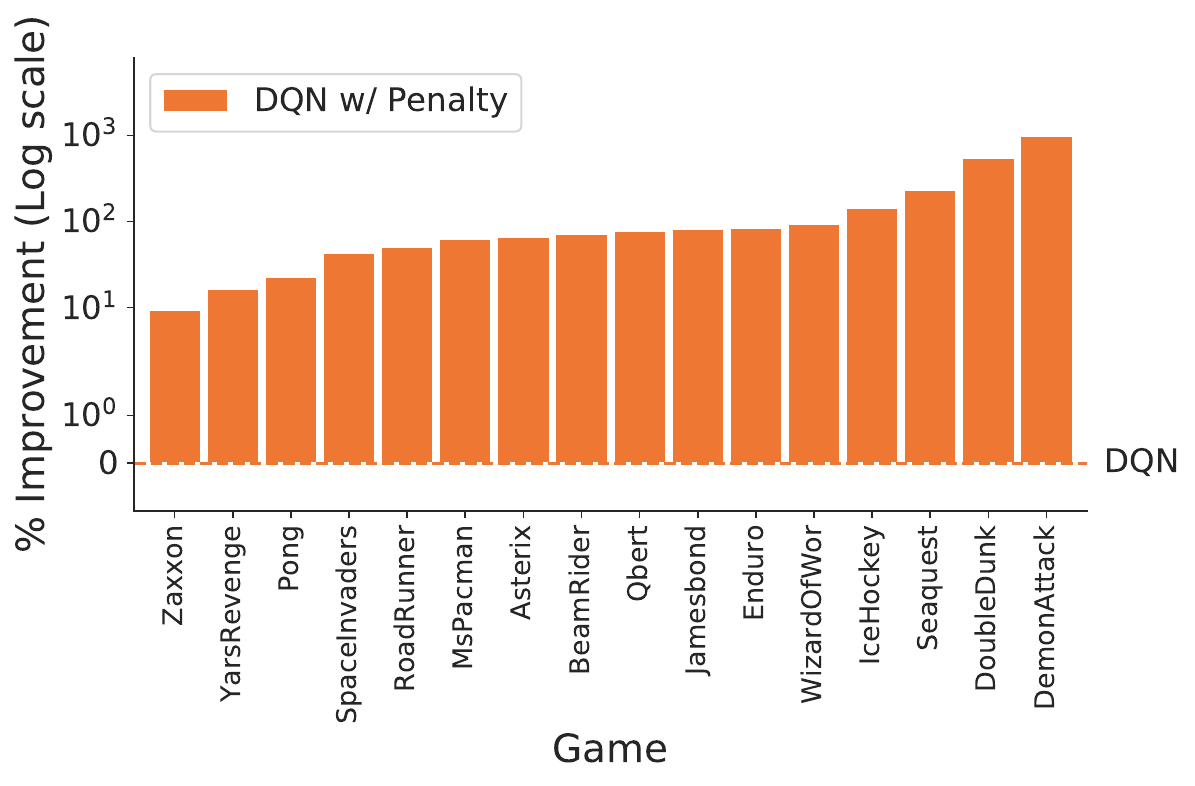}
    \includegraphics[width=0.49\linewidth]{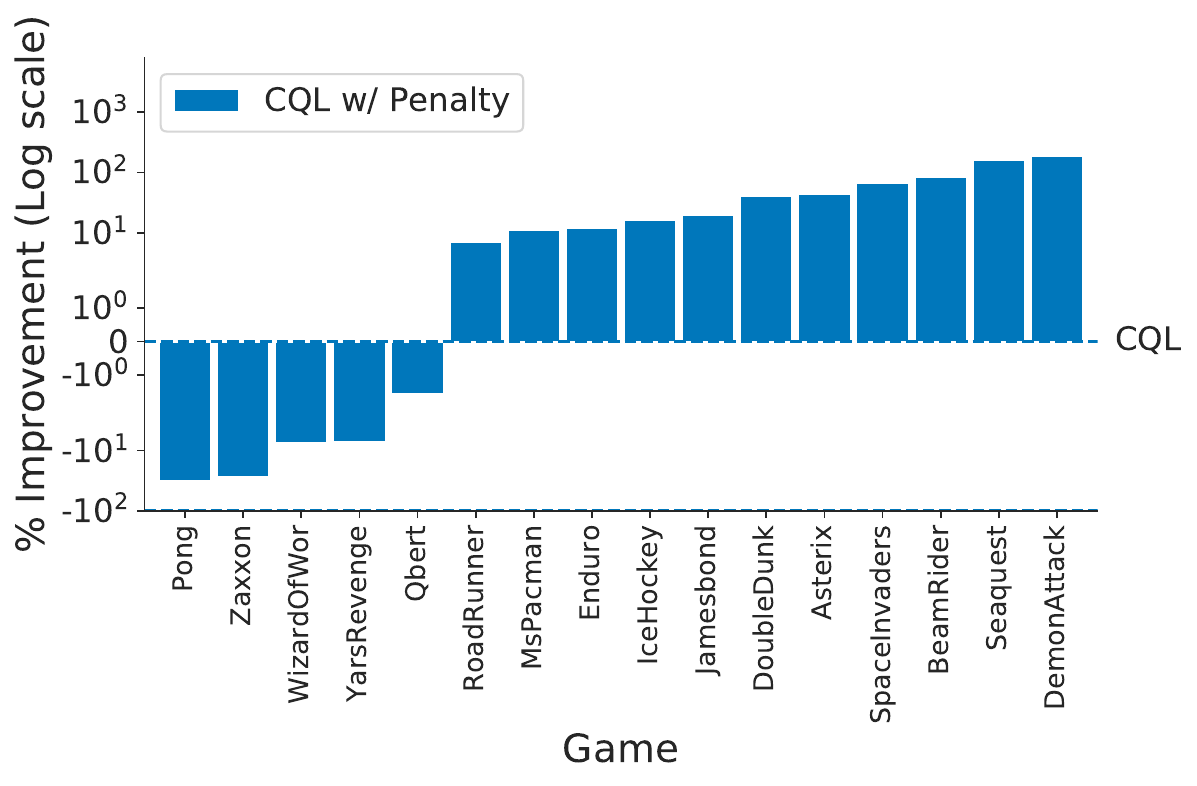}
    \vspace{-10pt}
    \caption{\small{DQN and CQL with $\gL_p(\Phi)$ penalty vs. their standard counterparts in the 5\% offline setting on Atari from Section~\ref{sec:problem}. $\gL_p$ improves DQN on 16/16 and CQL on 11/16 games.}}
    \label{fig:atari_offline_results}
    \vspace{-10pt}
\end{wrapfigure} 
We also evaluated the regularizer $\gL_p(\Phi)$ in the data-efficient online RL setting, with results in Appendix~\ref{sec:fix_rainbow_results}. This variant achieved median improvement of 20.6\%  performance with Rainbow~\citep{hessel2018rainbow}, however performed poorly with DQN, where it reduced median performance by 11.5\%.
Thus, while our proposed penalty is effective in many cases in offline and online settings, it does not solve the problem fully, i.e., it does not address the root cause of implicit under-parameterization and only addresses a symptom, and a more sophisticated solution may better prevent the issues with implicit under-parameterization. Nevertheless, our results suggest that mitigation of implicit under-parameterization can improve performance of data-efficient RL.

\vspace{-13pt}
\section{Related Work}
\label{sec:related}
\vspace{-10pt}

Prior work has extensively studied the learning dynamics of Q-learning with tabular and linear function approximation, to study error propagation~\citep{munos2003api,farahmand2010error} and to prevent divergence~\citep{de2002alp, maei09nonlineargtd,Sutton09b,Dai2018}, as opposed to deep Q-learning analyzed in this work. 
Q-learning has been shown to have favorable optimization properties with certain classes of features~\citep{ghosh2020representations}, but our work shows that the features learned by a neural net when minimizing TD error do not enjoy such guarantees, and instead suffer from rank collapse. Recent theoretical analyses of deep Q-learning have shown convergence under restrictive assumptions~\citep{yang2020theoretical,cai2019neural,zhang2020can,xu2019finite}, but Theorem~\ref{thm:near_optimality} shows that implicit under-parameterization appears when the estimates of the value function approach the optimum, potentially preventing convergence.
\citet{xu2005kernel,xu2007kernel} present variants of LSTD~\citep{boyan1999least}, which model the Q-function as a kernel-machine
but do not take into account the regularization from gradient descent, as done in Equation~\ref{eqn:self_dist_problem}, which is essential for implicit under-parameterization.
\citet{igl2020impact,fedus2020catastrophic} argue that non-stationarity arising from distribution shift hinders generalization and recommend periodic network re-initialization. Under-parameterization is not caused by this distribution shift, and we find that network re-initialization does little to prevent rank collapse (Figure~\ref{fig:reinit}).  \citet{luo2020i4r} proposes a regularization similar to ours, but in a different setting, finding that more expressive features increases performance of on-policy RL methods.
Finally, \citet{yang2019harnessing} study the effective rank of the $Q^*$-values when expressed as a $|\states| \times |\actions|$ matrix in online RL and find that low ranks for this $Q^*$-matrix are preferable. We analyze a fundamentally different object: the learned features (and illustrate that a rank-collapse of features can hurt), not the $Q^*$-matrix, whose rank is upper-bounded by the number of actions (\eg 24 for Atari). 

\vspace{-12pt}
\section{Discussion}
\vspace{-10pt}

We identified an implicit under-parameterization phenomenon in deep RL algorithms that use bootstrapping, where gradient-based optimization of a bootstrapped objective can lead to a reduction in the expressive power of the value network. This effect manifests as a collapse of the rank of the features learned by the value network, causing aliasing across states and often leading to poor performance. Our analysis reveals that this phenomenon is caused by the implicit regularization due to gradient descent on bootstrapped objectives. We observed that mitigating this problem by means of a simple regularization scheme improves performance of deep Q-learning methods.

While our proposed regularization provides some improvement, 
devising better mitigation strategies for implicit under-parameterization remains an exciting direction for future work. Our method explicitly attempts to prevent rank collapse, but relies on the emergence of useful features solely through the bootstrapped signal. An alternative path may be to develop new auxiliary losses  ~\citep[\eg][]{jaderberg2016reinforcement} that learn useful features while passively preventing underparameterization. 
More broadly, understanding the effects of neural nets and associated factors such as initialization, choice of optimizer, \etc on the learning dynamics of deep RL algorithms, using tools from deep learning theory, is likely to be key towards developing robust and data-efficient deep RL algorithms.

\section*{Acknowledgements}
We thank Lihong Li, Aaron Courville, Aurick Zhou, Abhishek Gupta, George Tucker, Ofir Nachum, Wesley Chung, Emmanuel Bengio, Zafarali Ahmed, and Jacob Buckman for feedback on an earlier version of this paper. We thank Hossein Mobahi for insightful discussions about self-distillation and Hanie Sedghi for insightful discussions about implicit regularization and generalization in deep networks. We additionally thank Michael Janner, Aaron Courville, Dale Schuurmans and Marc Bellemare for helpful discussions. AK was partly funded by the DARPA Assured Autonomy program, and DG was supported by a NSF graduate fellowship and compute support from Amazon.

\bibliography{iclr2021_conference}

\begin{thebibliography}{56}
\providecommand{\natexlab}[1]{#1}
\providecommand{\url}[1]{\texttt{#1}}
\expandafter\ifx\csname urlstyle\endcsname\relax
  \providecommand{\doi}[1]{doi: #1}\else
  \providecommand{\doi}{doi: \begingroup \urlstyle{rm}\Url}\fi

\bibitem[Abadi et~al.(2016)Abadi, Barham, Chen, Chen, Davis, Dean, Devin,
  Ghemawat, Irving, Isard, Kudlur, Levenberg, Monga, Moore, Murray, Steiner,
  Tucker, Vasudevan, Warden, Wicke, Yu, and Zheng]{tensorflow}
Martin Abadi, Paul Barham, Jianmin Chen, Zhifeng Chen, Andy Davis, Jeffrey
  Dean, Matthieu Devin, Sanjay Ghemawat, Geoffrey Irving, Michael Isard,
  Manjunath Kudlur, Josh Levenberg, Rajat Monga, Sherry Moore, Derek~G. Murray,
  Benoit Steiner, Paul Tucker, Vijay Vasudevan, Pete Warden, Martin Wicke, Yuan
  Yu, and Xiaoqiang Zheng.
\newblock Tensorflow: A system for large-scale machine learning.
\newblock In \emph{12th USENIX Symposium on Operating Systems Design and
  Implementation (OSDI 16)}, pp.\  265--283, 2016.
\newblock URL
  \url{https://www.usenix.org/system/files/conference/osdi16/osdi16-abadi.pdf}.

\bibitem[Achiam et~al.(2019)Achiam, Knight, and Abbeel]{Achiam2019TowardsCD}
Joshua Achiam, Ethan Knight, and Pieter Abbeel.
\newblock Towards characterizing divergence in deep q-learning.
\newblock \emph{ArXiv}, abs/1903.08894, 2019.

\bibitem[Agarwal et~al.(2020)Agarwal, Schuurmans, and
  Norouzi]{agarwal2020optimistic}
Rishabh Agarwal, Dale Schuurmans, and Mohammad Norouzi.
\newblock An optimistic perspective on offline reinforcement learning.
\newblock In \emph{International Conference on Machine Learning (ICML)}, 2020.

\bibitem[Arora et~al.(2018)Arora, Cohen, and Hazan]{arora2018optimization}
Sanjeev Arora, Nadav Cohen, and Elad Hazan.
\newblock On the optimization of deep networks: Implicit acceleration by
  overparameterization.
\newblock \emph{arXiv preprint arXiv:1802.06509}, 2018.

\bibitem[Arora et~al.(2019)Arora, Cohen, Hu, and Luo]{arora2019implicit}
Sanjeev Arora, Nadav Cohen, Wei Hu, and Yuping Luo.
\newblock Implicit regularization in deep matrix factorization.
\newblock In \emph{Advances in Neural Information Processing Systems}, pp.\
  7413--7424, 2019.

\bibitem[Bellemare et~al.(2013)Bellemare, Naddaf, Veness, and
  Bowling]{bellemare2013ale}
Marc~G. Bellemare, Yavar Naddaf, Joel Veness, and Michael Bowling.
\newblock The arcade learning environment: An evaluation platform for general
  agents.
\newblock \emph{J. Artif. Int. Res.}, 47\penalty0 (1):\penalty0 253–279, May
  2013.
\newblock ISSN 1076-9757.

\bibitem[Bellemare et~al.(2017)Bellemare, Dabney, and
  Munos]{bellemare2017distributional}
Marc~G Bellemare, Will Dabney, and R{\'e}mi Munos.
\newblock A distributional perspective on reinforcement learning.
\newblock In \emph{Proceedings of the 34th International Conference on Machine
  Learning-Volume 70}, pp.\  449--458. JMLR. org, 2017.

\bibitem[Bengio et~al.(2020)Bengio, Pineau, and Precup]{bengio2020interference}
Emmanuel Bengio, Joelle Pineau, and Doina Precup.
\newblock Interference and generalization in temporal difference learning.
\newblock \emph{arXiv preprint arXiv:2003.06350}, 2020.

\bibitem[Boyan(1999)]{boyan1999least}
Justin~A Boyan.
\newblock Least-squares temporal difference learning.
\newblock In \emph{ICML}, pp.\  49--56. Citeseer, 1999.

\bibitem[Brockman et~al.(2016)Brockman, Cheung, Pettersson, Schneider,
  Schulman, Tang, and Zaremba]{gym}
Greg Brockman, Vicki Cheung, Ludwig Pettersson, Jonas Schneider, John Schulman,
  Jie Tang, and Wojciech Zaremba.
\newblock Openai gym, 2016.

\bibitem[Cai et~al.(2019)Cai, Yang, Lee, and Wang]{cai2019neural}
Qi~Cai, Zhuoran Yang, Jason~D Lee, and Zhaoran Wang.
\newblock Neural temporal-difference and q-learning provably converge to global
  optima.
\newblock \emph{arXiv preprint arXiv:1905.10027}, 2019.

\bibitem[Chen \& Jiang(2019)Chen and Jiang]{chen2019information}
Jinglin Chen and Nan Jiang.
\newblock Information-theoretic considerations in batch reinforcement learning.
\newblock \emph{ICML}, 2019.

\bibitem[Dai et~al.(2018)Dai, Shaw, Li, Xiao, He, Liu, Chen, and Song]{Dai2018}
Bo~Dai, Albert Shaw, Lihong Li, Lin Xiao, Niao He, Zhen Liu, Jianshu Chen, and
  Le~Song.
\newblock Sbeed: Convergent reinforcement learning with nonlinear function
  approximation.
\newblock In \emph{International Conference on Machine Learning}, pp.\
  1133--1142, 2018.

\bibitem[De~Farias(2002)]{de2002alp}
Daniela~Pucci De~Farias.
\newblock \emph{The linear programming approach to approximate dynamic
  programming: Theory and application}.
\newblock PhD thesis, 2002.

\bibitem[Duffy(2015)]{duffy2015green}
Dean~G Duffy.
\newblock \emph{Green's functions with applications}.
\newblock CRC Press, 2015.

\bibitem[Ernst et~al.(2005)Ernst, Geurts, and Wehenkel]{Ernst05}
Damien Ernst, Pierre Geurts, and Louis Wehenkel.
\newblock Tree-based batch mode reinforcement learning.
\newblock \emph{Journal of Machine Learning Research}, 6\penalty0
  (Apr):\penalty0 503--556, 2005.

\bibitem[Farahmand et~al.(2010)Farahmand, Szepesv{\'a}ri, and
  Munos]{farahmand2010error}
Amir-massoud Farahmand, Csaba Szepesv{\'a}ri, and R{\'e}mi Munos.
\newblock Error propagation for approximate policy and value iteration.
\newblock In \emph{Advances in Neural Information Processing Systems (NIPS)},
  2010.

\bibitem[Fedus et~al.(2020{\natexlab{a}})Fedus, Ghosh, Martin, Bellemare,
  Bengio, and Larochelle]{fedus2020catastrophic}
William Fedus, Dibya Ghosh, John~D Martin, Marc~G Bellemare, Yoshua Bengio, and
  Hugo Larochelle.
\newblock On catastrophic interference in atari 2600 games.
\newblock \emph{arXiv preprint arXiv:2002.12499}, 2020{\natexlab{a}}.

\bibitem[Fedus et~al.(2020{\natexlab{b}})Fedus, Ramachandran, Agarwal, Bengio,
  Larochelle, Rowland, and Dabney]{fedus2020revisiting}
William Fedus, Prajit Ramachandran, Rishabh Agarwal, Yoshua Bengio, Hugo
  Larochelle, Mark Rowland, and Will Dabney.
\newblock Revisiting fundamentals of experience replay.
\newblock \emph{arXiv preprint arXiv:2007.06700}, 2020{\natexlab{b}}.

\bibitem[Fu et~al.(2019)Fu, Kumar, Soh, and Levine]{fu19diagnosing}
Justin Fu, Aviral Kumar, Matthew Soh, and Sergey Levine.
\newblock Diagnosing bottlenecks in deep q-learning algorithms.
\newblock In \emph{Proceedings of the 36th International Conference on Machine
  Learning}. PMLR, 2019.

\bibitem[Ghosh \& Bellemare(2020)Ghosh and Bellemare]{ghosh2020representations}
Dibya Ghosh and Marc~G Bellemare.
\newblock Representations for stable off-policy reinforcement learning.
\newblock \emph{arXiv preprint arXiv:2007.05520}, 2020.

\bibitem[Gulcehre et~al.(2020)Gulcehre, Wang, Novikov, Paine, Colmenarejo,
  Zolna, Agarwal, Merel, Mankowitz, Paduraru, et~al.]{gulcehre2020rl}
Caglar Gulcehre, Ziyu Wang, Alexander Novikov, Tom~Le Paine, Sergio~G{\'o}mez
  Colmenarejo, Konrad Zolna, Rishabh Agarwal, Josh Merel, Daniel Mankowitz,
  Cosmin Paduraru, et~al.
\newblock Rl unplugged: Benchmarks for offline reinforcement learning.
\newblock 2020.

\bibitem[Gunasekar et~al.(2017)Gunasekar, Woodworth, Bhojanapalli, Neyshabur,
  and Srebro]{gunasekar2017implicit}
Suriya Gunasekar, Blake~E Woodworth, Srinadh Bhojanapalli, Behnam Neyshabur,
  and Nati Srebro.
\newblock Implicit regularization in matrix factorization.
\newblock In \emph{Advances in Neural Information Processing Systems}, pp.\
  6151--6159, 2017.

\bibitem[Gunasekar et~al.(2018)Gunasekar, Lee, Soudry, and
  Srebro]{gunasekar2018implicit}
Suriya Gunasekar, Jason~D Lee, Daniel Soudry, and Nati Srebro.
\newblock Implicit bias of gradient descent on linear convolutional networks.
\newblock In \emph{Advances in Neural Information Processing Systems}, pp.\
  9461--9471, 2018.

\bibitem[Haarnoja et~al.(2018)Haarnoja, Zhou, Abbeel, and Levine]{Haarnoja18}
Tuomas Haarnoja, Aurick Zhou, Pieter Abbeel, and Sergey Levine.
\newblock Soft actor-critic: Off-policy maximum entropy deep reinforcement
  learning with a stochastic actor.
\newblock \emph{CoRR}, abs/1801.01290, 2018.
\newblock URL \url{http://arxiv.org/abs/1801.01290}.

\bibitem[Hessel et~al.(2018)Hessel, Modayil, Van~Hasselt, Schaul, Ostrovski,
  Dabney, Horgan, Piot, Azar, and Silver]{hessel2018rainbow}
Matteo Hessel, Joseph Modayil, Hado Van~Hasselt, Tom Schaul, Georg Ostrovski,
  Will Dabney, Dan Horgan, Bilal Piot, Mohammad Azar, and David Silver.
\newblock Rainbow: Combining improvements in deep reinforcement learning.
\newblock In \emph{Thirty-Second AAAI Conference on Artificial Intelligence},
  2018.

\bibitem[Hlawka et~al.(1991)Hlawka, Taschner, and
  Schoi{\ss}engeier]{Hlawka1991dirichlet}
Edmund Hlawka, Rudolf Taschner, and Johannes Schoi{\ss}engeier.
\newblock \emph{The Dirichlet Approximation Theorem}, pp.\  1--18.
\newblock Springer Berlin Heidelberg, Berlin, Heidelberg, 1991.
\newblock ISBN 978-3-642-75306-0.
\newblock \doi{10.1007/978-3-642-75306-0_1}.
\newblock URL \url{https://doi.org/10.1007/978-3-642-75306-0_1}.

\bibitem[Igl et~al.(2020)Igl, Farquhar, Luketina, Boehmer, and
  Whiteson]{igl2020impact}
Maximilian Igl, Gregory Farquhar, Jelena Luketina, Wendelin Boehmer, and Shimon
  Whiteson.
\newblock The impact of non-stationarity on generalisation in deep
  reinforcement learning.
\newblock \emph{arXiv preprint arXiv:2006.05826}, 2020.

\bibitem[Jacot et~al.(2018)Jacot, Gabriel, and Hongler]{ntk}
Arthur Jacot, Franck Gabriel, and Clement Hongler.
\newblock Neural tangent kernel: Convergence and generalization in neural
  networks.
\newblock In \emph{Advances in Neural Information Processing Systems 31}. 2018.

\bibitem[Jaderberg et~al.(2016)Jaderberg, Mnih, Czarnecki, Schaul, Leibo,
  Silver, and Kavukcuoglu]{jaderberg2016reinforcement}
Max Jaderberg, Volodymyr Mnih, Wojciech~Marian Czarnecki, Tom Schaul, Joel~Z
  Leibo, David Silver, and Koray Kavukcuoglu.
\newblock Reinforcement learning with unsupervised auxiliary tasks.
\newblock \emph{arXiv preprint arXiv:1611.05397}, 2016.

\bibitem[Johnson(2016)]{approxstack}
Robert Johnson.
\newblock Approximate irrational numbers by rational numbers, 2016.
\newblock URL \url{https://math.stackexchange.com/questions/1829743/}.

\bibitem[Kumar et~al.(2020{\natexlab{a}})Kumar, Gupta, and
  Levine]{kumar2020discor}
Aviral Kumar, Abhishek Gupta, and Sergey Levine.
\newblock Discor: Corrective feedback in reinforcement learning via
  distribution correction.
\newblock \emph{arXiv preprint arXiv:2003.07305}, 2020{\natexlab{a}}.

\bibitem[Kumar et~al.(2020{\natexlab{b}})Kumar, Zhou, Tucker, and
  Levine]{kumar2020conservative}
Aviral Kumar, Aurick Zhou, George Tucker, and Sergey Levine.
\newblock Conservative q-learning for offline reinforcement learning.
\newblock \emph{arXiv preprint arXiv:2006.04779}, 2020{\natexlab{b}}.

\bibitem[Lange et~al.(2012)Lange, Gabel, and Riedmiller]{lange2012batch}
Sascha Lange, Thomas Gabel, and Martin Riedmiller.
\newblock Batch reinforcement learning.
\newblock In \emph{Reinforcement learning}, pp.\  45--73. Springer, 2012.

\bibitem[Levine et~al.(2020)Levine, Kumar, Tucker, and Fu]{levine2020offline}
Sergey Levine, Aviral Kumar, George Tucker, and Justin Fu.
\newblock Offline reinforcement learning: Tutorial, review, and perspectives on
  open problems.
\newblock \emph{arXiv preprint arXiv:2005.01643}, 2020.

\bibitem[Liu et~al.(2018)Liu, Kumaraswamy, Le, and White]{martha2018sparse}
Vincent Liu, Raksha Kumaraswamy, Lei Le, and Martha White.
\newblock The utility of sparse representations for control in reinforcement
  learning.
\newblock \emph{CoRR}, abs/1811.06626, 2018.
\newblock URL \url{http://arxiv.org/abs/1811.06626}.

\bibitem[Luo et~al.(2020)Luo, Meng, He, Chen, and Wang]{luo2020i4r}
Xufang Luo, Qi~Meng, Di~He, Wei Chen, and Yunhong Wang.
\newblock I4r: Promoting deep reinforcement learning by the indicator for
  expressive representations.
\newblock In Christian Bessiere (ed.), \emph{Proceedings of the Twenty-Ninth
  International Joint Conference on Artificial Intelligence, {IJCAI-20}}, pp.\
  2669--2675. International Joint Conferences on Artificial Intelligence
  Organization, 7 2020.
\newblock \doi{10.24963/ijcai.2020/370}.
\newblock URL \url{https://doi.org/10.24963/ijcai.2020/370}.
\newblock Main track.

\bibitem[Maei et~al.(2009)Maei, Szepesv\'{a}ri, Bhatnagar, Precup, Silver, and
  Sutton]{maei09nonlineargtd}
Hamid~R. Maei, Csaba Szepesv\'{a}ri, Shalabh Bhatnagar, Doina Precup, David
  Silver, and Richard~S. Sutton.
\newblock Convergent temporal-difference learning with arbitrary smooth
  function approximation.
\newblock In \emph{Proceedings of the 22nd International Conference on Neural
  Information Processing Systems}, 2009.

\bibitem[Mnih et~al.(2015)Mnih, Kavukcuoglu, Silver, Rusu, Veness, Bellemare,
  Graves, Riedmiller, Fidjeland, Ostrovski, Petersen, Beattie, Sadik,
  Antonoglou, King, Kumaran, Wierstra, Legg, and Hassabis]{Mnih2015}
Volodymyr Mnih, Koray Kavukcuoglu, David Silver, Andrei~A Rusu, Joel Veness,
  Marc~G Bellemare, Alex Graves, Martin Riedmiller, Andreas~K Fidjeland, Georg
  Ostrovski, Stig Petersen, Charles Beattie, Amir Sadik, Ioannis Antonoglou,
  Helen King, Dharshan Kumaran, Daan Wierstra, Shane Legg, and Demis Hassabis.
\newblock {Human-level control through deep reinforcement learning}.
\newblock \emph{Nature}, 518\penalty0 (7540):\penalty0 529--533, feb 2015.
\newblock ISSN 0028-0836.

\bibitem[Mobahi et~al.(2020)Mobahi, Farajtabar, and Bartlett]{mobahi2020self}
Hossein Mobahi, Mehrdad Farajtabar, and Peter~L Bartlett.
\newblock Self-distillation amplifies regularization in hilbert space.
\newblock \emph{arXiv preprint arXiv:2002.05715}, 2020.

\bibitem[Munos(2003)]{munos2003api}
R\'{e}mi Munos.
\newblock Error bounds for approximate policy iteration.
\newblock In \emph{Proceedings of the Twentieth International Conference on
  International Conference on Machine Learning}, ICML’03, pp.\  560–567.
  AAAI Press, 2003.
\newblock ISBN 1577351894.

\bibitem[Munos \& Szepesv{\'a}ri(2008)Munos and
  Szepesv{\'a}ri]{munos2008finite}
R{\'e}mi Munos and Csaba Szepesv{\'a}ri.
\newblock Finite-time bounds for fitted value iteration.
\newblock \emph{Journal of Machine Learning Research}, 9\penalty0
  (May):\penalty0 815--857, 2008.

\bibitem[Paszke et~al.(2019)Paszke, Gross, Massa, Lerer, Bradbury, Chanan,
  Killeen, Lin, Gimelshein, Antiga, Desmaison, Kopf, Yang, DeVito, Raison,
  Tejani, Chilamkurthy, Steiner, Fang, Bai, and Chintala]{NIPS2019_9015}
Adam Paszke, Sam Gross, Francisco Massa, Adam Lerer, James Bradbury, Gregory
  Chanan, Trevor Killeen, Zeming Lin, Natalia Gimelshein, Luca Antiga, Alban
  Desmaison, Andreas Kopf, Edward Yang, Zachary DeVito, Martin Raison, Alykhan
  Tejani, Sasank Chilamkurthy, Benoit Steiner, Lu~Fang, Junjie Bai, and Soumith
  Chintala.
\newblock Pytorch: An imperative style, high-performance deep learning library.
\newblock In H.~Wallach, H.~Larochelle, A.~Beygelzimer, F.~d\textquotesingle
  Alch\'{e}-Buc, E.~Fox, and R.~Garnett (eds.), \emph{Advances in Neural
  Information Processing Systems 32}, pp.\  8026--8037. Curran Associates,
  Inc., 2019.

\bibitem[Puterman(1994)]{puterman1994markov}
Martin~L Puterman.
\newblock \emph{Markov Decision Processes: Discrete Stochastic Dynamic
  Programming}.
\newblock John Wiley \& Sons, Inc., 1994.

\bibitem[Riedmiller(2005)]{Riedmiller2005}
Martin Riedmiller.
\newblock Neural fitted q iteration--first experiences with a data efficient
  neural reinforcement learning method.
\newblock In \emph{European Conference on Machine Learning}, pp.\  317--328.
  Springer, 2005.

\bibitem[Ruhe(1975)]{ruhe1975closeness}
Axel Ruhe.
\newblock On the closeness of eigenvalues and singular values for almost normal
  matrices.
\newblock \emph{Linear Algebra and its Applications}, 11\penalty0 (1):\penalty0
  87--93, 1975.

\bibitem[Silver et~al.(2017)Silver, Schrittwieser, Simonyan, Antonoglou, Huang,
  Guez, Hubert, Baker, Lai, Bolton, et~al.]{silver2017mastering}
David Silver, Julian Schrittwieser, Karen Simonyan, Ioannis Antonoglou, Aja
  Huang, Arthur Guez, Thomas Hubert, Lucas Baker, Matthew Lai, Adrian Bolton,
  et~al.
\newblock Mastering the game of go without human knowledge.
\newblock \emph{nature}, 550\penalty0 (7676):\penalty0 354--359, 2017.

\bibitem[Suggala et~al.(2018)Suggala, Prasad, and
  Ravikumar]{suggala2018connecting}
Arun Suggala, Adarsh Prasad, and Pradeep~K Ravikumar.
\newblock Connecting optimization and regularization paths.
\newblock In \emph{Advances in Neural Information Processing Systems}, pp.\
  10608--10619, 2018.

\bibitem[Sutton et~al.(2009)Sutton, Maei, Precup, Bhatnagar, Silver,
  Szepesv\'{a}ri, and Wiewiora]{Sutton09b}
Richard~S. Sutton, Hamid~Reza Maei, Doina Precup, Shalabh Bhatnagar, David
  Silver, Csaba Szepesv\'{a}ri, and Eric Wiewiora.
\newblock Fast gradient-descent methods for temporal-difference learning with
  linear function approximation.
\newblock In \emph{International Conference on Machine Learning (ICML)}, 2009.

\bibitem[Townsend(2016)]{townsend2016differentiating}
James Townsend.
\newblock Differentiating the singular value decomposition.
\newblock Technical report, Technical Report 2016, https://j-towns. github.
  io/papers/svd-derivative~…, 2016.

\bibitem[Xu \& Gu(2019)Xu and Gu]{xu2019finite}
Pan Xu and Quanquan Gu.
\newblock A finite-time analysis of q-learning with neural network function
  approximation.
\newblock \emph{arXiv preprint arXiv:1912.04511}, 2019.

\bibitem[Xu et~al.(2005)Xu, Xie, Hu, and Lu]{xu2005kernel}
Xin Xu, Tao Xie, Dewen Hu, and Xicheng Lu.
\newblock Kernel least-squares temporal difference learning.
\newblock \emph{International Journal of Information Technology}, 11\penalty0
  (9):\penalty0 54--63, 2005.

\bibitem[Xu et~al.(2007)Xu, Hu, and Lu]{xu2007kernel}
Xin Xu, Dewen Hu, and Xicheng Lu.
\newblock Kernel-based least squares policy iteration for reinforcement
  learning.
\newblock \emph{IEEE Transactions on Neural Networks}, 18\penalty0
  (4):\penalty0 973--992, 2007.

\bibitem[Yang et~al.(2019)Yang, Zhang, Xu, and Katabi]{yang2019harnessing}
Yuzhe Yang, Guo Zhang, Zhi Xu, and Dina Katabi.
\newblock Harnessing structures for value-based planning and reinforcement
  learning.
\newblock \emph{arXiv preprint arXiv:1909.12255}, 2019.

\bibitem[Yang et~al.(2020)Yang, Xie, and Wang]{yang2020theoretical}
Zhuoran Yang, Yuchen Xie, and Zhaoran Wang.
\newblock A theoretical analysis of deep q-learning.
\newblock In \emph{Learning for Dynamics and Control}, pp.\  486--489. PMLR,
  2020.

\bibitem[Zhang et~al.(2020)Zhang, Cai, Yang, Chen, and Wang]{zhang2020can}
Yufeng Zhang, Qi~Cai, Zhuoran Yang, Yongxin Chen, and Zhaoran Wang.
\newblock Can temporal-difference and q-learning learn representation? a
  mean-field theory.
\newblock \emph{arXiv preprint arXiv:2006.04761}, 2020.

\end{thebibliography}
\bibliographystyle{iclr2021_conference}

\appendix
\newpage

\part*{Appendices}

\numberwithin{equation}{section}
\numberwithin{proposition}{section}
\numberwithin{figure}{section}
\numberwithin{table}{section}

\vspace{-5pt}
\section{Additional Evidence for Implicit Under-Parameterization}
\label{app:more_evidence}
\vspace{-9pt}
In this section, we present additional evidence that demonstrates the existence of the implicit under-parameterization phenomenon from Section~\ref{sec:problem}. In all cases, we plot the values of $\srank_\delta(\Phi)$ computed on a batch size of 2048 i.i.d. sampled transitions from the dataset.
\vspace{-5pt}
\subsection{Offline RL}\label{more_evidence_offline}
\begin{figure}[H]
    \centering
    \vspace{-8pt}
    \includegraphics[width=\linewidth]{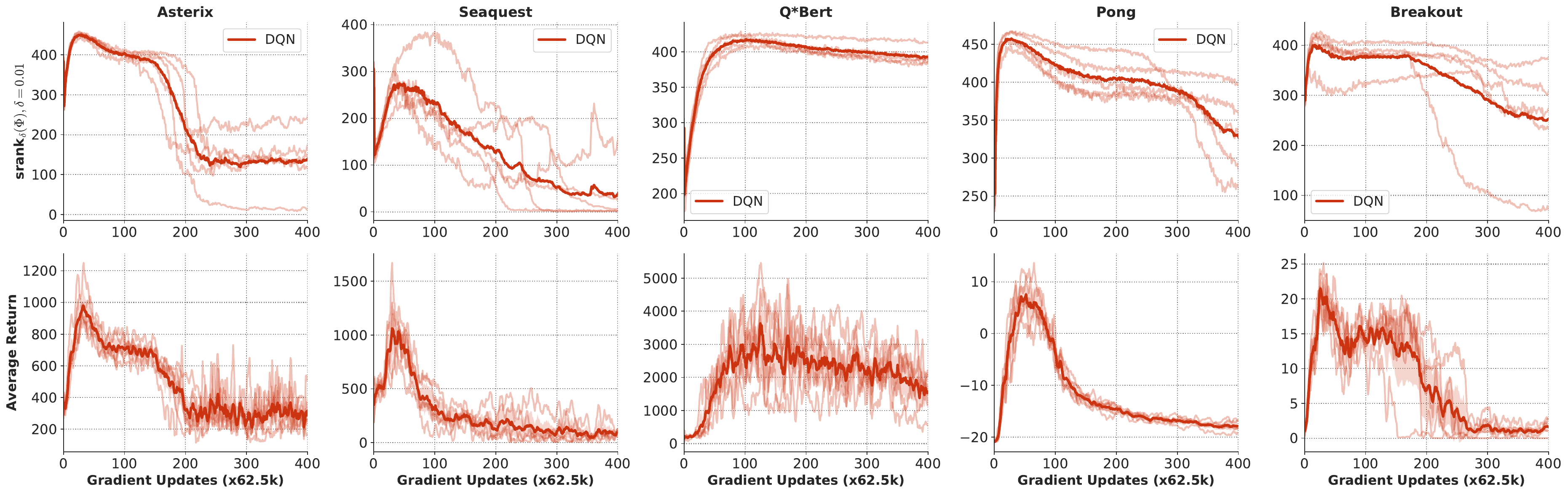}
    \caption{{\textbf{Offline DQN on Atari}. $\srank_\delta(\Phi)$ and performance of DQN on five Atari games in the offline RL setting using the 5\% DQN Replay dataset~\citep{agarwal2020optimistic} (marked as \textbf{DQN}). Note that low $\srank$ (top row) generally corresponds to worse policy performance (bottom row). Also note that rank collapse begins to generally occur close to the position of peak return. Average across 5 runs is showed for each game with individual runs.}}
    \label{fig:offline_problem_app}
\end{figure}

\begin{figure}[H]
    \centering
    \vspace{-20pt}
    \includegraphics[width=\linewidth]{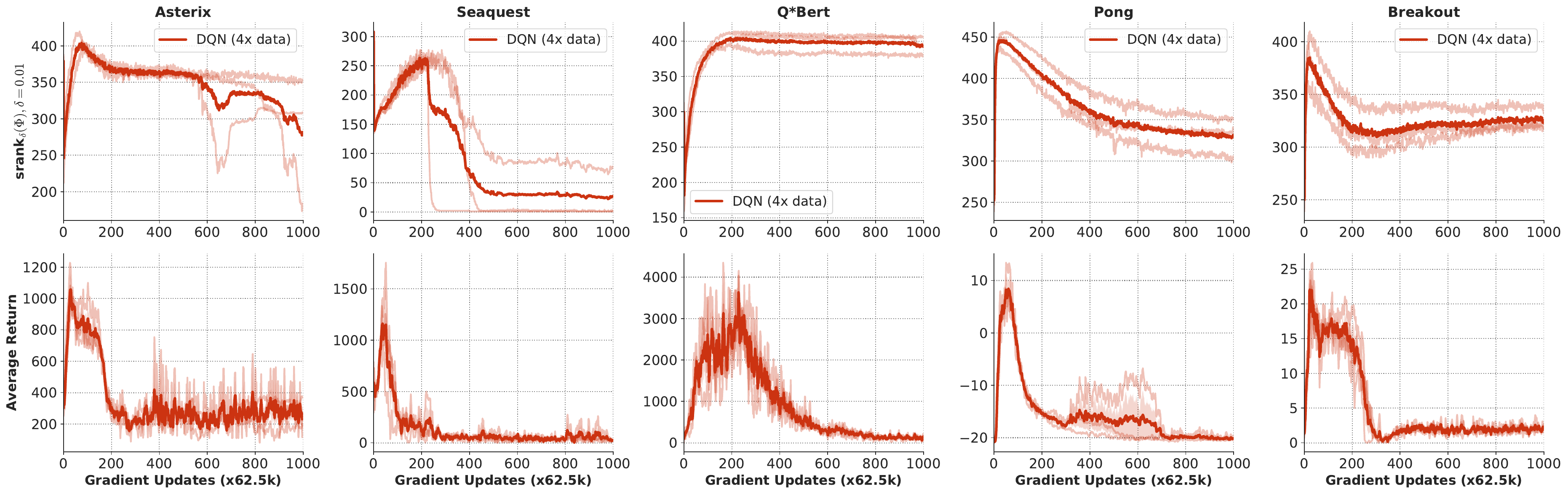}
    \caption{\textbf{Offline DQN on Atari}. $\srank_\delta(\Phi)$ and performance of DQN on five Atari games in the offline RL setting using the 20\% DQN Replay dataset~\citep{agarwal2020optimistic} (marked as \textbf{DQN}) trained for 1000 iterations. Note that low $\srank$ (top row) generally corresponds to worse policy performance (bottom row). Average across 5 runs is showed for each game with individual runs.}
    \label{fig:offline_problem_app_20k}
\end{figure}

\vspace{-5pt}
\begin{figure}[H]
    \centering
    \vspace{-15pt}
    \includegraphics[width=\linewidth]{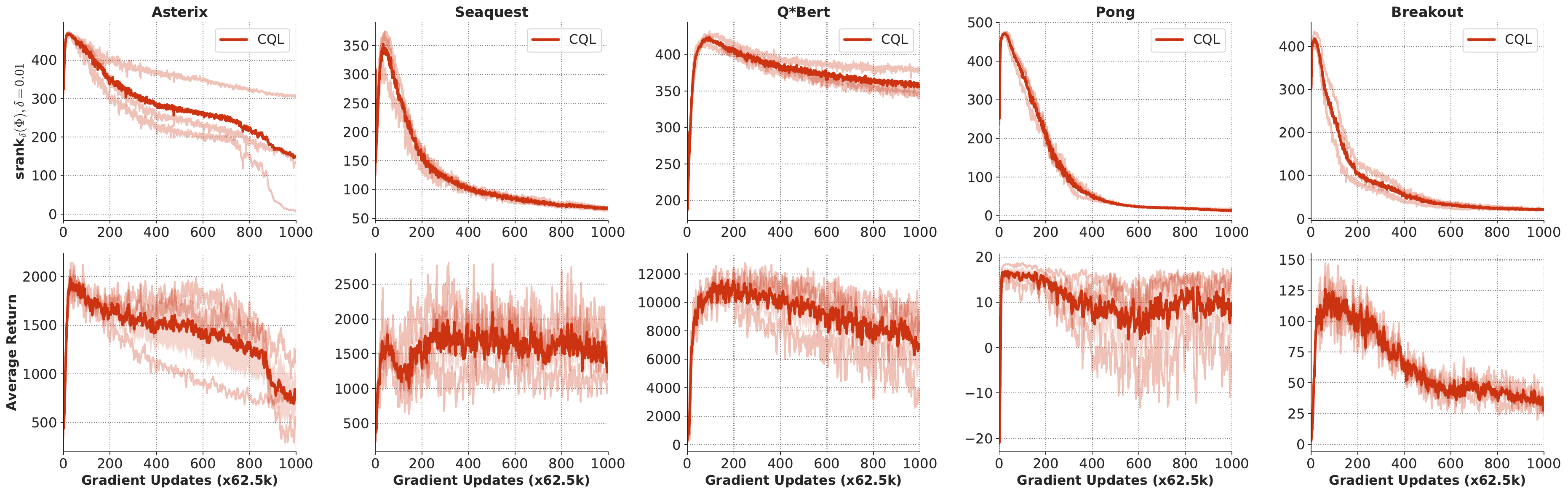}
    \caption{\textbf{Offline CQL on Atari}. $\srank_\delta(\Phi)$ and performance of CQL on five Atari games in the offline RL setting using the 5\% DQN Replay. Note that low $\srank$ (top row) generally corresponds to worse policy performance (bottom row). Average across 5 runs for each game trained for 1000 iterations.}
    \label{fig:offline_problem_cql_app}
\end{figure}

\vspace{-5pt}
\begin{figure}[H]
    \vspace{-15pt}
    \centering
    \includegraphics[width=0.21\linewidth]{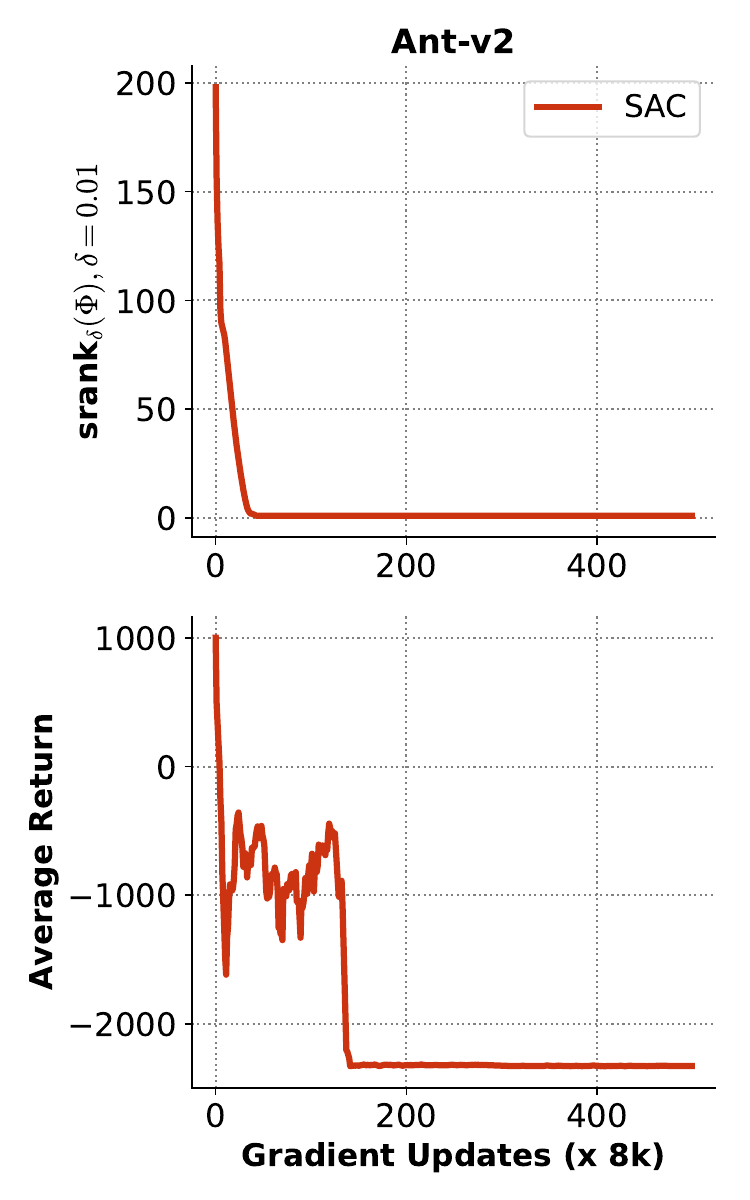}
    \includegraphics[width=0.21\linewidth]{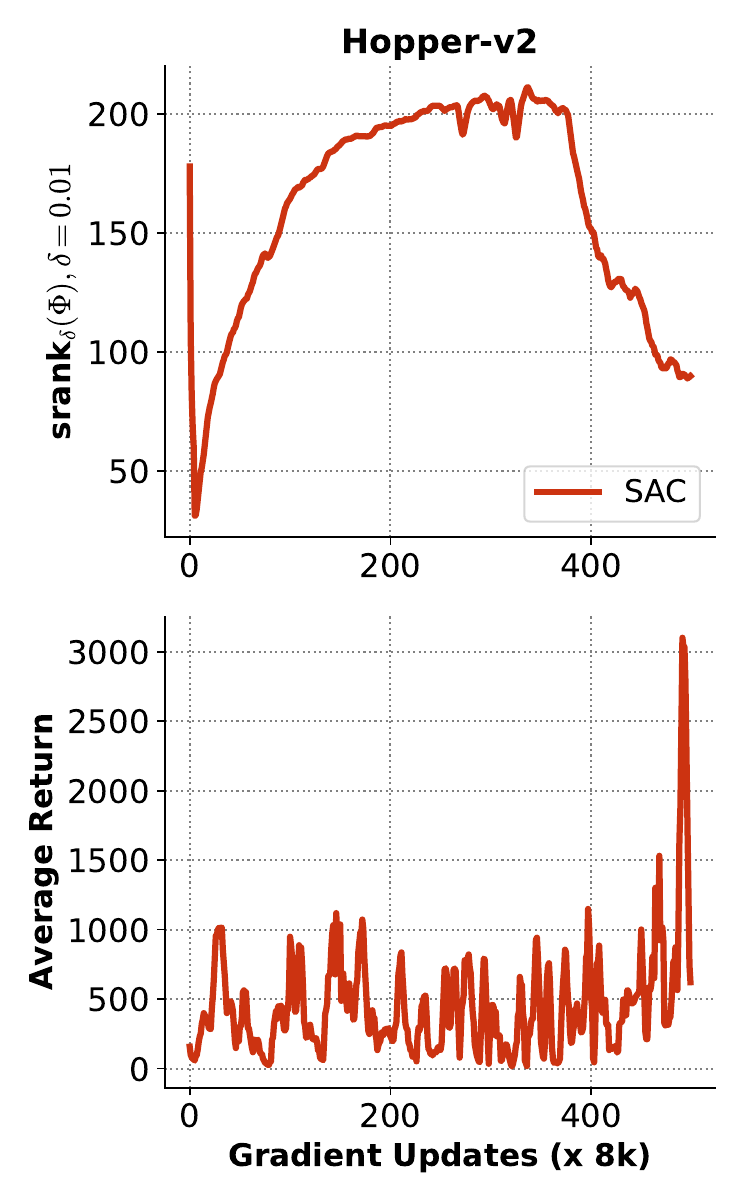}
    \includegraphics[width=0.21\linewidth]{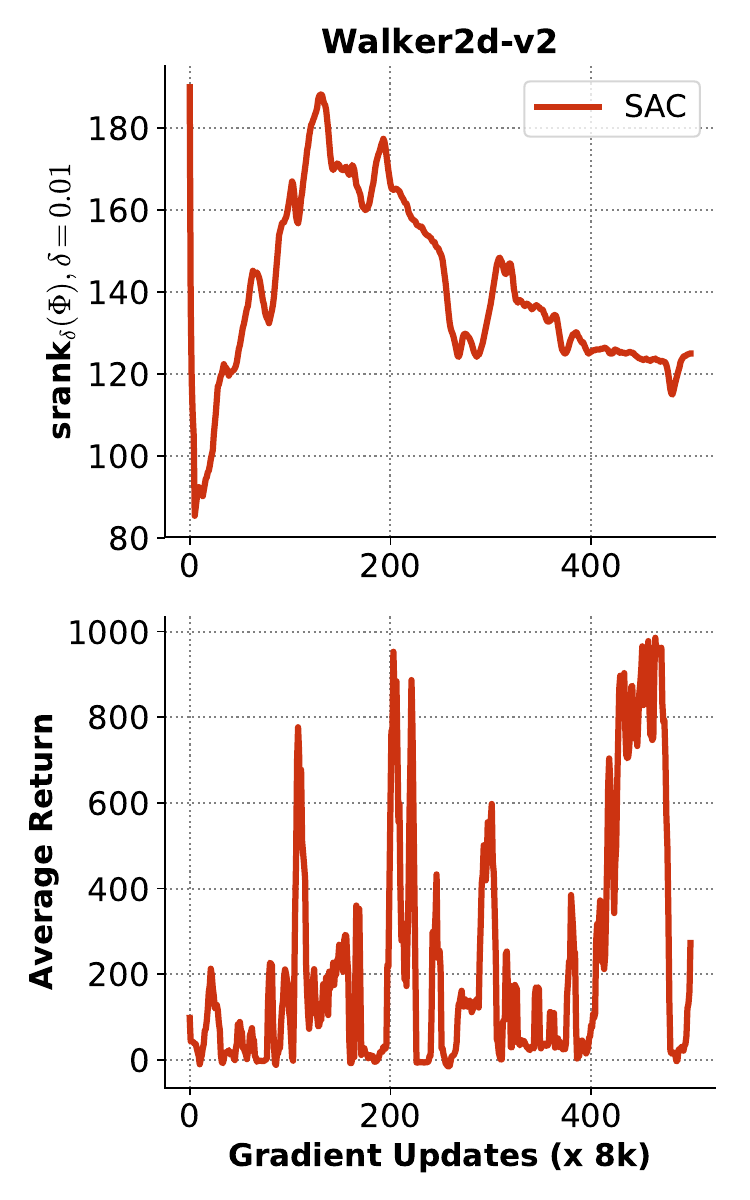}
    \vspace{-7pt}
    \caption{\textbf{Offline Control on MuJoCo}. $\srank_\delta(\Phi)$ and performance of SAC on three Gym environments the offline RL setting. Implicit Under-parameterization is conspicuous from the rank reduction, which highly correlates with performance degradation. We use 20\% uniformly sampled data from the entire replay experience of an online SAC agent, similar to the 20\% setting from \citet{agarwal2020optimistic}.
    }    
    \label{fig:offline_3_mujoco_envs_other}
\end{figure}

\vspace{-5pt}
\begin{figure}[H]
    \vspace{-15pt}
    \centering
    \includegraphics[width=0.21\linewidth]{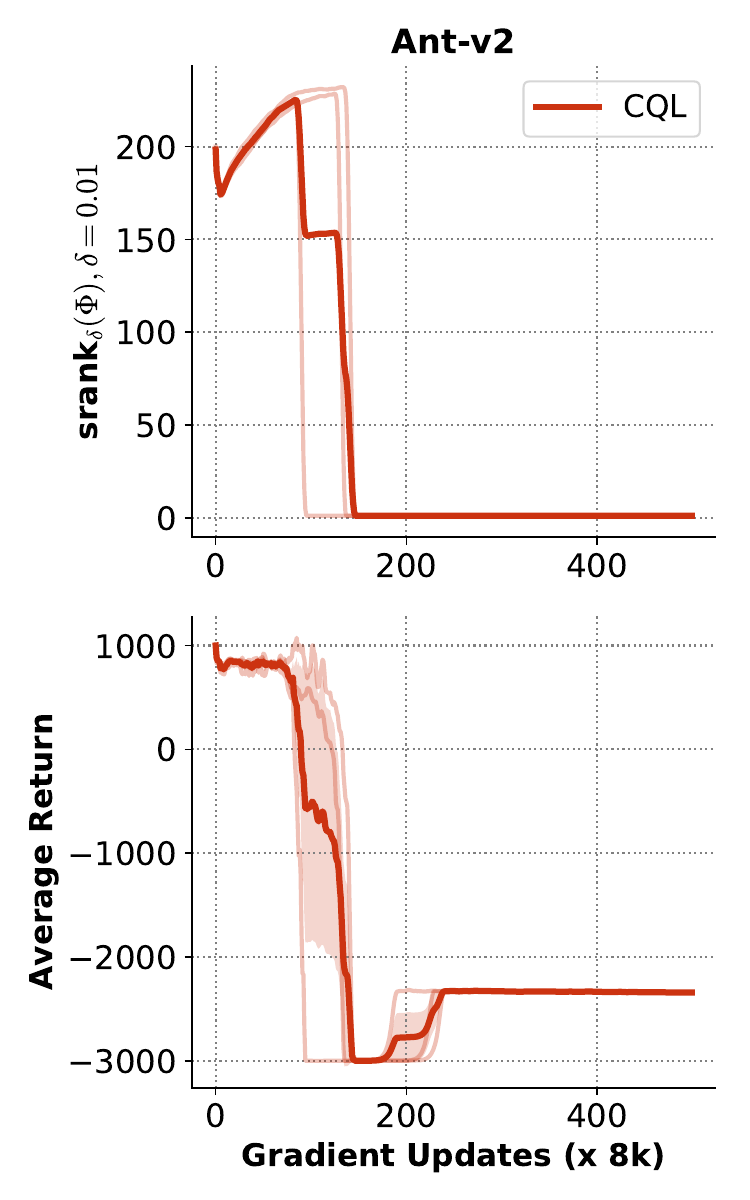}
    \includegraphics[width=0.21\linewidth]{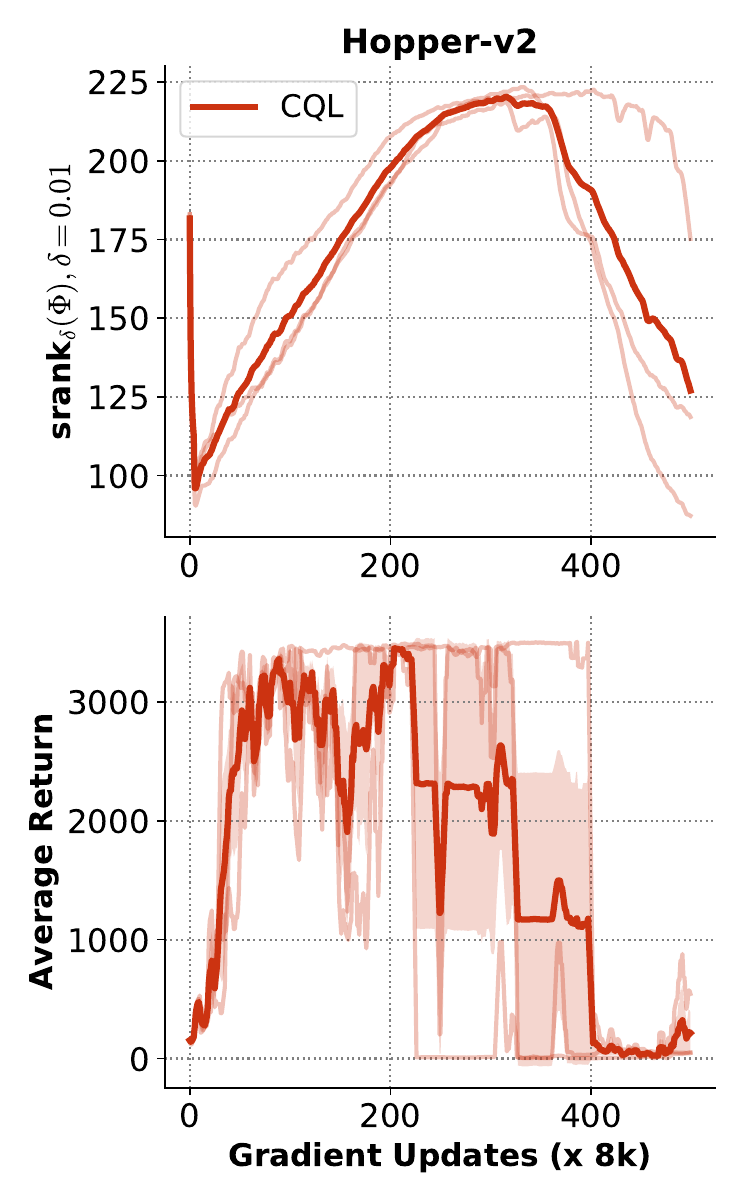}
    \includegraphics[width=0.21\linewidth]{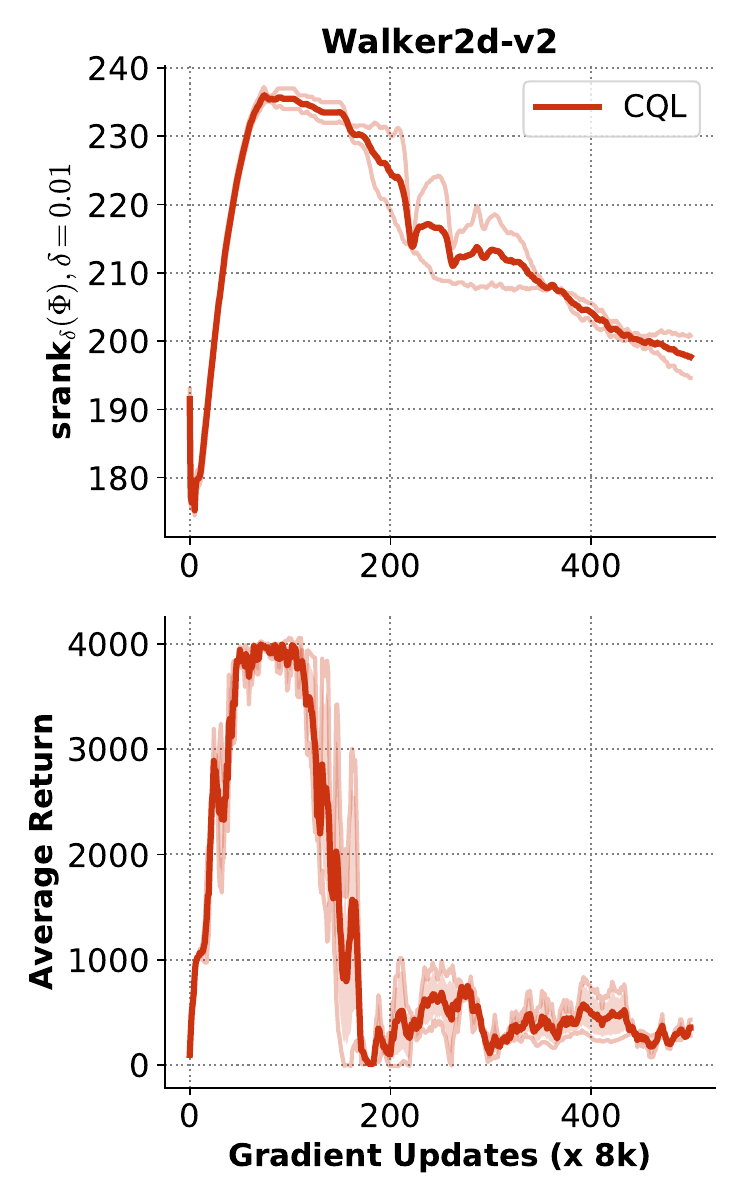}
    \vspace{-7pt}
    \caption{\textbf{Offline Control on MuJoCo}. $\srank_\delta(\Phi)$ and performance of CQL on three Gym environments the offline RL setting. Implicit Under-parameterization is conspicuous from the rank reduction, which highly correlates with performance degradation. We use 20\% uniformly sampled data from the entire replay experience of an online SAC agent, similar to the 20\% setting from \citet{agarwal2020optimistic}.
    }    
    \label{fig:offline_3_mujoco_envs}
    \vspace{-0.5cm}
\end{figure}

\subsection{Data Efficient Online RL}
\label{app:data_efficient_online_rl}
In the data-efficient online RL setting, we verify the presence of implicit under-parameterization on both DQN and Rainbow~\citep{hessel2018rainbow} algorithms when larger number of gradient updates are made per environment step. In these settings we find that more gradient updates per environment step lead to a larger decrease in effective rank, whereas effective rank can increase when the amount of data re-use is reduced by taking fewer gradient steps.
\begin{figure}[H]
    \vspace{-10pt}
    \centering
    \includegraphics[width=\linewidth]{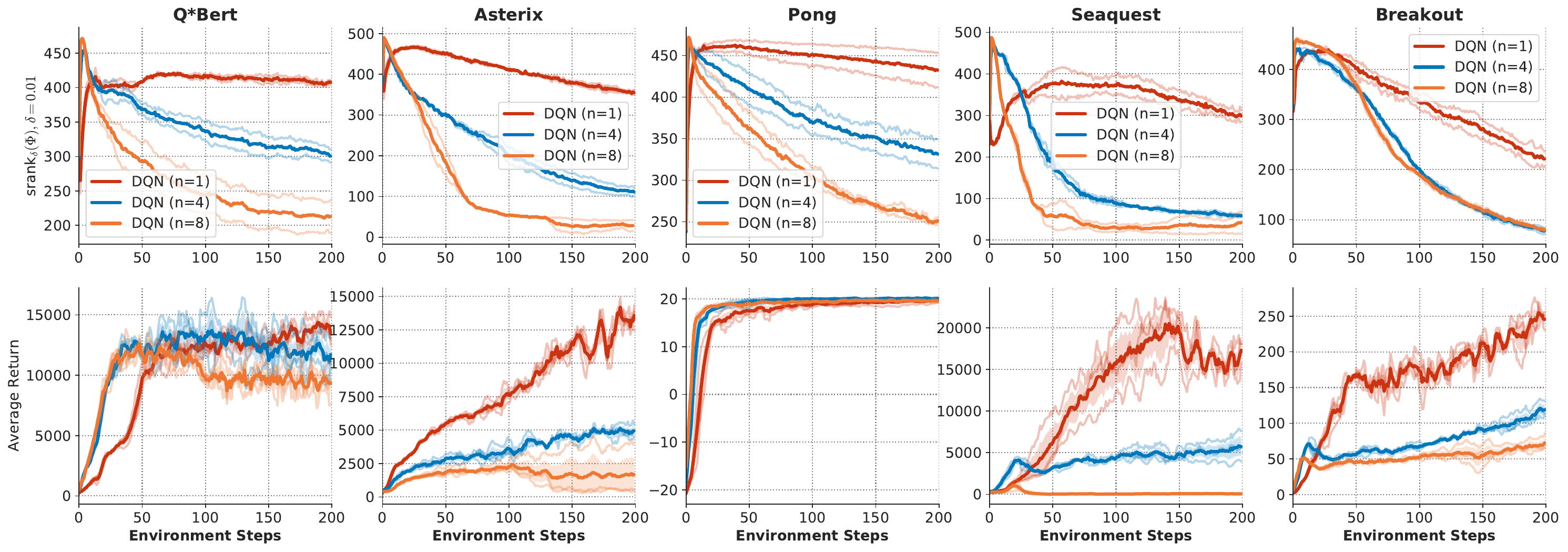}
    \caption{\textbf{Online DQN on Atari}. $\srank_\delta(\Phi)$ and performance of DQN on 5 Atari games in the online RL setting, with varying numbers of gradient steps per environment step ($n$). Rank collapse happens earlier with more gradient steps, and the corresponding performance is poor. This indicates that implicit under-parameterization aggravates as the rate of data re-use is increased.}
    \vspace{-15pt}
    \label{fig:online_5_games_atari}
\end{figure}

\begin{figure}[H]
    \vspace{-15pt}
    \centering
    \includegraphics[width=0.68\linewidth]{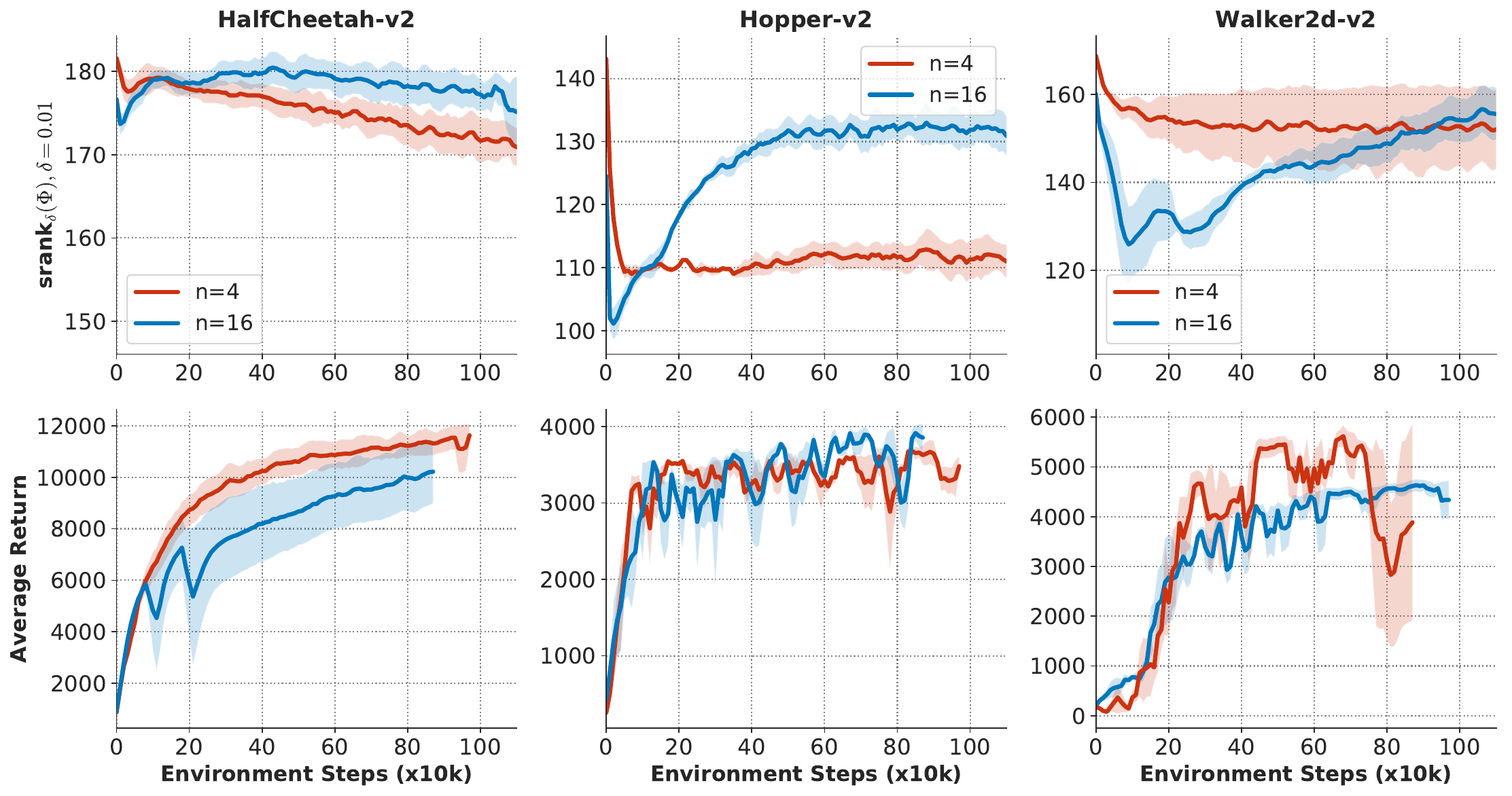}
    \includegraphics[width=0.22\linewidth]{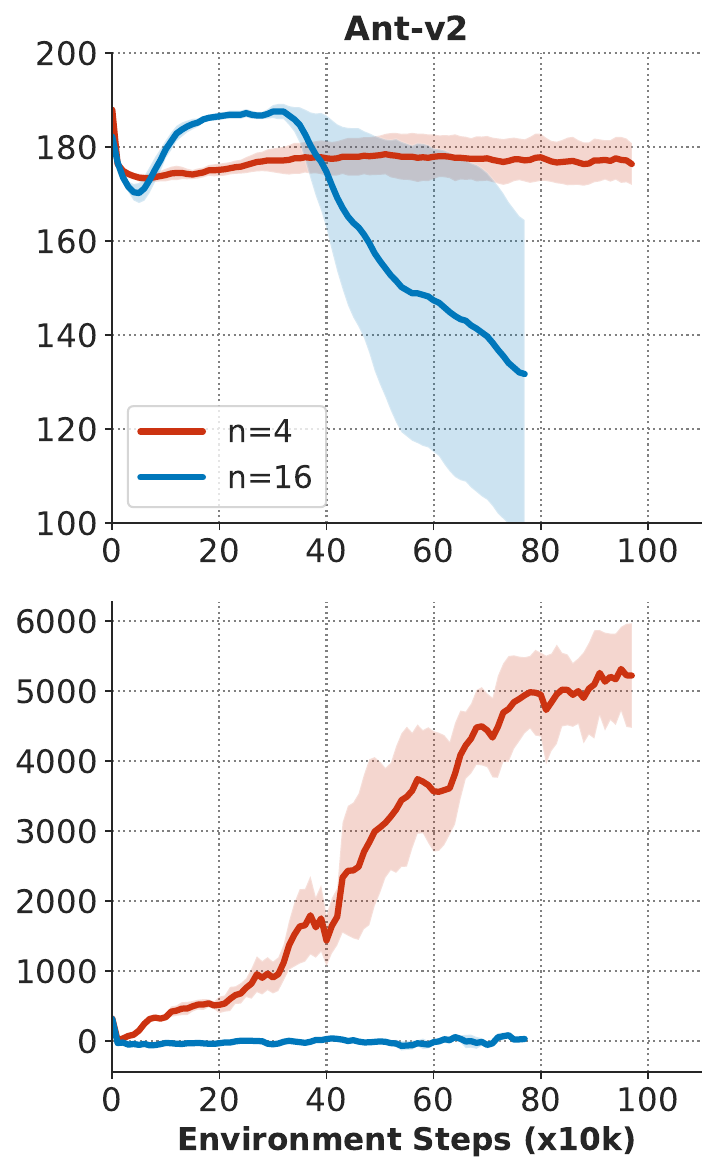}
    \caption{\textbf{Online SAC on MuJoCo}. $\srank_\delta(\Phi)$ and performance of SAC on three Gym environments the online RL setting, with varying numbers of gradient steps per environment step ($n$). While in the simpler environments, HalfCheetah-v2, Hopper-v2 and Walker2d-v2 we actually observe an increase in the values of effective rank, which also corresponds to good performance with large $n$ values in these cases, on the more complex Ant-v2 environment rank decreases with larger $n$, and the corresponding performance is worse with more gradient updates.}
    \label{fig:online_3_mujoco_envs}
\end{figure}

\begin{figure}[H]
    \centering
    \vspace{-8pt}
    \includegraphics[width=\linewidth]{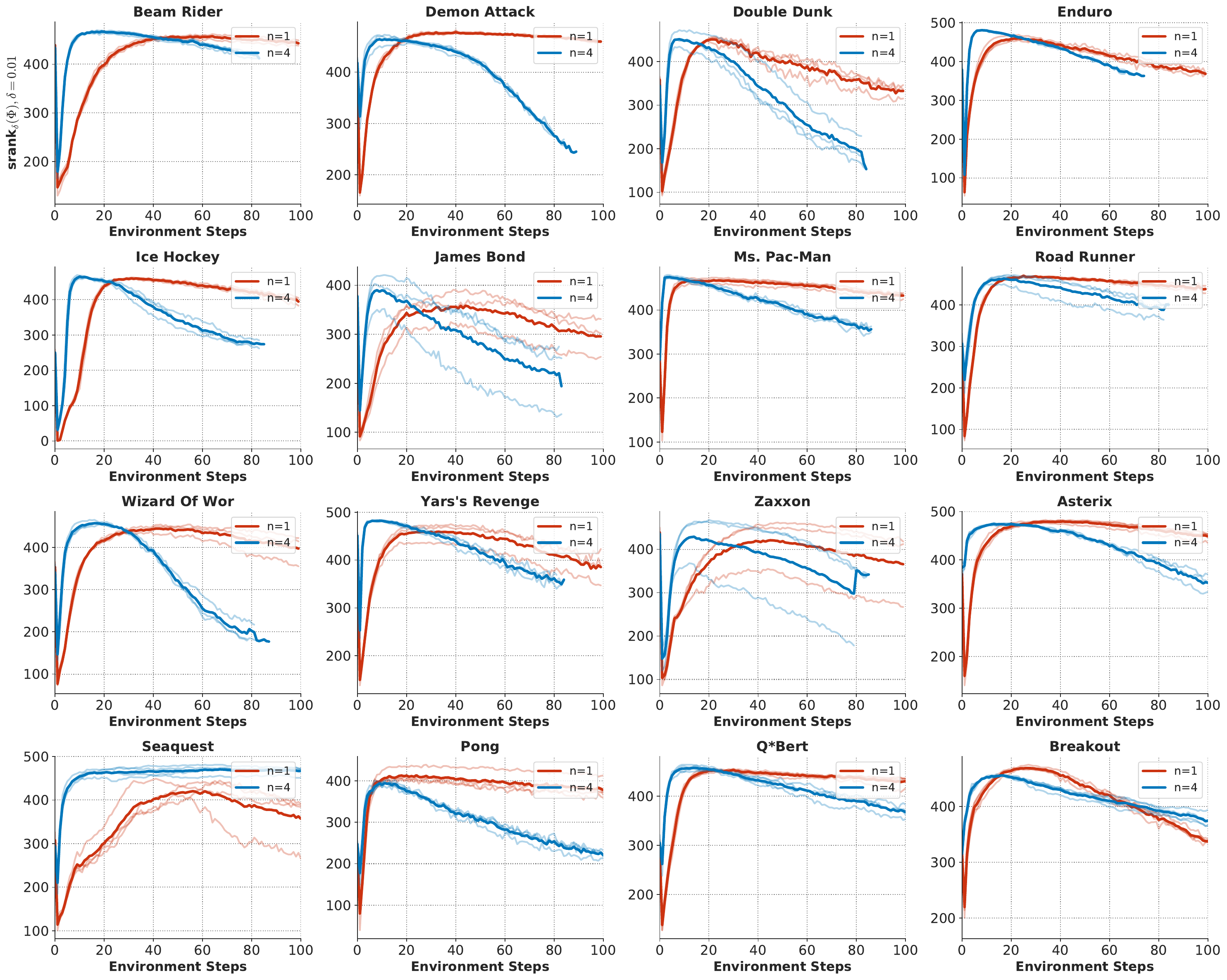}
    \caption{\textbf{Online Rainbow on Atari}. $\srank_\delta(\Phi)$ Rainbow on 16 Atari games in the data-efficient online setting, with varying numbers of gradient steps per environment step ($n$). Rank collapse happens earlier with more gradient steps, and the corresponding performance is poor. This plot indicates the multi-step returns, prioritized replay or distributional C51 does not address the implicit under-parameterization issue.}
    \label{fig:online_rainbow}
\end{figure}

\vspace{-8pt}
\subsection{Does Bootstrapping Cause Implicit Under-Parameterization?}
\label{app:evidence_bootstrap}
\vspace{-5pt}
In this section, we provide additional evidence to support our claim from Section~\ref{sec:problem} that suggests that bootstrapping-based updates are a key component behind the existence of implicit under-parameterization. To do so, we empirically demonstrate the following points empirically:
\begin{itemize}
    \item Implicit under-parameterization occurs even when the form of the bootstrapping update is changed from Q-learning that utilizes a $\max_{a'}$ backup operator to a policy evaluation (fitted Q-evaluation) backup operator, that computes an expectation of the target Q-values under the distributions specified by a different policy. \textbf{Thus, with different bootstrapped updates, the phenomenon still appears.}
    
\begin{figure}[H]
    \vspace{-8pt}
    \centering
    \includegraphics[width=\linewidth]{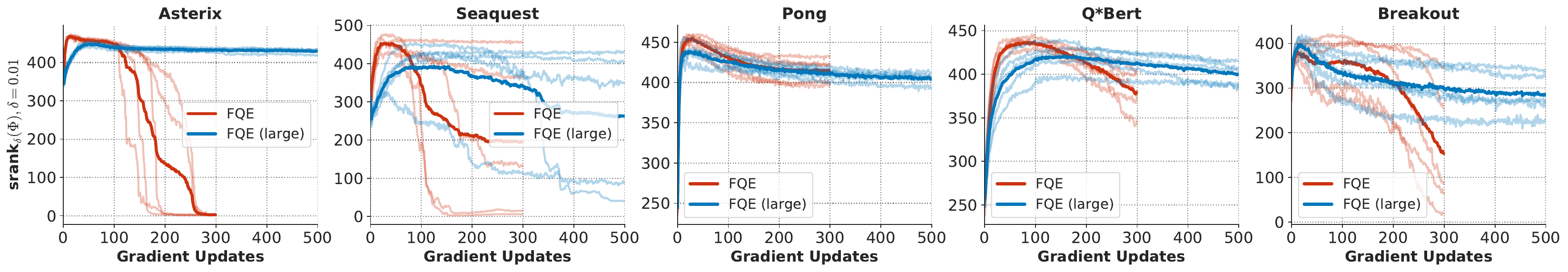}
    \caption{\textbf{Offline Policy Evaluation on Atari}. $\srank_\delta(\Phi)$ and performance of offline policy evaluation~(FQE) on 5 Atari games in the offline RL setting using the 5\% and 20\% DQN Replay dataset~\citep{agarwal2020optimistic}. The rank degradation shows that under-parameterization is not specific to the Bellman optimality operator but happens even when other bootstrapping-based backup operators are combined with gradient descent. Furthermore, the rank degradation also happens when we increase the dataset size.}
    \vspace{-0.2cm}
    \label{fig:offline_policy_eval_5_games}
\end{figure}
    
    \item Implicit under-parameterization does not occur when Monte-Carlo regression targets - that compute regression targets for the Q-function by computing a non-parameteric estimate the future trajectory return, \ie, $\by_k(\bs_t, \ba_t) = \sum_{t' = t}^{\infty} \gamma^t r(\bs_{t'}, \ba_{t'})$ and do not use bootstrapping. In this setting, we find that the values of effective rank actually increase over time and stabilize, unlike the corresponding case for bootstrapped updates. \textbf{Thus, other factors kept identically the same, implicit under-parameterization happens only when bootstrapped updates are used.}

\begin{figure}[H]
    \vspace{-8pt}
    \centering
    \includegraphics[width=\linewidth]{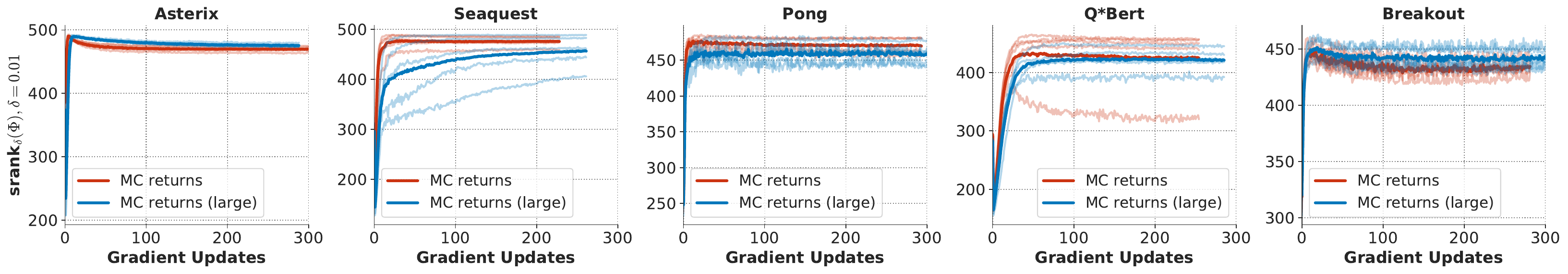}
    \caption{\textbf{Monte Carlo Offline Policy Evaluation}. $\srank_\delta(\Phi)$ on 5 Atari games in when using Monte Carlo returns for targets and thus removing bootstrapping updates. Rank collapse does not happen in this setting implying that is bootstrapping was essential for under-parameterization. We perform the experiments using $5\%$ and $20\%$ DQN replay dataset from \citet{agarwal2020optimistic}.}
    \label{fig:monte_carlo_atari}
    \vspace{-0.3cm}
\end{figure}
\item {For the final point in this section, we verify if the non-stationarity of the policy in the Q-learning (control) setting, i.e., when the Bellman optimality operator is used is not a reason behind the emergence of implicit under-parameterization. The non-stationary policy in a control setting causes the targets to change and, as a consequence, leads to non-zero errors. However, rank drop is primarily caused by bootstrapping rather than non-stationarity of the control objective. To illustrate this, we ran an experiment in the control setting on Gridworld, regressing to the target computed using the true value function $Q^\pi$ for the current policy $\pi$ (computed using tabular Q-evaluation) instead of using the bootstrap TD estimate. The results, shown in \figref{fig:non_stationarity}, show that the $\srank_\delta$ doesn't decrease significantly when regressing to true control values and infact increases with more iterations as compared to Figure~\ref{fig:gridworld_fix_results} where rank drops with bootstrapping. This experiment, alongside with experiments discussed above, ablating bootstrapping in the stationary policy evaluation setting shows that rank-deficiency is caused due to bootstrapping.}
\end{itemize}


\vspace{-10pt}
\subsection{How Does Implicit Regularization Inhibit Data-Efficient RL?}\label{app:regression}
\vspace{-5pt}

\begin{figure}
\centering
\vspace{-20pt}
    \begin{subfigure}{0.38\textwidth}
        \centering
        \includegraphics[width=0.8\linewidth]{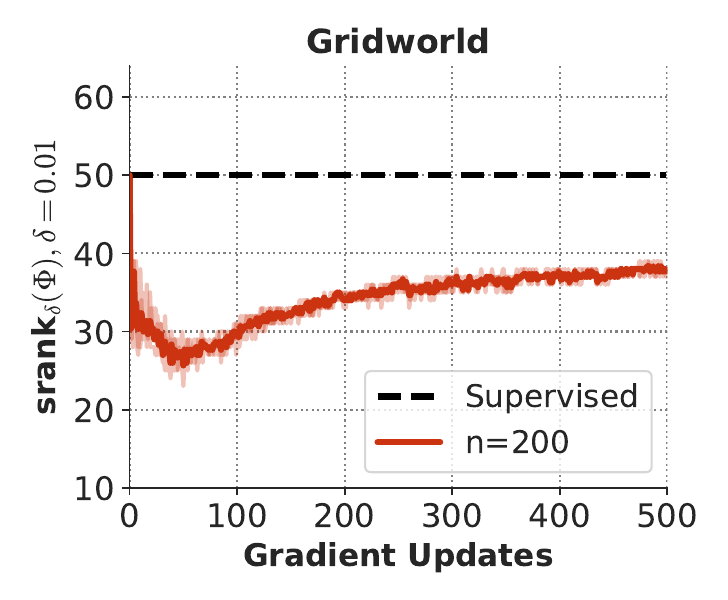}
        \vspace{-7pt}
        \caption{{Using oracle access to exact $Q^\pi$ for computing the target values does not significantly decrease in rank, and training for longer can increase feature rank. This experiment uses a policy iteration style setup but the essential trend of rank drop as we train more does not occur.}}\label{fig:non_stationarity}
        \vspace{-7pt}
    \end{subfigure}~~~~~
    \begin{subfigure}{0.38\textwidth}
        \includegraphics[width=0.65\linewidth]{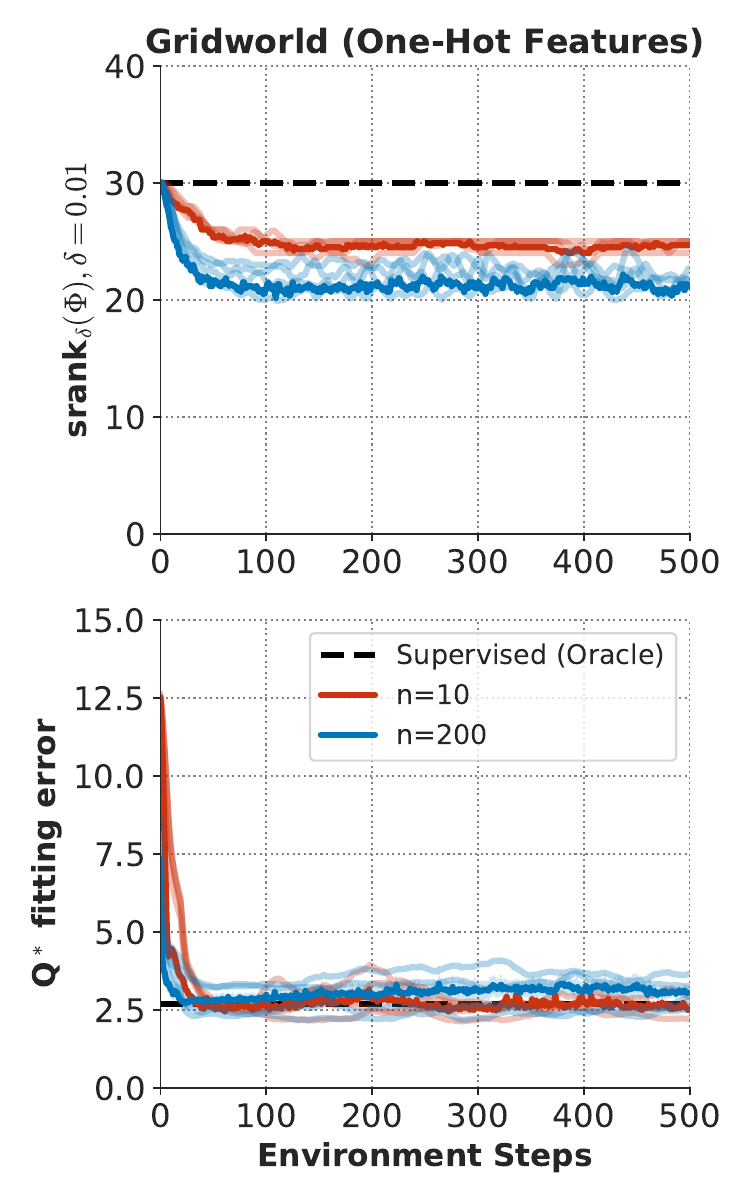}
        \vspace{-7pt}
        \caption{\small{$Q^*$ fitting error and $\srank$ in a one-hot variant of the gridworld environment.}} 
        \label{fig:expressivity_td_onehot}
        \vspace{-7pt}
    \end{subfigure}
    \caption{Gridword: (a) {Effective rank change due to non-stationarity} (b) $Q^*$ fitting error vs. $\srank_\delta$.}
    \vspace{-15pt}
\end{figure}

Implicit under-parameterization leads to a trade-off between minimizing the TD error \emph{vs}. encouraging low rank features as shown in Figure~\ref{fig:expressivity_Td}. This trade-off often results in decrease in effective rank, at the expense of increase in TD error, resulting in lower performance. Here we present additional evidence to support this.

Figure~\ref{fig:expressivity_td_onehot} shows a gridworld problem with one-hot features, which naturally leads to reduced state-aliasing.
In this setting, we find that the amount of rank drop with respect to the supervised projection of oracle computed $Q^*$ values is quite small and the regression error to $Q^*$ actually decreases unlike the case in Figure~\ref{fig:expressivity}, where it remains same or even increases. The method is able to learn policies that attain good performance as well. Hence, this justifies that when there's very little rank drop, for example, 5 rank units in the example on the right, FQI methods are generally able to learn $\Phi$ that is able to fit $Q^*$. This provides evidence showing that typical Q-networks learn $\Phi$ that can fit the optimal Q-function when rank collapse does not occur.

In Atari, we do not have access to $Q^*$, and so we instead measure the error in fitting the target value estimates, $\bR + \gamma P^\pi \bQ_{k}$. As rank decreases, the TD error increases~(\Figref{fig:td_err_appendix}) and the value function is unable to fit the target values, culminating in a performance plateau~(\Figref{fig:online_5_games_atari}).

\begin{figure}[H]
    \vspace{-11pt}
    \centering
    \includegraphics[width=\linewidth]{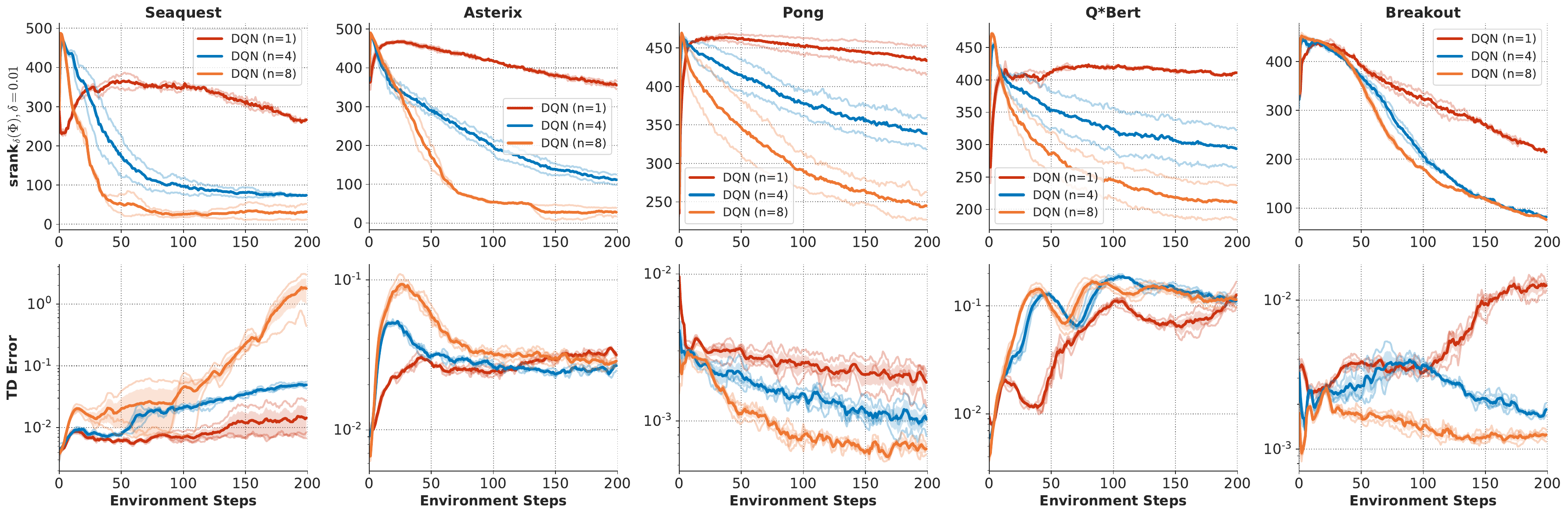}
    \caption{\textbf{TD error \emph{vs}. Effective rank on Atari}. We observe that Huber-loss TD error is often higher when there is a larger implicit under-parameterization, measured in terms of drop in effective rank. The results are shown for the data-efficient online RL setting.}
    \label{fig:td_err_appendix}
    \vspace{-15pt}
\end{figure}

\subsection{Trends in Values of Effective Rank With Penalty.}\label{app:fix_results}
In this section, we present the trend in the values of the effective rank when the penalty $\gL_p(\Phi)$ is added. In each plot below, we present the value of $\srank_\delta(\Phi)$ with and without penalty respectively.

\vspace{-5pt}
\subsubsection{Offline RL: DQN}
\begin{figure}[H]
\centering
\vspace{-10pt}
\includegraphics[width=0.99\linewidth]{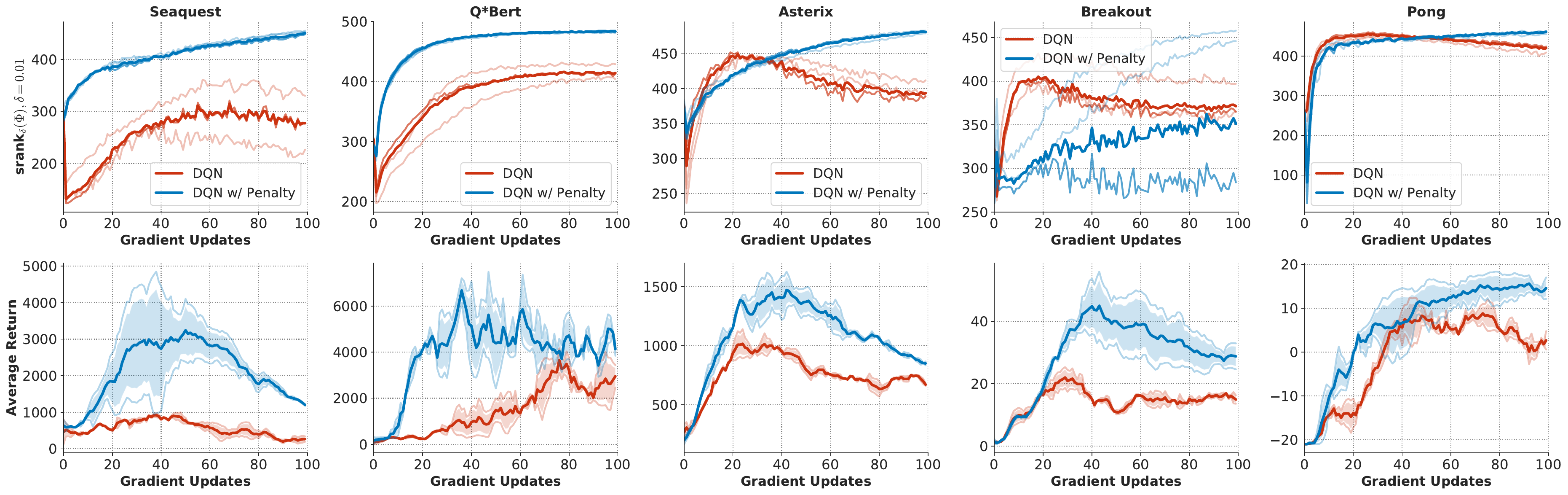}
\vspace{-5pt}
\caption{\textbf{Effective rank values with the penalty on DQN.} Trends in effective rank and performance for offline DQN. Note that the performance of DQN with penalty is generally better than DQN and that the penalty (blue) is effective in increasing the values of effective rank. We report performance at the end of 100 epochs, as per the protocol set by \citet{agarwal2020optimistic} in Figure~\ref{fig:atari_offline_results}.}
\vspace{-10pt}
\end{figure}

\begin{figure}[H]
\centering
\includegraphics[width=0.99\linewidth]{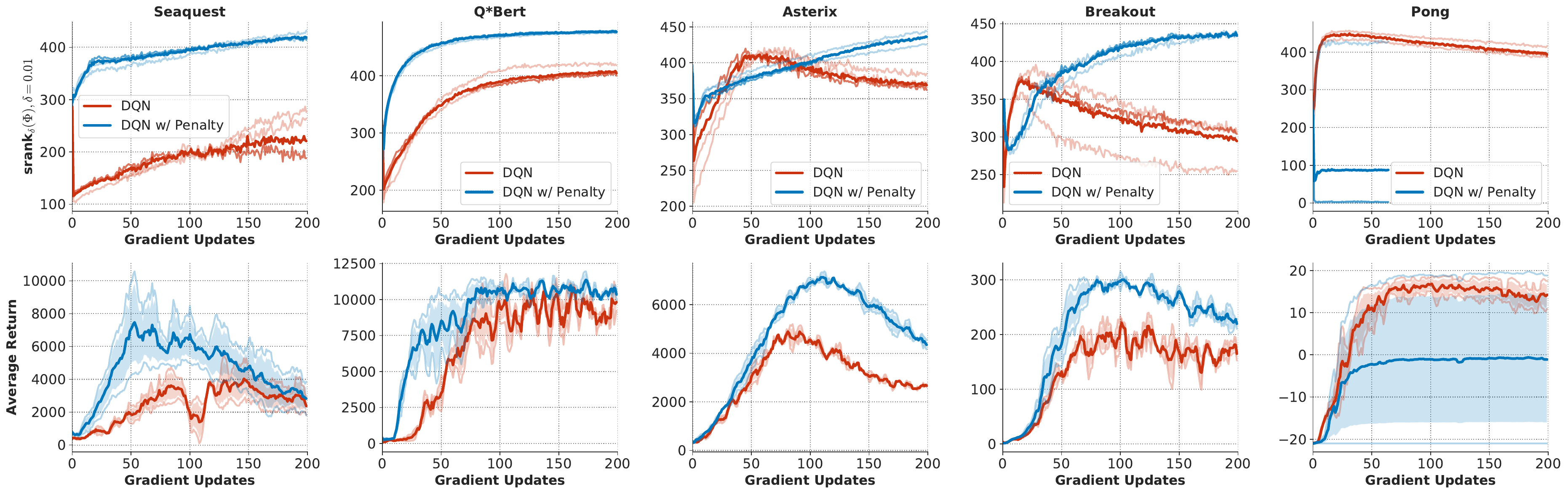}
\vspace{-5pt}
\caption{\textbf{Effective rank values with the penalty on DQN on a 4x larger dataset.} Trends in effective rank and performance for offline DQN with a 4x larger dataset, where distribution shift effects are generally removed. Note that the performance of DQN with penalty is generally better than DQN and that the penalty (blue) is effective in increasing the values of effective rank in most cases. Infact in \textsc{Pong}, where the penalty is not effective in increasing rank, we observe suboptimal performance (blue \vs red).}
\vspace{-5pt}
\end{figure}

\begin{figure}[H]
\centering
\begin{subfigure}{0.48\linewidth}
    \centering
    \includegraphics[width=\linewidth]{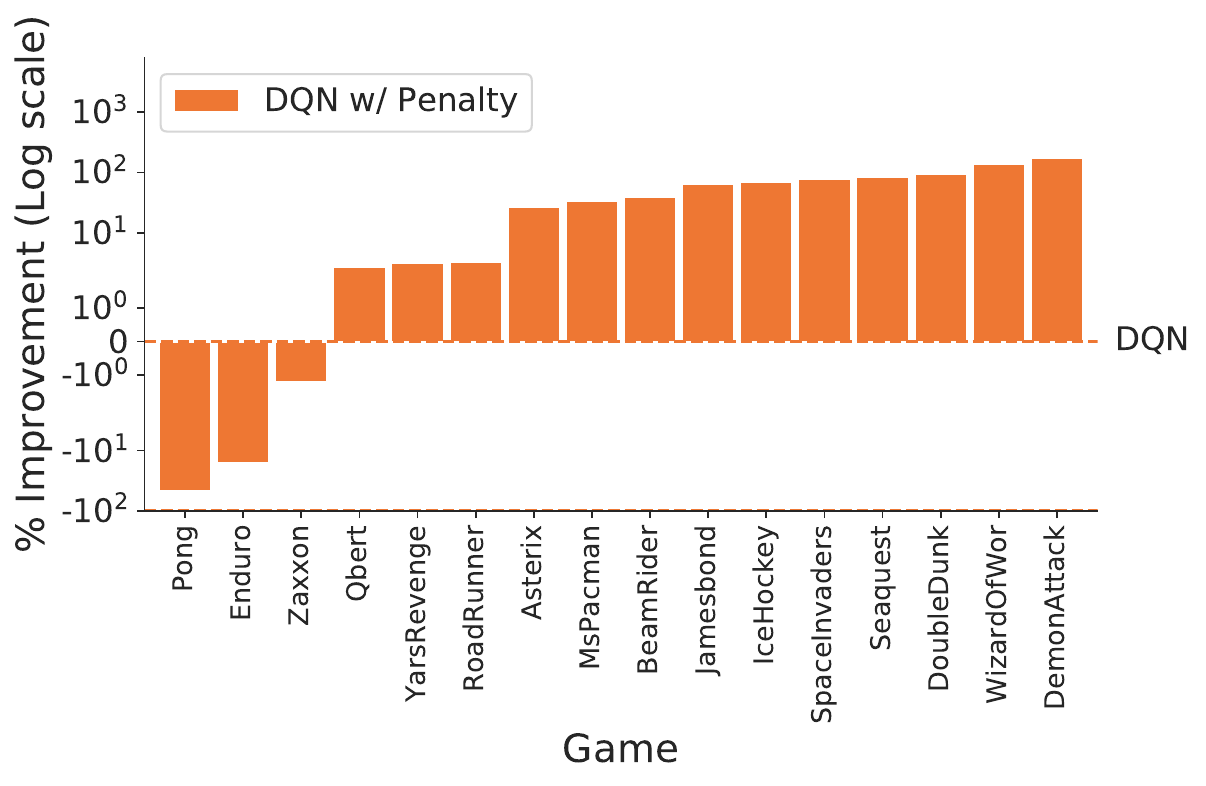}
    \caption{Offline DQN with $\gL_p(\Phi)$.}
    \end{subfigure}
    \begin{subfigure}{0.48\linewidth}
    \centering
    \includegraphics[width=\linewidth]{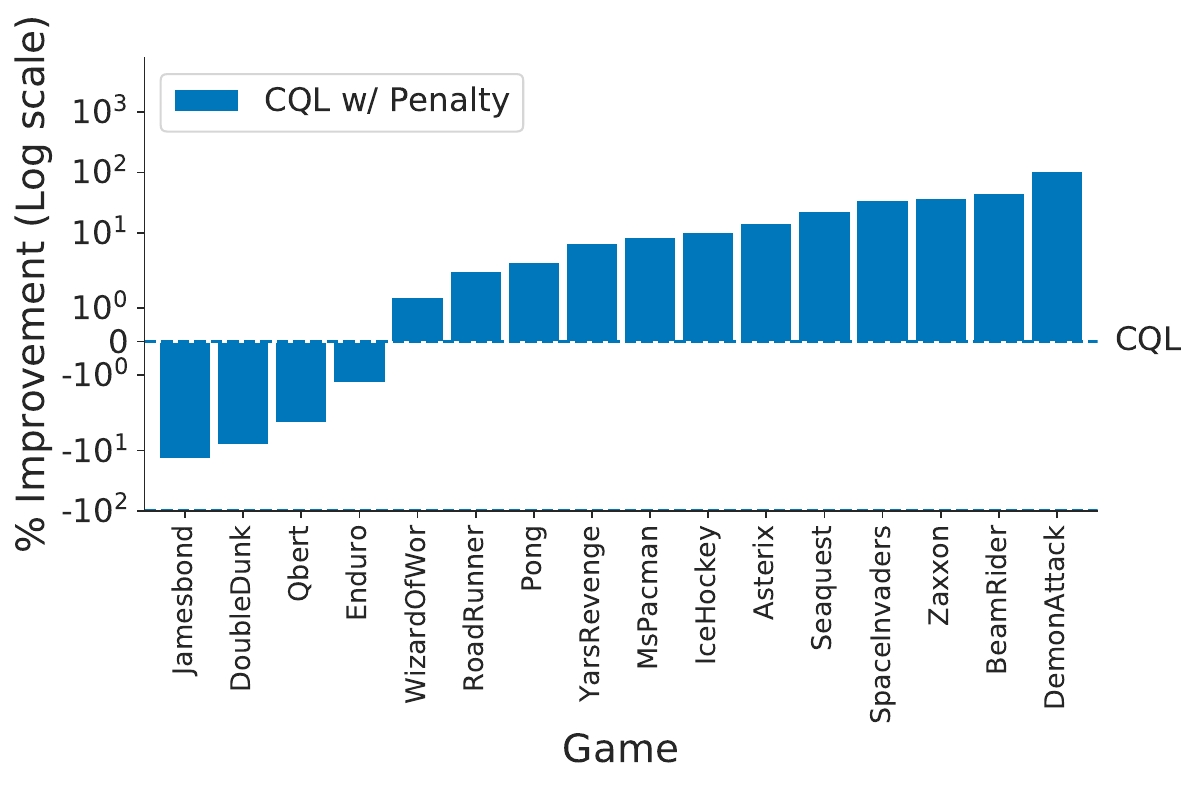}
    \vspace{-5pt}
    \caption{Offline CQL with $\gL_p(\Phi)$.}
    \end{subfigure}
    \caption{Performance improvement of (a) offline DQN and (b) offline CQL with the $\gL_p(\Phi)$ penalty on 20\% Atari dataset, i.e., the dataset referred to as \textbf{4x} large in Figure~\ref{fig:offline_problem}.}
\vspace{-10pt}
\end{figure}

\subsubsection{Offline RL: CQL With $L_p(\Phi)$ Penalty}

\begin{figure}[H]
\centering
\vspace{-5pt}
\includegraphics[width=0.99\linewidth]{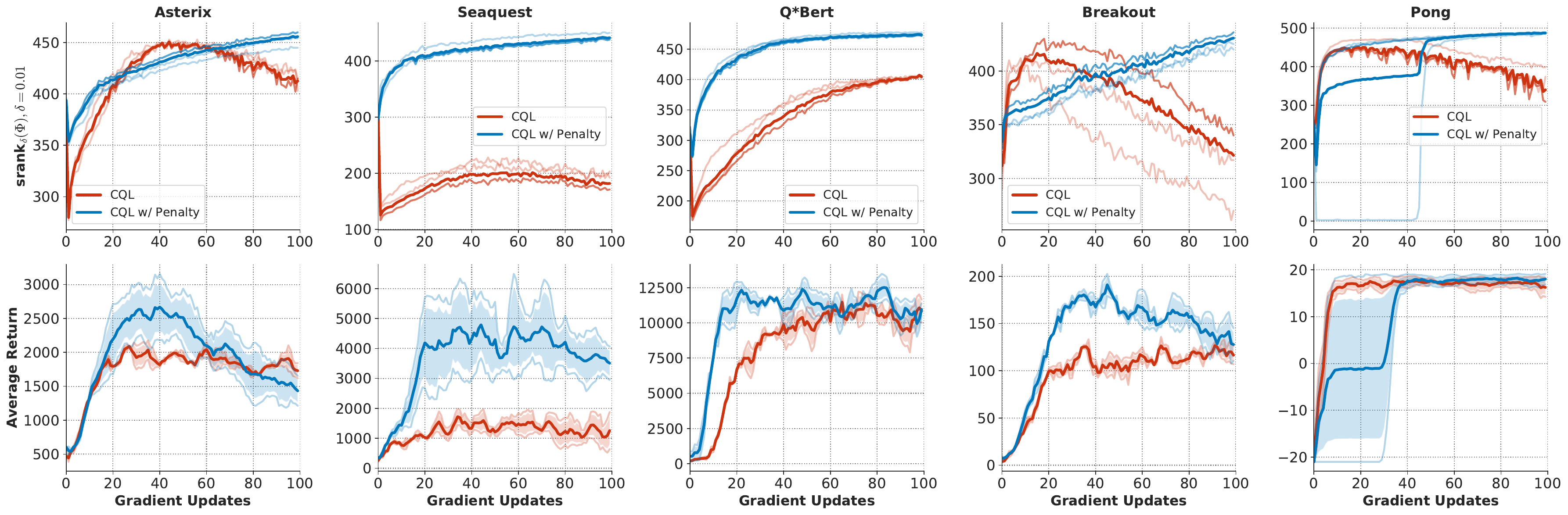}
\caption{\textbf{Effective rank values with the penalty $L_p(\Phi)$ on CQL.} Trends in effective rank and performance for offline DQN. Note that the performance of CQL with penalty is generally better than vanilla CQL and that the penalty (blue) is effective in increasing the values of effective rank. We report performance at the end of 100 epochs, as per the protocol set by \citet{agarwal2020optimistic} in Figure~\ref{fig:atari_offline_results}.}
\end{figure}

\subsection{Data-Efficient Online RL:  Rainbow}\label{sec:fix_rainbow_results}

\subsubsection{Rainbow With $L_p(\Phi)$ Penalty: Rank Plots}
\begin{figure}[H]
\centering
\vspace{-5pt}
\includegraphics[width=0.85\linewidth]{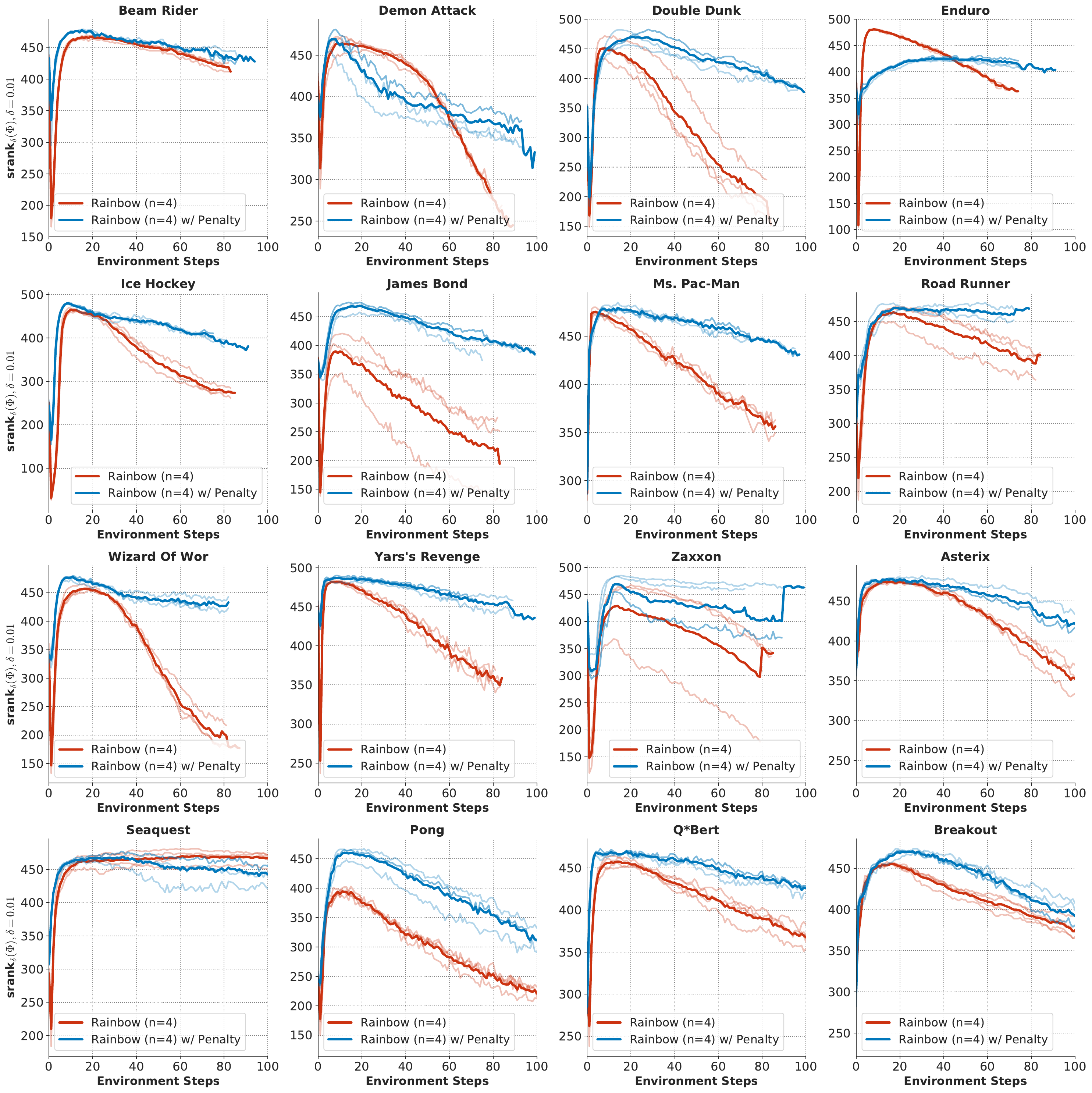}
\vspace{-5pt}
\caption{\textbf{Effective rank values with the penalty on Rainbow in the data-efficient online RL setting.} Trends in effective rank and performance for online Rainbow, where distribution shift effects are generally removed. Note that the performance of DQN with penalty is generally better than DQN and that the penalty (blue) is effective in increasing the values of effective rank in most cases. Infact in \textsc{Pong}, where the penalty is not effective in increasing rank, we observe suboptimal performance (blue \vs red).}
\vspace{-5pt}
\end{figure}

\subsubsection{Rainbow With $L_p(\Phi)$ Penalty: Performance}
\begin{wrapfigure}{r}{0.54\linewidth}
\centering
\vspace{-20pt}
\includegraphics[width=\linewidth]{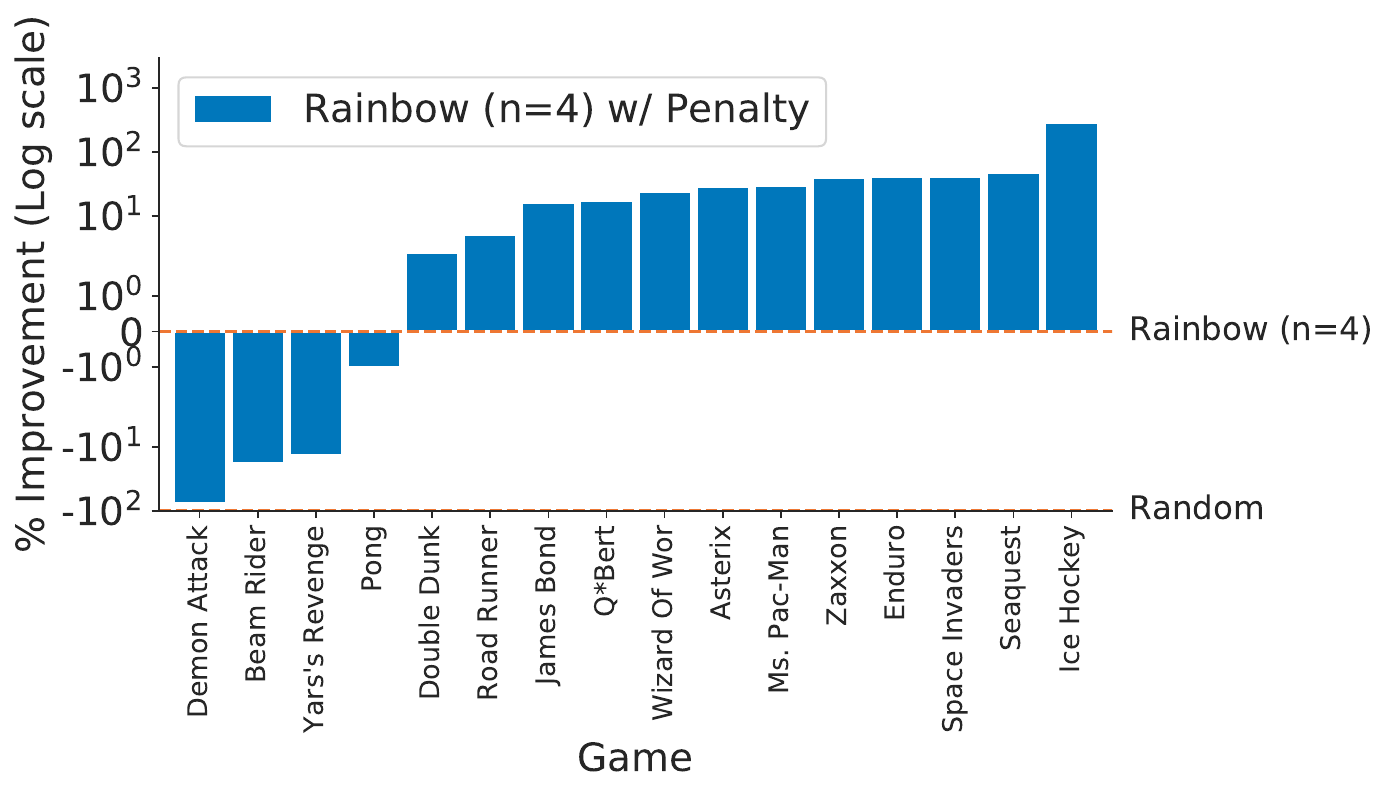}
\vspace{-10pt}
\caption{Performance of Rainbow $(n=4)$ and Rainbow $(n=4)$ with the $\gL_p(\Phi)$ penalty (Equation~\ref{eqn:penalties}. Note that the penalty improves on the base Rainbow in \textbf{12/16} games.}
\vspace{-20pt}
\end{wrapfigure}

In this section, we present additional results for supporting the hypothesis that preventing rank-collapse leads to better performance. In the first set of experiments, we apply the proposed $\gL_p$ penalty to Rainbow  in the data-efficient online RL setting $(n=4)$. In the second set of experiments, we present evidence for prevention of rank collapse by comparing rank values for different runs.

As we will show in the next section, the state-of-the-art Rainbow~\citep{hessel2018rainbow} algorithm also suffers form rank collapse in the data-efficient online RL setting when more updates are performed per gradient step. In this section, we applied our penalty $\gL_p$ to Rainbow with $n=4$, and obtained a median \textbf{20.66\%} improvement on top of the base method. This result is summarized below.

\begin{figure}[H]
\centering
\includegraphics[width=0.95\linewidth]{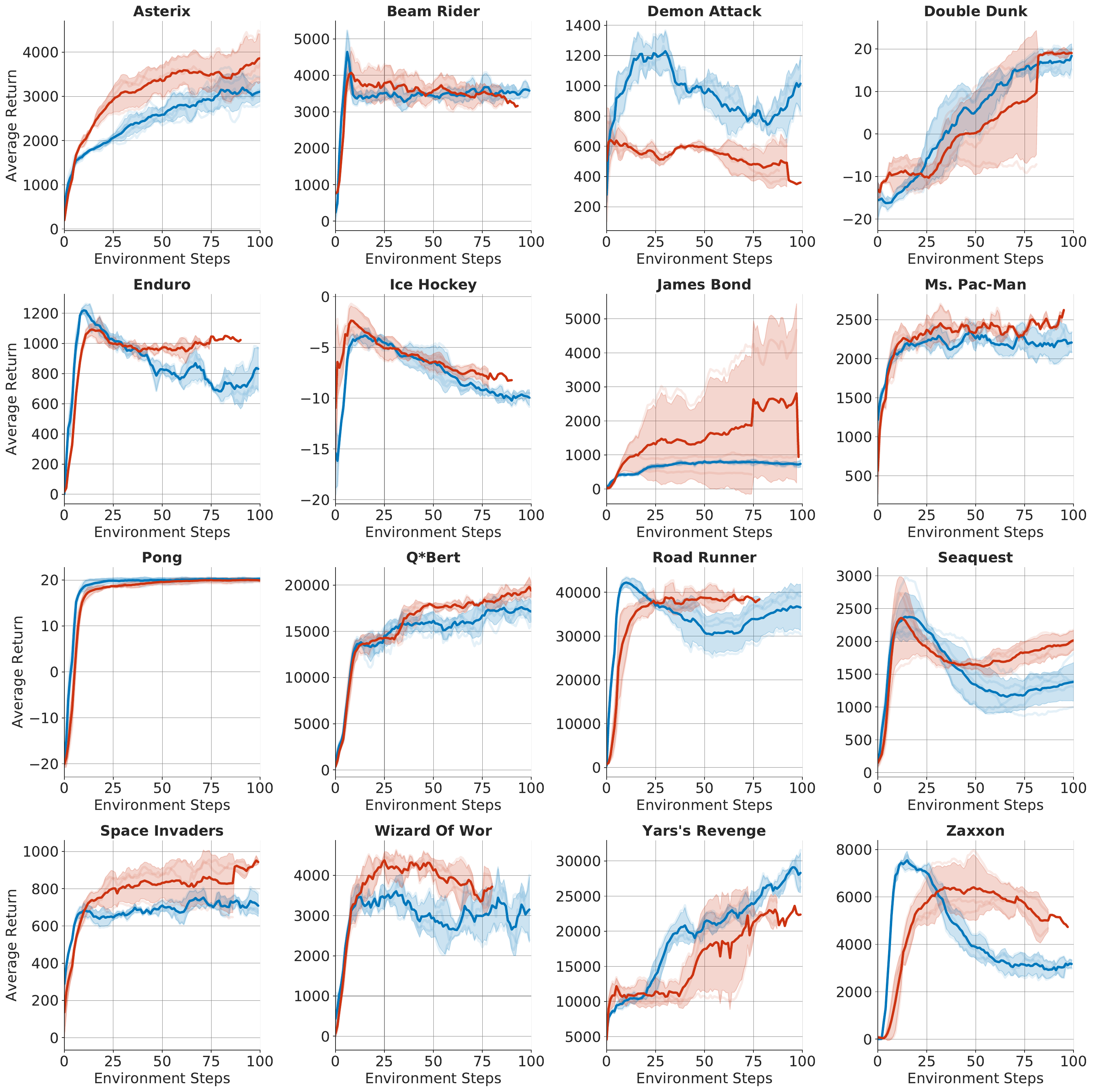}
\caption{Learning curves with $n=4$ gradient updates per environment step for Rainbow\textbf{(Blue)} and Rainbow  with the $\gL_p(\Phi)$ penalty (Equation~\ref{eqn:penalties}) \textbf{(Red)} on 16 games , corresponding to the bar plot above. One unit on the x-axis is equivalent to 1M environment steps.}
\end{figure}

\vspace{-0.2cm}
\subsection{Relaxing the Normality Assumption in Theorem~\ref{thm:self_distillation}}
\label{app:eigval_srank}

We can relax the normality assumption on $\bS$ in Theorem~\ref{thm:self_distillation}. An analogous statement holds for non-normal matrices $\bS$ for a slightly different notion of effective rank, denoted as $\srank_{\delta, \lambda}(\bM_k)$, that utilizes eigenvalue norms instead of singular values. Formally, let $\lambda_1(\bM_k), \cdots, \lambda_2(\bM_k), \cdots$ be the (complex) eigenvalues of $\bM_k$ arranged in decreasing order of their norms, \ie, $|\lambda_1(\bM_k)| \geq |\lambda_2(\bM_k)| \geq \cdots$, then,
\begin{equation*}
    \small{\srank_{\delta, \lambda}(\bM_k) = \min \left\{ k: \frac{\sum_{i=1}^k |\lambda_{i}(\bM_k)|}{\sum_{i=1}^{d} |\lambda_i(\bM_k)|} \geq 1 - \delta \right\}.}
\end{equation*}
A statement essentially analogous to Theorem~\ref{thm:self_distillation} suggests that in this general case, 
$\srank_{\delta, \lambda}(\bM_k)$ decreases for \textit{all} (complex) diagonalizable matrices $\bS$, which is the set of almost all matrices of size $\text{dim}(\bS)$. Now, if $\bS$ is approximately normal, \ie when $|\sigma_i(\bS) - |\lambda_i(\bS)||$ is small, then the result in Theorem 4.1 also holds approximately as we discuss at the end of Appendix~\ref{app:self_distill_proofs}.

We now provide empirical evidence showing that the trend in the values of effective rank computed using singular values, $\srank_\delta(\Phi)$ is almost identical to the trend in the effective rank computed using normalized eigenvalues, $\srank_{\delta, \lambda}(\Phi)$. Since eigenvalues are only defined for a square matrix $\Phi$, in practice, we use a batch of $d = \text{dim}(\phi(\bs, \ba))$ state-action pairs for computing the eigenvalue rank and compare to the corresponding singular value rank in Figures~\ref{fig:rank_notions} and \ref{fig:gridworld_eigen_rank}.

\begin{figure}[h]
    \centering
    \includegraphics[width=\linewidth]{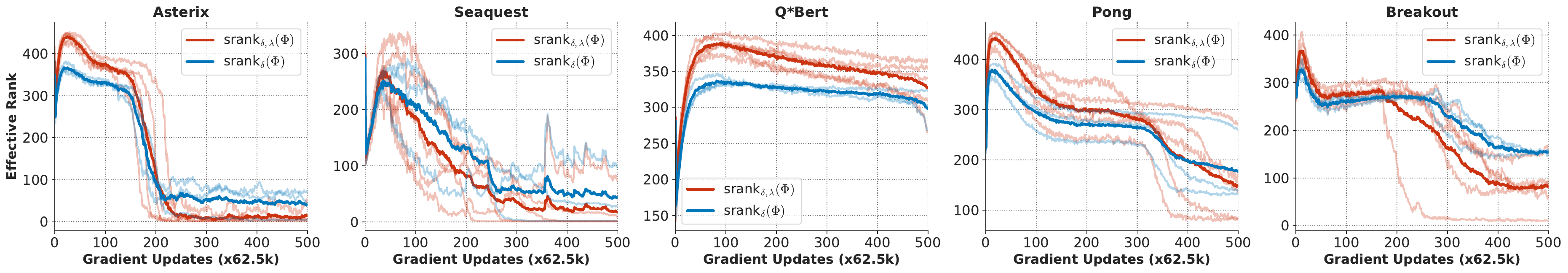}
    \caption{{\textbf{Comparing different measures of effective rank} on a run of offline DQN in the 5\% replay setting, previously studied in Figure~\ref{fig:offline_problem_cql_app}.}}
    \label{fig:rank_notions}
\end{figure}

{\textbf{Connection to Theorem~\ref{thm:self_distillation}.} We computed the effective rank of $\Phi$ instead of $\bS$, since $\bS$ is a theoretical abstraction that cannot be computed in practice as it depends on the Green's kernel~\citep{duffy2015green} obtained by assuming that the neural network behaves as a kernel regressor. Instead, we compare the different notions of ranks of $\Phi$ since $\Phi$ is the practical counterpart for the matrix, $\bS$, when using neural networks (as also indicated by the analysis in Section~\ref{sec:grad_descent}). In fact, on the gridworld (Figure~\ref{fig:gridworld_eigen_rank}), we experiment with a feature $\Phi$ with dimension equal to the number of state-action pairs, \ie $\text{dim}(\phi(\bs, \ba)) = |\mathcal{S}||\mathcal{A}|$, with the same number of parameters as a kernel parameterization of the Q-function: $Q(\bs, \ba) = \sum_{\bs', \ba'} \bw(\bs', \ba') k(\bs, \ba, \bs', \ba')$. This can also be considered as performing gradient descent on a ``wide'' linear network , and we measure the feature rank while observing similar rank trends.}

{Since we do not require the assumption that $\bS$ is normal in Theorem~\ref{thm:self_distillation} to obtain a decreasing trend in $\srank_{\delta, \lambda}(\Phi)$, and we find that in practical scenarios (Figures \ref{fig:rank_notions} and \ref{fig:gridworld_eigen_rank}), $\srank_{\delta}(\Phi) \approx \srank_{\delta, \lambda}(\Phi)$ with an extremely similar qualitative trend we believe that Theorem~\ref{thm:self_distillation} still explains the rank-collapse practically observed in deep Q-learning and is not vacuous.}

\begin{figure}[h]
    \centering
    \includegraphics[width=0.23\textwidth]{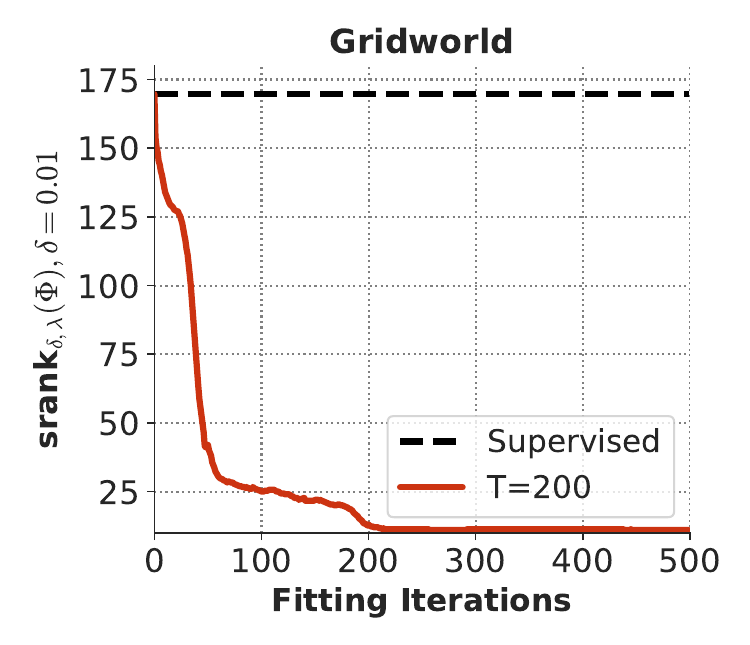}
    \includegraphics[width=0.23\textwidth]{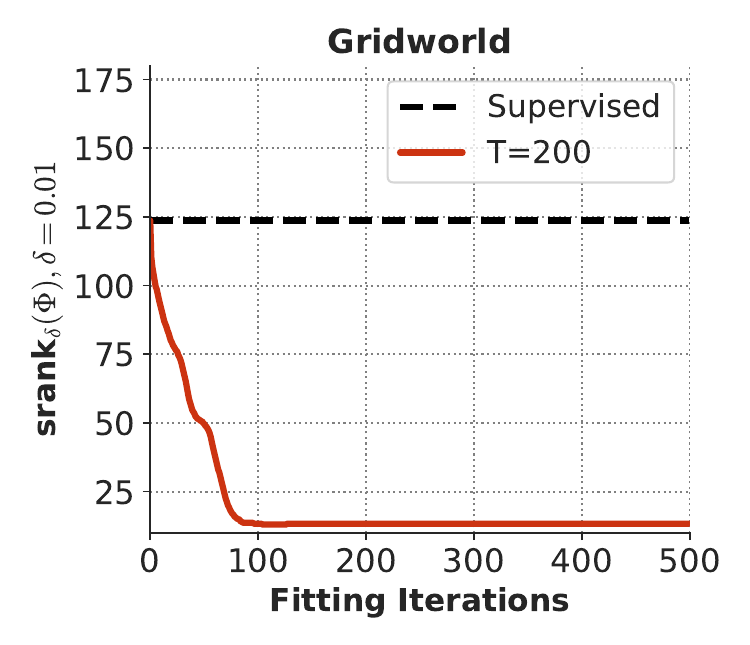}
    ~\vline~
    \includegraphics[width=0.23\textwidth]{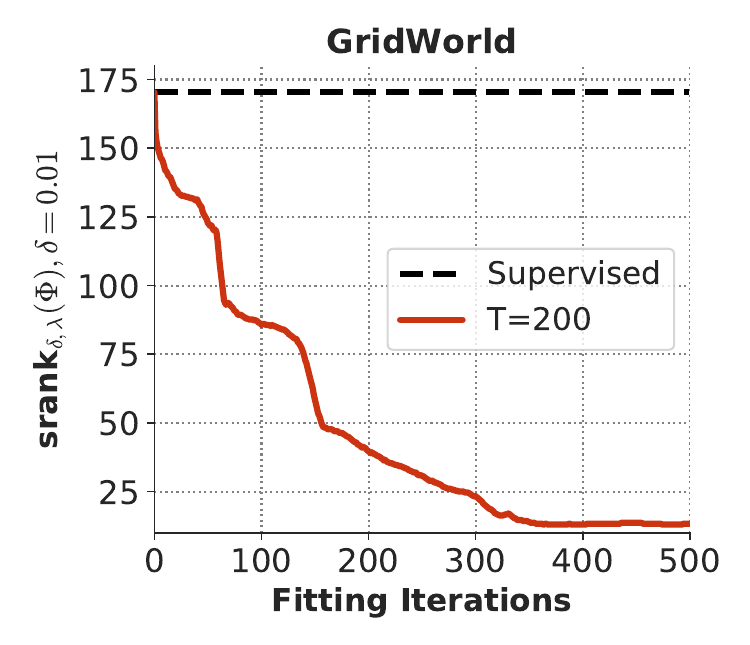}
    \includegraphics[width=0.23\textwidth]{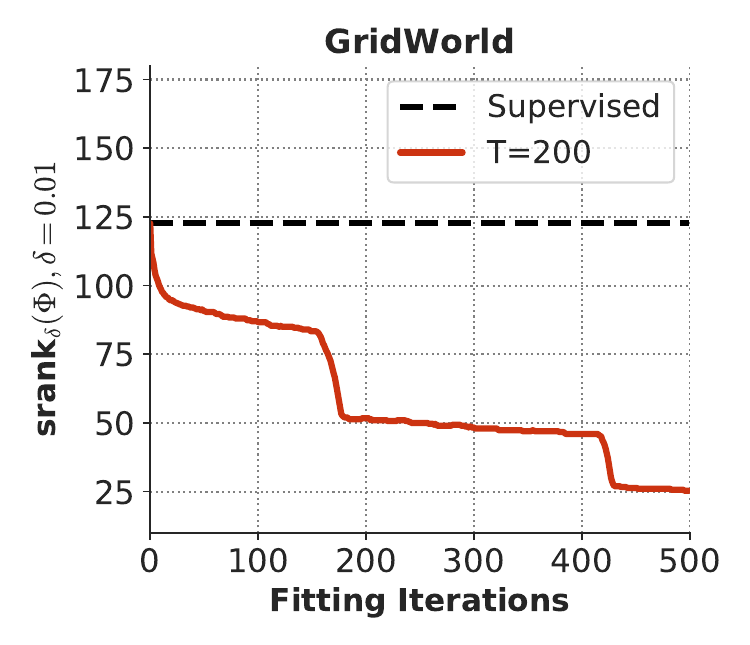}
    \caption{{Comparison of $\srank_{\delta, \lambda}(\Phi)$ and $\srank_{\delta}(\Phi)$ on the gridworld in the offline  setting \textbf{(left)} and the online setting \textbf{(right)}, when using a deep linear network with $\text{dim}(\phi(\bs, \ba)) = |\mathcal{S}| |\mathcal{A}|$. Note that both notions of effective rank exhibit a similar decreasing trend and are closely related to each other.}}
    \label{fig:gridworld_eigen_rank}
\end{figure}

{\subsection{Normalized Plots for Figure~\ref{fig:online_problem}/ Figure~\ref{fig:online_5_games_atari}}
\label{app:normalized_plots}
In this section, we provide a set of normalized $\srank$ and performance trends for Atari games (the corresponding unnormalized plots are found in Figure~\ref{fig:online_5_games_atari}). In these plots, each unit on the x-axis is equivalent to one gradient update, and so since $n=8$ prescribes $8\times$ many updates as compared to $n=1$, it it runs for $8\times$ as long as $n=1$. These plots are in Figure~\ref{fig:online_normalized_atari}.}

\begin{figure}[h]
    \centering
    \includegraphics[width=\linewidth]{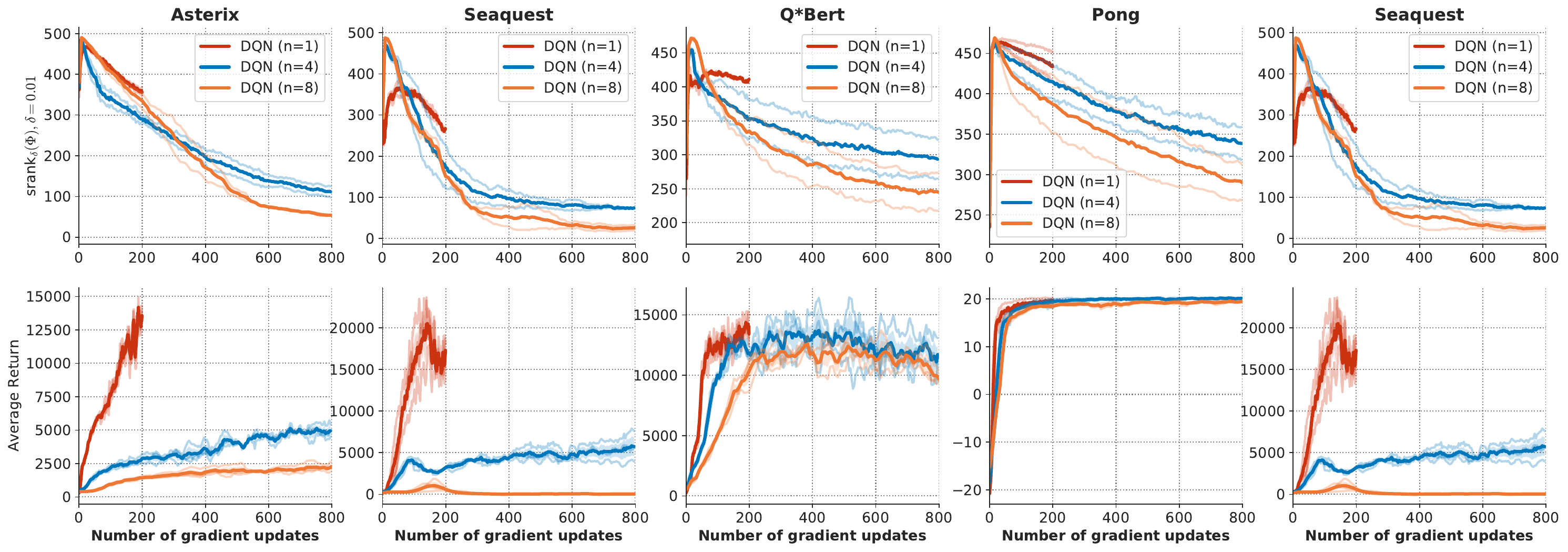}
    \caption{{\textbf{Rank collapse in DQN as a function of gradient updates on the x-axis} for five Atari games in the data-efficient online RL setting. This setting was previously studied in Figure~\ref{fig:online_5_games_atari}}. Note that lesser number of updates per unit amount data, \ie smaller values of $n$ possess larger $\srank_\delta$ values.}
    \label{fig:online_normalized_atari}
    \vspace{-0.15cm}
\end{figure}

{Note that the trend that effective rank decreases with larger $n$ values also persists when rescaling the x-axis to account for the number of gradient steps, in all but one game. This is expected since it tells us that performing bootstrapping based updates in the data-efficient setting (larger $n$ values) still leads to more aggressive rank drop as updates are being performed on a relatively more static dataset for larger values of $n$.}


\section{Hyperparameters \& Experiment Details}\label{sec:hyperparameters}

\subsection{Atari Experiments}
We follow the experiment protocol from \citet{agarwal2020optimistic} for all our  experiments including hyperparameters and agent architectures provided in Dopamine and report them for completeness and ease of reproducibility in Table~\ref{table:hyperparams}. We only use hyperparameter selection over the regularization experiment $\alpha_p$ based on results from 5 Atari games~(Asterix, Seaquest, Pong, Breakout and Seaquest). We will also open source our code to further aid in reproducing our results.

\begin{table*}[t]
\small
\caption{Hyperparameters used by the offline and online RL agents in our experiments.}
\centering
\begin{tabular}{lrr}
\toprule
Hyperparameter & \multicolumn{2}{r}{Setting (for both variations)} \\
\midrule
Sticky actions && Yes        \\
Sticky action probability && 0.25\\
Grey-scaling && True \\
Observation down-sampling && (84, 84) \\
Frames stacked && 4 \\
Frame skip~(Action repetitions) && 4 \\
Reward clipping && [-1, 1] \\
Terminal condition && Game Over \\
Max frames per episode && 108K \\
Discount factor && 0.99 \\
Mini-batch size && 32 \\
Target network update period & \multicolumn{2}{r}{every 2000 updates} \\
Training environment steps per iteration && 250K \\
Update period every && 4 environment steps \\
Evaluation $\epsilon$ && 0.001 \\
Evaluation steps per iteration && 125K \\
$Q$-network: channels && 32, 64, 64 \\
$Q$-network: filter size && $8\times8$, $4\times4$, $3\times3$\\
$Q$-network: stride && 4, 2, 1\\
$Q$-network: hidden units && 512 \\
Hardware && Tesla P100 GPU \\
\midrule
Hyperparameter & Online & Offline\\
\midrule
Min replay size for sampling & 20,000 & - \\
Training $\epsilon$~(for $\epsilon$-greedy exploration) & 0.01 & - \\
$\epsilon$-decay schedule & 250K steps & - \\
Fixed Replay Memory & No & Yes \\
Replay Memory size~(Online) &  1,000,000 steps & -- \\
Fixed Replay size~(5\%) & -- & 2,500,000 steps \\
Fixed Replay size~(20\%) & -- & 10,000,000 steps \\
Replay Scheme & Uniform & Uniform \\
Training Iterations & 200 & 500 \\
\bottomrule
\end{tabular}
\label{table:hyperparams}
\end{table*}

\textbf{Evaluation Protocol}. Following \citet{agarwal2020optimistic}, the Atari environments used in our experiments are stochastic due to sticky actions, \ie\ there is 25\% chance at every time step that the environment will execute the agent's previous action again, instead of the agent's new action. All agents~(online or offline) are compared using the best evaluation score~(averaged over 5 runs) achieved during training where the evaluation is done online every training iteration using a $\epsilon$-greedy policy with $\epsilon=0.001$. We report offline training results with same hyperparameters over 5 random seeds of the DQN replay data collection, game simulator and network initialization.

{\bf Offline Dataset}. As suggested by \citet{agarwal2020optimistic}, we randomly subsample the DQN Replay dataset containing 50 millions transitions to create smaller offline datasets with the same data distribution as the original dataset. We use the 5\% DQN replay dataset for most of our experiments. We also report results using the 20\% dataset setting (4x larger) to show that our claims are also valid even when we have higher coverage over the state space. 

{\bf Optimizer related hyperparameters}. For existing off-policy agents, step size and optimizer were taken as published. We used the DQN~(Adam) algorithm for all our experiments, given its superior performance over the DQN~(Nature) which uses RMSProp, as reported by \citet{agarwal2020optimistic}.

{\bf Atari 2600 games used}. For all our experiments in \Secref{sec:problem}, we used the same set of 5 games as utilized by \citet{agarwal2020optimistic, bellemare2017distributional} to present analytical results. For our empirical evaluation in Appendix~\ref{app:fix_results}, we use the set of games employed by \citet{fedus2020revisiting} which are deemed suitable for offline RL  by \citet{gulcehre2020rl}. Similar in spirit to \citet{gulcehre2020rl}, we use the set of 5 games used for analysis for hyperparameter tuning for offline RL methods.

\textbf{5 games subset}: \textsc{Asterix}, \textsc{Qbert}, \textsc{Pong}, \textsc{Seaquest}, \textsc{Breakout}

\textbf{16 game subset}: In addition to 5 games above, the following 11 games: \textsc{Double Dunk}, \textsc{James Bond}, \textsc{Ms. Pacman}, \textsc{Space Invaders}, \textsc{Zaxxon}, \textsc{Wizard of Wor}, \textsc{Yars' Revenge}, \textsc{Enduro}, \textsc{Road Runner}, \textsc{BeamRider}, \textsc{Demon Attack}

\subsection{Gridworld Experiments}
We use the gridworld suite from \citet{fu19diagnosing} to obtain gridworlds for our experiments. All of our gridworld results are computed using the $16\times16$ \textsc{Grid16smoothobs} environment, which consists of a 256-cell grid, with walls arising randomly with a probability of 0.2. Each state allows 5 different actions (subject to hitting the boundary of the grid): move left, move right, move up, move down and no op. The goal in this environment is to minimize the cumulative discounted distance to a fixed goal location where the discount factor is given by $\gamma = 0.95$. The features for this Q-function are given by randomly chosen vectors which are smoothened spatially in a local neighborhood of a grid cell $(x, y)$.

We use a deep Q-network with two hidden layers of size $(64, 64)$, and train it using soft Q-learning with entropy coefficient of 0.1, following the code provided by authors of \citet{fu19diagnosing}. We use a first-in-first out replay buffer of size 10000 to store past transitions. 

\section{Proofs for Section~\ref{sec:self_distill}}
\label{app:self_distill_proofs}
In this section, we provide the technical proofs from Section~\ref{sec:self_distill}. We first derive a solution to optimization problem Equation~\ref{eqn:self_dist_problem} and show that it indeed comes out to have the form described in Equation~\ref{eqn:unfolding}. We first introduce some notation, including definition of the kernel $\bG$ which was used for this proof. This proof closely follows the proof from \citet{mobahi2020self}.

\paragraph{Definitions.} For any universal kernel $u$, the Green's function~\citep{duffy2015green} of the linear kernel operator $L$ given by: $\left[ L \bQ \right](\bs, \ba) := \sum_{(\bs', \ba')} u((\bs, \ba), (\bs', \ba')) \bQ(\bs', \ba')$ is given by the function $g((\bs, \ba), (\bs', \ba'))$ that satisfies:
\begin{equation}
    \sum_{(\bs, \ba)} u((\bs, \ba), (\bs', \ba'))~~ g((\bs', \ba'), (\bar{\bs}, \bar{\ba})) = \delta((\bs, \ba) - (\bar{\bs}, \bar{\ba})), 
\end{equation}
where $\delta$ is the Dirac-delta function. Thus, Green's function can be understood as a kernel that ``inverts'' the universal kernel $u$ to the identity (Dirac-delta) matrix. We can then define the matrix $\bG$ as the matrix of vectors $\bg_{(\bs, \ba)}$ evaluated on the training dataset, $\mathcal{D}$, however note that the functional $\bg_{(\bs, \ba)}$ can be evaluated for other state-action tuples, not present in $\mathcal{D}$.
\begin{equation}
    \bG((\bs_i, \ba_i), (\bs_j, \ba_j)) := \bg((\bs_i, \ba_i), (\bs_j, \ba_j)) ~~~~ \text{and}~~~ \bg_{(\bs, \ba)}[i] = \bg((\bs, \ba), (\bs_i, \ba_i))~~ \forall (\bs_i, \ba_i) \in \mathcal{D}.   
\end{equation}

\begin{lemma}
The solution to Equation~\ref{eqn:self_dist_problem} is given by Equation~\ref{eqn:unfolding}.
\end{lemma}
\begin{proof}
This proof closely follows the proof of Proposition 1 from \citep{mobahi2020self}. We revisit key aspects the key parts of this proof here.

We restate the optimization problem below, and solve for the optimum $\bQ_k$ to this equation by applying the functional derivative principle.
\begin{equation*}
\min_{\bQ \in \mathcal{Q}} J(\bQ) := \sum_{\bs_i, \ba_i \in \mathcal{D}} \left(\bQ(\bs_i, \ba_i) - y_{k}(\bs_i, \ba_i) \right)^2 + c \sum_{(\bs, \ba)} \sum_{(\bs', \ba')} u\mathbf{(}(\bs, \ba), (\bs', \ba')\mathbf{)} \bQ(\bs, \ba) \bQ(\bs', \ba').
\end{equation*}
The functional derivative principle would say that the optimal $\bQ_k$ to this problem would satisfy, for any other function $f$ and for a small enough $\varepsilon \rightarrow 0$,
\begin{equation}
    \forall f \in \mathcal{Q}:~~ \frac{\partial J(\bQ_k + \varepsilon f)}{\partial \varepsilon}\Big\vert_{\varepsilon = 0} = 0.
\end{equation}
By setting the gradient of the above expression to $0$, we obtain the following stationarity conditions on $\bQ_k$ (also denoting $(\bs_i, \ba_i) := \bx_i$) for brevity:
\begin{equation}
    \sum_{\bx_i \in \mathcal{D}} \delta(\bx - \bx_i) \left( \bQ_k(\x_i) - \by_k(\x_i)\right) + c \sum_{\x} u(\bx, \bx') \bQ_k(\x')  = 0.
\end{equation}
Now, we invoke the definition of the Green's function discussed above and utilize the fact that the Dirac-delta function can be expressed in terms of the Green's function, we obtain a simplified version of the above relation:   
\begin{equation}
    \sum_{\bx} u(\bx, \bx') \sum_{\bx_i \in \mathcal{D}} (\bQ_k(\bx_i) - \by_k(\bx_i)) g(\bx', \bx_i) = -c \sum_{\bx} u(\bx, \bx') \bQ_k(\bx').
\end{equation}
Since the kernel $u(\bx, \bx')$ is universal and positive definite, the optimal solution $\bQ_k(\bx)$ is given by:
\begin{equation}
    \bQ_k(\bs, \ba) = -\frac{1}{c} \sum_{(\bs_i, \ba_i) \in \mathcal{D}} (\bQ_k(\bs_i, \ba_i) - \by_k(\bs_i, \ba_i))~\cdot~ g((\bs, \ba), (\bs_i, \ba_i)).
\end{equation}
Finally we can replace the expression for residual error, $\bQ_k(\bs_i, \ba_i) - \by_k(\bs_i, \ba_i)$ using the green's kernel on the training data by solving for it in closed form, which gives us the solution in Equation~\ref{eqn:unfolding}. 
\begin{equation}
    \bQ_k(\bs, \ba) = -\frac{1}{c} \bg_{(\bs, \ba)}^T (\bQ_k - \by_k) = \bg_{(\bs, \ba)}^T (c\bI + \bG)^{-1} \by_k.
\end{equation}
\end{proof}

Next, we now state and prove a slightly stronger version of Theorem~\ref{thm:self_distillation} that immediately implies the original theorem. 


\begin{theorem}
\label{thm:self_distillation_full}
Let $\bS$ be a shorthand for $\bS = \gamma P^\pi \bA$ and assume $\bS$ is a normal matrix. Then there exists an infinite, strictly increasing sequence of fitting iterations, $(k_l)_{l=1}^\infty$ starting from $k_1 = 0$,
such that, for any two singular-values $\sigma_i(\bS)$ and $\sigma_j(\bS)$ of $\bS$ with $\sigma_i(\bS) \leq \sigma_j(\bS)$, 
\vspace{-0.03in}
\begin{equation}
    \forall~ l \in \mathbb{N}~~\text{and}~~l' \geq l,~~ \frac{\sigma_i(\bM_{k_{l^\prime}})}{\sigma_j(\bM_{k_{l^\prime}})} < \frac{\sigma_i(\bM_{k_l})}{\sigma_j(\bM_{k_l})} \leq \frac{\sigma_i(\bS)}{\sigma_j(\bS)}.
\vspace{-0.03in}
\end{equation}
Therefore, the effective rank of $\bM_k$ satisfies: $\srank_{\delta}(\bM_{k_{l^\prime}}) \leq \srank_\delta(\bM_{k_l})$. Furthermore,
\vspace{-0.03in}
\begin{equation}
    \forall~ l \in \mathbb{N}~~\text{and}~~t \geq k_l,~~ \frac{\sigma_i(\bM_t)}{\sigma_j(\bM_t)} < \frac{\sigma_i(\bM_{k_l})}{\sigma_j(\bM_{k_l})} + \mathcal{O}\left(\left(\frac{\sigma_i(\bS)}{\sigma_j(\bS)}\right)^{k_l}\right).
\vspace{-0.03in}
\end{equation}
Therefore, the effective rank of $\bM_t$, $\srank_{\delta}(\bM_t)$,  outside the chosen subsequence is also controlled above by the effective rank on the subsequence $(\srank_{\delta}(\bM_{k_l}))_{l=1}^\infty$.
\end{theorem}

To prove this theorem, we first show that for any two fitting iterations, $t < t'$, if $\bS^t$ and $\bS^{t'}$ are positive semi-definite, the ratio of singular values and the effective rank decreases from $t$ to $t'$. As an immediate consequence, this shows that when $\bS$ is positive semi-definite, the effective rank decreases at \emph{every} iteration, \ie by setting $k_l=l$ (Corollary \ref{lemma:all_symmetric}). 

To extend the proof to arbitrary normal matrices, we show that for any $\bS$, a sequence of fitting iterations $(k_l)_{l=1}^\infty$ can be chosen such that $\bS^{k_l}$ is (approximately) positive semi-definite. For this subsequence of fitting iterations, the ratio of singular values and effective rank also decreases. Finally, to control the ratio and effective rank on fitting iterations $t$ outside this subsequence, we construct an upper bound on the ratio $f(t)$: $\frac{\sigma_i(\bM_t)}{\sigma_j(\bM_t)} < f(t)$, and relate this bound to the ratio of singular values on the chosen subsequence.

\begin{lemma}[$\srank_\delta(\bM_k)$ decreases when $\bS^k$ is PSD.]
\label{lemma:mk_symmetric}
Let $\bS$ be a shorthand for $\bS = \gamma P^\pi \bA$ and assume $\bS$ is a normal matrix. Choose any $t, t' \in \mathbb{N}$ such that $t < t'$. If $\bS^t$ and $\bS^{t'}$ are positive semi-definite, then for any two singular-values $\sigma_i(\bS)$ and $\sigma_j(\bS)$ of $\bS$, such that $0 < \sigma_i(\bS) < \sigma_j(\bS)$:
\begin{equation}
    \frac{\sigma_i(\bM_{t'})}{\sigma_j(\bM_{t'})} < \frac{\sigma_i(\bM_t)}{\sigma_j(\bM_t)} \leq \frac{\sigma_i(\bS)}{\sigma_j(\bS)}.
\end{equation}
Hence, the effective rank of $\bM_k$ decreases from $t$ to $t'$: $\srank_{\delta}(\bM_{t'}) \leq \srank_\delta(\bM_t)$. 
\end{lemma}

\begin{proof} 
First note that $\bM_k$ is given by:
\begin{equation}
    \bM_k := \sum_{i=1}^k \gamma^{k-i} (P^\pi \bA)^{k-i} =  \sum_{i=1}^{k} \bS^{k-i}.
\end{equation}
From hereon, we omit the leading $\gamma$ term since it is a constant scaling factor that does not affect ratio or effective rank. Almost every matrix $\bS$ admits a complex orthogonal eigendecomposition. Thus, we can write $\bS := \bU \mathbf{\lambda}(\bS) \bU^H$. And any power of $\bS$, \ie, $\bS^i$ can be expressed as: $\bS^i = \bU \mathbf{\lambda}(\bS)^i \bU^H$, and hence, we can express $\bM_k$ as:
\begin{equation}
\label{eqn:eigen_val}
    \bM_k := \bU \left({\sum_{i=0}^{k-1} \lambda(\bS)^i} \right) \bU^H = \bU ~\cdot~ \text{\textbf{diag}}\left( \frac{1 - \lambda(\bS)^{k}}{1 - \lambda(\bS)} \right) ~\cdot ~\bU^H.
\end{equation}

Since $\bS$ is normal, its eigenvalues and singular values are further related as $\sigma_k(S) = |\lambda_k(S)|$. And this also means that $\bM_k$ is normal, indicating that $\sigma_i(\bM_k) = |\lambda_i(\bM_k)|$. Thus, the singular values of $\bM_k$ can be expressed as
\begin{equation}
    \sigma_i(\bM_k) := \left|\frac{1 - \lambda_i(\bS)^{k}}{1 - \lambda_i(\bS)} \right|,
\end{equation}

When $\bS^k$ is positive semi-definite, $\lambda_i(\bS)^k = \sigma_i(\bS)^k$, enabling the following simplification:
\begin{equation}
    \sigma_i(\bM_k) = \frac{|1 - \sigma_i(\bS)^{k}|}{|1 - \lambda_i(\bS)|}.
\end{equation}

To show that the ratio of singular values decreases from $t$ to $t'$, we need to show that $f(\sigma) = \frac{|1 - \sigma^{t'}|}{|1 - \sigma^t|}$ is an increasing function of $\sigma$ when $t' > t$. It can be seen that this is the case, which implies the desired result.

To further show that $\srank_\delta(\bM_t) \geq \srank_\delta(\bM_{t'})$, we can simply show that $\forall i \in [1, \cdots, n]$, $h_k(i) := \frac{\sum_{j=1}^i \sigma_j(\bM_k)}{\sum_{j=1}^n \sigma_j(\bM_k)}$ increases with $k$, and this would imply that the $\srank_\delta(\bM_k)$ cannot increase from $k=t$ to $k=t'$. We can decompose $h_k(i)$ as:
\begin{equation}
\label{eqn:hki}
    h_k(i) = \sum_{j=1}^i \frac{\sigma_j(\bM_k)}{\sum_{l} \sigma_l(\bM_k)} = \frac{1}{1 + \frac{\sum_{j=i+1}^n \sigma_j(\bM_k)}{\sum_{j=1}^i \sigma_j(\bM_k)}}.
\end{equation}
Since $\sigma_{j}(\bM_k)/\sigma_{l}(\bM_i)$ decreases over time $k$ for all $j, l$ if $\sigma_{j}(\bS) \leq \sigma_{l}(\bS)$, the ratio in the denominator of $h_k(i)$ decreases with increasing $k$ implying that $h_k(i)$ increases from $t$ to $t'$. 
\end{proof}

\begin{corollary}[$\srank_\delta(\bM_k)$ decreases for PSD $\bS$ matrices.]
\label{lemma:all_symmetric}
Let $\bS$ be a shorthand for $\bS = \gamma P^\pi \bA$. Assuming that $\bS$ is positive semi-definite, for any $k, t \in \mathbb{N}$, such that $t > k$ and that for any two singular-values $\sigma_i(\bS)$ and $\sigma_j(\bS)$ of $\bS$, such that $\sigma_i(\bS) < \sigma_j(\bS)$, 
\begin{equation}
    \frac{\sigma_i(\bM_t)}{\sigma_j(\bM_t)} < \frac{\sigma_i(\bM_k)}{\sigma_j(\bM_k)} \leq \frac{\sigma_i(\bS)}{\sigma_j(\bS)}.
\end{equation}
Hence, the effective rank of $\bM_k$ decreases with more fitting iterations: $\srank_{\delta}(\bM_t) \leq \srank_\delta(\bM_k)$. 
\end{corollary}

In order to now extend the result to arbitrary normal matrices, we must construct a subsequence of fitting iterations $(k_l)_{l=1}^\infty$ where $\bS^{k_l}$ is (approximately) positive semi-definite. To do so, we first prove a technical lemma that shows that rational numbers, \ie numbers that can be expressed as $r = \frac{p}{q}$, for integers $p, q \in \mathbb{Z}$ are ``dense'' in the space of real numbers. 

\begin{lemma}[Rational numbers are dense in the real space.]
\label{lemma:rational_numbers}
For any real number $\alpha$, there exist infinitely many rational numbers $\frac{p}{q}$ such that $\alpha$ can be approximated by $\frac{p}{q}$ upto $\frac{1}{q^2}$ accuracy.
\begin{equation}
    \left\vert \alpha - \frac{p}{q} \right\vert \leq \frac{1}{q^2}.
\end{equation}
\end{lemma}
\begin{proof}
We first use Dirichlet's approximation theorem (see \citet{Hlawka1991dirichlet} for a proof of this result using a pigeonhole argument and extensions) to obtain that for any real numbers $\alpha$ and $N \geq 1$, there exist integers $p$ and $q$ such that $1 \leq q \leq N$ and, 
\begin{equation}
|q \alpha - p| \leq \frac{1}{|N| + 1} < \frac{1}{N}.
\end{equation}
Now, since $q \geq 1 > 0$, we can divide both sides by $q$, to obtain:
\begin{equation}
    \left\vert \alpha - \frac{p}{q} \right\vert \leq \frac{1}{N q} \leq \frac{1}{q^2}.
\end{equation}
To obtain infinitely many choices for $\frac{p}{q}$, we observe that Dirichlet's lemma is valid only for all values of $N$ that satisfy $N \leq \frac{1}{|q \alpha - p|}$. Thus if we choose an $N'$ such that $N' \geq N_{\max}$ where $N_{\max}$ is defined as: 
\begin{equation}
\label{eqn:N_max}
    N_{\max} = \max \left\{ \frac{1}{|q' \alpha - p'|} ~~\Big\vert~~ p', q' \in \mathbb{Z}, 1 \leq q' \leq q  \right\}.
\end{equation}
Equation~\ref{eqn:N_max} essentially finds a new value of $N$, such that the current choices of $p$ and $q$, which were valid for the first value of $N$ do not satisfy the approximation error bound. Applying Dirichlet's lemma to this new value of $N'$ hence gives us a new set of $p'$ and $q'$ which satisfy the $\frac{1}{q'^2}$ approximation error bound. Repeating this process gives us countably many choices of $(p, q)$ pairs that satisfy the approximation error bound. As a result, rational numbers are dense in the space of real numbers, since for any arbitrarily chosen approximation accuracy given by $\frac{1}{q^2}$, we can obtain atleast one rational number, $\frac{p}{q}$ which is closer to $\alpha$ than $\frac{1}{q^2}$. This proof is based on \citet{approxstack}.
\end{proof}

Now we utilize Lemmas~\ref{lemma:mk_symmetric} and \ref{lemma:rational_numbers} to prove Proposition~\ref{thm:self_distillation}.

\paragraph{Proof of Proposition \ref{thm:self_distillation} and Theorem~\ref{thm:self_distillation_full}} Recall from the proof of Lemma~\ref{lemma:mk_symmetric} that the singular values of $\bM_k$ are given by:
\begin{equation}
    \sigma_i(\bM_k) := \left|\frac{1 - \lambda_i(\bS)^{k}}{1 - \lambda_i(\bS)} \right|,
\end{equation}
\textbf{Bound on Singular Value Ratio:}
The ratio between $\sigma_i(\bM_k)$ and $\sigma_j(\bM_k)$ can be expressed as
\begin{equation}
    \frac{\sigma_i(\bM_k)}{\sigma_j(\bM_k)} = \left|\frac{1 - \lambda_i(\bS)^{k}}{1 - \lambda_j(\bS)^{k}} \right|\left|\frac{1 - \lambda_j(\bS)}{1 - \lambda_i(\bS)} \right|.
\end{equation}
For a normal matrix $\bS$, $\sigma_i(\bS) = |\lambda_i(\bS)|$, so this ratio can be bounded above as
\begin{equation}
    \frac{\sigma_i(\bM_k)}{\sigma_j(\bM_k)} \leq \frac{1 + \sigma_i(\bS)^{k}}{|1 - \sigma_j(\bS)^{k}|} \left|\frac{1 - \lambda_j(\bS)}{1 - \lambda_i(\bS)} \right|.
\end{equation}
Defining $f(k)$ to be the right hand side of the equation, we can verify that $f$ is a monotonically decreasing function in $k$ when $\sigma_i < \sigma_j$. This shows that this ratio of singular values in bounded above and in general, must decrease towards some limit $\lim_{k \to \infty} f(k)$.

\textbf{Construction of Subsequence:} We now show that there exists a subsequence $(k_l)_{l=1}^\infty$ for which $\bS^{k_l}$ is approximately positive semi-definite. For ease of notation, let's represent the i-th eigenvalue as $\lambda_i(\bS) = |\lambda_i(\bS)| \cdot e^{i \theta_i}$, where $\theta_i > 0$ is the polar angle of the complex value $\lambda_i(\bs)$ and $|\lambda_i(\bS)|$ is its magnitude (norm). Now, using Lemma~\ref{lemma:rational_numbers}, we can approximate any polar angle, $\theta_i$ using a rational approximation, \ie, we apply lemma~\ref{lemma:rational_numbers} on $\frac{\theta_i}{2 \pi}$
\begin{equation}
    \exists~ p_i, q_i \in \mathbb{N}, ~~\text{s.t.}~~ \left|\frac{\theta_i}{2 \pi} - \frac{p_i}{q_i} \right| \leq \frac{1}{q_i^2}. 
\end{equation}
Since the choice of $q_i$ is within our control we can estimate $\theta_i$ for all eigenvalues $\lambda_i(\bS)$ to infinitesimal accuracy. Hence, we can approximate $\theta_i \approx 2\pi \frac{p_i}{q_i}$. We will now use this approximation to construct an infinite sequence $(k_l)_{l=1}^\infty$, shown below:
\begin{equation}
    k_l = l \cdot \text{LCM}(q_1, \cdots, q_n) ~~~ \forall~ j \in \mathbb{N}, 
\end{equation}
where $\text{LCM}$ is the least-common-multiple of natural numbers $q_1, \cdots q_n$.

In the absence of any approximation error in $\theta_i$, we note that for any $i$ and for any $l \in \mathbb{N}$ as defined above, $\lambda_i(\bS)^{k_l} = |\lambda_i(\bS)|^{k_l} \cdot \exp\left(2i \pi \cdot \frac{p_i}{q_i} \cdot k_l \right) = |\lambda_i(\bS)|^{k_l}$, since the polar angle for any $k_l$ is going to be a multiple of $2 \pi$, and hence it would fall on the real line. As a result, $\bS^{k_l}$ will be positive semi-definite, since all eigenvalues are positive and real. Now by using the proof for Lemma~\ref{lemma:all_symmetric}, we obtain the ratio of $i$ and $j$ singular values are increasing over the sequence of iterations $(k_j)_{j=1}^\infty$. Since the approximation error in $\theta_i$ can be controlled to be infinitesimally small to prevent any increase in the value of $\srank_\delta$ due to it (this can be done given the discrete form of $\srank_\delta$), we note that the above argument applies even with the approximation, proving the required result on the subsequence.

\textbf{Controlling All Fitting Iterations using Subsequence: }

We now relate the ratio of singular values within this chosen subsequence to the ratio of singular values elsewhere. Choose $t, l \in \mathbb{N}$ such that $t > k_l$. Earlier in this proof, we showed that the ratio between singular values is bounded above by a monotonically decreasing function $f(t)$, so 
\begin{equation}
    \frac{\sigma_i(\bM_{t})}{\sigma_j(\bM_{t})} \leq f(t) < f(k_l).
\end{equation}
Now, we show that that $f(k_l)$ is in fact very close to the ratio of singular values: 
\begin{equation}
    f(k_l) = \frac{|1 - \sigma_i(\bS)^{k_l}|}{|1 - \sigma_j(\bS)^{k_l}|} \left|\frac{1 - \lambda_j(\bS)}{1 - \lambda_i(\bS)} \right| \leq \frac{\sigma_i(\bM_{t})}{\sigma_j(\bM_{t})} + {\frac{2 \sigma_i(\bS)^{k_l}}{|1 - \sigma_j(\bS)^{k_l}|} \left|\frac{1 - \lambda_j(\bS)}{1 - \lambda_i(\bS)} \right|}.
\end{equation}
The second term goes to zero as $k_l$ increases; algebraic manipulation shows that this gap be bounded by 
\begin{equation}
    f(k_l) \leq \frac{\sigma_i(\bM_{k_l})}{\sigma_j(\bM_{k_l})} + \left(\frac{\sigma_i(S)}{\sigma_j(S)}\right)^{k_l} \underbrace{\frac{2 \sigma_j(\bS)}{|1 - \sigma_j(\bS)|} \left|\frac{1 - \lambda_j(\bS)}{1 - \lambda_i(\bS)} \right|}_{\text{constant}}.
\end{equation}

Putting these inequalities together proves the final statement,

\begin{equation}
     \frac{\sigma_i(\bM_{t})}{\sigma_j(\bM_{t})} \leq \frac{\sigma_i(\bM_{k_l})}{\sigma_j(\bM_{k_l})} + \mathcal{O}\left(\left(\frac{\sigma_i(S)}{\sigma_j(S)}\right)^{k_l}\right).
\end{equation}

{\textbf{Extension to approximately-normal $\bS$.}
We can extend the result in Theorem~\ref{thm:self_distillation_full} (and hence also Theorem~\ref{thm:self_distillation}) to approximately-normal $\bS$. Note that the main requirement for normality of $\bS$ (i.e., $\sigma_i(\bS) = |\lambda_i(\bs)|$) is because it is straightforward to relate the eigenvalue of $\bS$ to $\bM$ as shown below.
\begin{equation}
    \left|\lambda_i(\bM_k)\right| := \left|\frac{1 - \lambda_i(\bS)^{k}}{1 - \lambda_i(\bS)} \right|,
\end{equation}
Now, since the matrix $\bS$ is approximately normal, we can express it using its Schur's triangular form as, $\bS = \bU \cdot (\Lambda + \bN) \cdot \bU^H$, where $\Lambda$ is a diagonal matrix and $\bN$ is an ``offset'' matrix. The \textit{departure from normality} of $\bS$ is defined as: $\Delta(\bS) := \inf_{\bN}~ ||\bN||_2$, where the infimum is computed over all matrices $\bN$ that can appear in the Schur triangular form for $\bS$.  
For a normal $\bS$ only a single value of $\bN = 0$ satisfies the Schur's triangular form. For an approximately normal matrix $\bS$, $||\bN||_2 \leq \Delta(\bS) \leq \varepsilon$, for a small $\varepsilon$.}

{Furthermore note that from Equation 6 in \citet{ruhe1975closeness}, we obtain that
\begin{equation}
    \left\vert\sigma_i(\bS) - \left\vert\lambda_i(\bS)\right\vert\right\vert \leq \Delta(\bS) \leq \varepsilon,
\end{equation}
implying that singular values and norm-eigenvalues are close to each other for $\bS$.}

{Next, let us evaluate the departure from normality of $\bM_k$. First note that, $\bS^j = \bU \cdot (\Lambda + \bN)^j \cdot \bU^H$, and so, $\bM_k = \bU \cdot \left( \sum_{j=1}^k (\Lambda + \bN)^j \right) \cdot \bU^H$ and if $||\bN||_2 \leq \varepsilon$, for a small epsilon (\ie considering only terms that are linear in N for $(\Lambda + \bN)^j$), we note that:
\begin{equation}
    \left\vert \sigma_i(\bM_k) - |\lambda_i(\bM_k)| \right\vert \leq \sum_{j=1}^k j \cdot |\lambda_1(\bS)|^{j-1} \Delta(\bS) \leq \frac{1}{(1 - |\lambda_1(\bS)|)^2} \cdot \varepsilon.
\end{equation}
Thus, the matrix $\bM_k$ is also approximately normal provided that the max eigenvalue norm of $\bS$ is less than 1. This is true, since $\bS = \gamma P^\pi \bA$ (see Theorem~\ref{thm:self_distillation}, where both $P^\pi$ and $\bA$ have eigenvalues less than 1, and $\gamma < 1$.}

{Given that we have shown that $\bM_k$ is approximately normal, we can show that $\srank_\delta(\bM_k)$ only differs from $\srank_{\delta, \lambda}(\bM_k)$, \ie, the effective rank of eigenvalues, in a bounded amount. If the value of $\varepsilon$ is then small enough, we still retain the conclusion that $\srank_\delta(\bM_k)$ generally decreases with more training by following the proof of Theorem~\ref{thm:self_distillation_full}.}

\section{Proofs for Section~\ref{sec:grad_descent}} 
\label{app:grad_descent_proofs}

In this section, we provide technical proofs from Section~\ref{sec:grad_descent}. We start by deriving properties of optimization trajectories of the weight matrices of the deep linear network similar to \citet{arora2018optimization} but customized to our set of assumptions, then prove Proposition~\ref{thm:sing_val_evolution}, and finally discuss how to extend these results to the fitted Q-iteration setting and some extensions not discussed in the main paper. Similar to Section~\ref{sec:self_distill}, we assume access to a dataset of transitions, $\mathcal{D} = \{(\bs_i, \ba_i, r(\bs_i, \ba_i), \bs'_i \}$ in this section, and assume that the same data is used to re-train the function.

\textbf{Notation and Definitions.} The Q-function is represented using a deep linear network with at least 3 layers, such that
\begin{equation}
    Q(\bs, \ba) = \bW_{N} \bW_{N-1} \cdots \bW_1 [\bs; \ba], \text{~~where~~} N \geq 3,  \bW_{N} \in \mathbb{R}^{1 \times d_{N-1}},
\end{equation} and \mbox{$\bW_i \in \mathbb{R}^{d_i \times d_{i-1}}$ for $i=1, \dots, N-1$}. We index the weight matrices by a tuple $(k, t)$: $\bW_j(k, t)$ denotes the weight matrix $\bW_j$ at the $t$-th step of gradient descent during the $k$-th fitting iteration (Algorithm~\ref{alg:fqi}). Let the end-to-end weight matrix $\bW_N \bW_{N-1} \cdots \bW_1$ be denoted shorthand as $\bW_{N:1}$, and let the features of the penultimate layer of the network, be denoted as $\bW_\phi(k, t) := \bW_{N-1}(k, t) \cdots \bW_1(k, t)$. $\bW_{\phi}(k, t)$ is the matrix that maps an input $[\bs;\ba]$ to corresponding features $\Phi(\bs, \ba)$. In our analysis, it is sufficient to consider the effective rank of $\bW_\phi(k, t)$ since the features $\Phi$ are given by: $\Phi(k, t) = \bW_\phi(k, t) [\mathcal{S}; \mathcal{A}]$, which indicates that:
\begin{equation*}
    \rank(\Phi(k, t)) = \rank(\bW_\phi(k, t) [\mathcal{S}; \mathcal{A}]) \leq \min\left( \rank(\bW_\phi(k, t)), \rank([\mathcal{S}; \mathcal{A}])  \right).
\end{equation*}
Assuming the state-action space has full rank, we are only concerned about $\rank(\bW_\phi(k, t))$ which justifies our choice for analyzing $\srank_\delta(\bW_\phi(k, t))$.

Let $L_{k+1}(\bW_{N:1}(k, t))$ denote the mean squared Bellman error optimization objective in the $k$-th fitting iteration. 
\begin{equation*}
    L_{k+1}(\bW_{N:1}(k, t)) = \sum_{i=1}^{|\mathcal{D}|} \left( \deepnet \stateactioni - \by_k(\bs_i, \ba_i) \right)^2, \text{~~where~~} \by_k = \bR + \gamma P^\pi \bQ_k.
\end{equation*}
When gradient descent is used to update the weight matrix, the updates to $\bW_i(k, t)$ are given by:
\begin{equation*}
    \bW_j(k, t+1) \leftarrow \bW_j(k, t) - \eta \gradlossktj. 
\end{equation*}
If the learning rate $\eta$ is small, we can approximate this discrete time process with a continuous-time differential equation, which we will use for our analysis. We use $\dot{W}(k, t)$ to denote the derivative of $W(k, t)$ with respect to $t$, for a given $k$.
\begin{equation}
\label{eqn:continuous_time_flow}
    \dot \bW_j(k, t) = -\eta \gradlossktj
\end{equation}

In order to quantify the evolution of singular values of the weight matrix, $\bW_\phi(k, t)$, we start by quantifying the evolution of the weight matrix $\bW_\phi(k, t)$ using a more interpretable differential equation. In order to do so, we make an assumption similar to but not identical as \citet{arora2018optimization}, that assumes that all except the last weight matrix are ``balanced'' at initialization $t = 0, k = 0$. i.e.
\begin{equation}
    \forall ~i \in [0, \cdots, N-2]: \bW_{i+1}^T(0, 0) \bW_{i+1}(0, 0) = \bW_i(0, 0) \bW_i(0, 0)^T.
\end{equation}
Note the distinction from \citet{arora2018optimization}, the last layer is not assumed to be balanced. As a result, we may not be able to comment about the learning dynamics of the end-to-end weight matrix, but we prevent the vacuous case where all the weight matrices are rank 1. Now we are ready to derive the evolution of the feature matrix, $\bW_\phi(k, t)$.

\begin{lemma}[Adaptation of balanced weights~\citep{arora2018optimization} across FQI iterations]
\label{thm:weights_are_balanced}
Assume the weight matrices evolve according to Equation~\ref{eqn:continuous_time_flow}, with respect to $L_{k}$ for all fitting iterations $k$. Assume balanced initialization only for the first $N-1$ layers, i.e., $ \bW_{j+1}(0, 0)^T \bW_{j+1}(0, 0) = \bW_j(0, 0) \bW_j(0, 0)^T, \forall~ j \in 1, \cdots, N-2$. Then the weights remain balanced throughout, i.e. 
\begin{equation}
\label{eqn:actual_balanced}
    \forall~ k, t~~~ \bW_{j+1}(k, t)^T \bW_{j+1}(k, t) = \bW_j(k, t) \bW_j(k, t)^T, \forall~ j \in 1, \cdots, N-2.
\end{equation}
\end{lemma}
\begin{proof}
First consider the special case of $k = 0$. To beign with, in order to show that weights remain balanced throughout training in $k=0$ iteration, we will follow the proof technique in \citet{arora2018optimization}, with some modifications. First note that the expression for $\gradlossktj$ can be expressed as:
\begin{align*}
    \gradlossktj ~&=~ \left(\prod_{i=j+1}^N \bW_i^T \right) \cdot \frac{d L_k (\bW_{N:1})}{d \bW_{N:1}} \cdot \left(\prod_{i=1}^{j-1} \bW_i^T \right)\\
    &= \left( \bW_{j+1}^T \bW_{j+2}^T \cdots \bW_{N}^T \right) \cdot \frac{d L_k (\bW_{N:1})}{d \bW_{N:1}} \cdot \left(\prod_{i=1}^{j-1} \bW_i^T \right).
\end{align*}
Now, since the weight matrices evolve as per Equation~\ref{eqn:continuous_time_flow}, by multiplying the similar differential equation for $\bW_j$ with $\bW_j^T(k, t)$ on the right and multiplying evolution of $\bW_{j+1}$ with $\bW_{j+1}^T(k, t)$ from the left, and adding the two equations, we obtain:
\begin{equation}
   \forall ~j \in [0, \cdots, N-2]:~ \bW_{j+1}^T(0, t) \dot \bW_{j+1}(0, t) = \dot \bW_{j}(0, t) \bW_j^T(0, t).
\end{equation}
We can then take transpose of the equation above, and add it to itself, to obtain an easily integrable expression:
\begin{multline}
    \frac{d ~\bW_{j+1}(0, t) \bW_{j+1}(0, t)^T}{d t} = \bW_{j+1}^T(0, t) \dot \bW_{j+1}(0, t) + \dot \bW_{j+1}(0, t) \bW_{j+1}^T(0, t) =\\ 
        \dot \bW_{j}(0, t) \bW_j^T(0, t) + \bW_{j}(0, t) \dot \bW_j^T(0, t) = \frac{d ~\bW_{j}^T(0, t) \bW_{j}(0, t)}{d t}.
\end{multline}
Since we have assumed the balancedness condition at the initial timestep $0$, and the derivatives of the two quantities are equal, their integral will also be the same, hence we obtain:
\begin{equation}
\label{eqn:balanced}
    \bW_{j+1}^T(0, t) \bW_{j+1}^T(0, t) = \bW_j(0, t) \bW_j(0, t)^T.
\end{equation}

Now, since the weights after $T$ iterations in fitting iteration $k=0$ are still balanced, the initialization for $k=1$ is balanced. Note that since the balancedness property does not depend on which objective gradient is used to optimize the weights, as long as $\bW_{j}$ and $\bW_{j+1}$ utilize the same gradient of the loss function. Formalizing this, we can show inductively that the weights will remain balanced across all fitting iterations $k$ and at all steps $t$ within each fitting iteration. Thus, we have shown the result in Equation~\ref{eqn:actual_balanced}.
\end{proof}

Our next result aims at deriving the evolution of the feature-matrix that under the balancedness condition. We will show that the feature matrix, $\bW_\phi(k, t)$ evolves according to a similar, but distinct differential equation as the end-to-end weight matrix, $\bW_{N:1}(k, t)$, which still allows us to appeal to techniques and results from \citet{arora2019implicit} to study properties of the singular value evolution and hence, discuss properties related to the effective rank of the matrix, $\bW_\phi(k, t)$. 

\begin{lemma}[(Extension of Theorem 1 from \citet{arora2018optimization}]
\label{lemma:weight_control}
Under conditions specified in Lemma~\ref{thm:weights_are_balanced}, the feature matrix, $\bW_\phi(k, t)$ evolves as per the following continuous-time differential equation, for all fitting iterations $k$:
\begin{equation*}
    \label{eqn:gradient_descent_diff_eqn}
    \resizebox{1.0\textwidth}{!}{$
    \dot{\bW}_\phi(k, t) = -\eta \sum_{j=1}^{N-1} \left[\bW_\phi(k, t) \bW_\phi(k, t)^T\right]^{\frac{N-j}{N-1}} \cdot \bW_{N}(k, t)^{T} \frac{d L_{k}(\bW_{N:1}(k, t))}{d \bW_{N:1}} \cdot \left[\bW_\phi(k, t)^T \bW_\phi(k, t) \right]^{\frac{j-1}{N-1}}.$}
\end{equation*}
\end{lemma}
\begin{proof}
In order to prove this statement, we utilize the fact that the weights upto layer $N-2$ are balanced throughout training. Now consider the singular value decomposition of any weight $\bw_j$ (unless otherwise states, we use $\bW_j$ to refer to $\bW_j(k, t)$ in this section, for ease of notation. $\bW_j = \bU_j \Sigma_j \bV_j^T$. The belancedness condition re-written using SVD of the weight matrices is equivalent to
\begin{equation}
    \bV_{j+1} \Sigma_{j+1}^T \Sigma_{j+1} \bV_{j+1}^T = \bU_{j} \Sigma_{j} \Sigma_{j}^T \bU_j^T.
\end{equation}
Thus for all $j$ on which the balancedness condition is valid, it must hold that $\Sigma_{j+1}^T \Sigma_{j+1} = \Sigma_{j} \Sigma_{j}^T$, since these are both the eigendecompositions of the same matrix (as they are equal). As a result, the weight matrices $\bW_j$ and $\bW_{j+1}$ share the same singular value space which can be written as $\rho_1 \geq \rho_2 \geq \cdots \geq \rho_m$. The ordering of eigenvalues can be different, and the matrices $\bU$ and $\bV$ can also be different (and be rotations of one other) but the unique values that the singular values would take are the same. Note the distinction from \citet{arora2018optimization}, where they apply balancedness on all matrices, and that in our case would trivially give a rank-1 matrix.

Now this implies, that we can express the feature matrix, also in terms of the common singular values, $\left( \rho_1, \rho_2, \cdots, \rho_m \right)$, for example, as $\bW_j(k, t) = \bU_{j+1} \text{Diag}\left(\sqrt{\rho_1}, \cdots, \sqrt{\rho_m} \right) \bV_{j}^T$, where $\bU)j = \bV_{j+1} \bO_j$, where $\bO_j$ is an orthonormal matrix. Using this relationship, we can say the following:
\begin{align*}
\prod_{i=j}^{N-1} \bW_i(k, t) \prod_{i=j}^{N-1} \bW_i(k, t)^T &= \left[ \bW_\phi(k, t) \bW_\phi^T(k, t) \right]^{\frac{N-j}{N-1}}\\
\prod_{i=1}^{j} \bW_i(k, t)^T \prod_{i=1}^{j} \bW_i(k, t) &= \left[ \bW_\phi(k, t)^T \bW_\phi(k, t) \right]^{\frac{j}{N-1}}.    
\end{align*}
Now, we can use these expressions to obtain the desired result, by taking the differential equations governing the evolution of $\bW_i(k, t)$, for $i \in [1, \cdots, N-1]$, multiplying the $i$-th equation by $\prod_{i+1}^{N-1} \bW_j(k, t)$ from the left, and $\prod_{1}^{i-1} \bW_j(k, t)$ to the right, and then summing over $i$.
\begin{align*}
    \dot \bW_\phi(k, t) ~&= \sum_{i=1}^{N-1} \left(\prod_{j=i+1}^{N-1} \bW_j(k, t) \right) \dot \bW_i(k, t) \left( \prod_{j=1}^{i-1} \bW_j(k, t) \right) \\
    &= - \eta \sum_{i=1}^{N-1} \left( \prod_{i+1}^{N-1} \bW_j(k, t) \prod_{i+1}^{N} \bW_j(k, t)^T \right) \frac{d L_k (\bW_{N:1})}{d \bW_{N:1}} \left( \prod_{j=1}^{i-1} \bW_{j}(k, t)^T \prod_{j=1}^{i-1} \bW_j(k, t) \right)
\end{align*}
The above equation simplifies to the desired result by taking out $\bW_N(k, t)$ from the first summation, and using the identities above for each of the terms.
\end{proof}

Comparing the previous result with Theorem 1 in \citet{arora2018optimization}, we note that the resulting differential equation for weights holds true for arbitrary representations or features in the network provided that the layers from the input to the feature layer are balanced. A direct application of \citet{arora2018optimization} restricts the model class to only fully balanced networks for convergence analysis and the resulting solutions to the feature matrix will then only have one active singular value, leading to less-expressive neural network configurations.    

\textbf{Proof of Proposition~\ref{thm:sing_val_evolution}.}~~ Finally, we are ready to use Lemma~\ref{lemma:weight_control} to prove the relationship with evolution on singular values. This proof can be shown via a direct application of Theorem 3 in \citet{arora2019implicit}. Given that the feature matrix, $\bW_\phi(k, t)$ satisfies a very similar differential equation as the end-to-end matrix, with the exception that the gradient of the loss with respect to the end-to-end matrix is pre-multiplied by $\bW_N(k, t)^T$. As a result, we can directly invoke \citet{arora2019implicit}'s result and hence, we have $\forall r \in [1, \cdots, \dim(W)]$ that:
\begin{equation}
     \dot{\sigma}_{r}(k, t) = -N \cdot \left( \sigma_r^2(k, t) \right)^{1 - \frac{1}{N-1}} ~\cdot~ \left\langle \bW_{N}(k, t)^T \frac{d L_{N, k}(\bW_{K, t})}{d \bW}, \bu_r(k, t) \bv_r(k, t)^T \right\rangle.
\end{equation}
 
Further, as another consequence of the result describing the evolution of weight matrices, we can also obtain a result similar to \citet{arora2019implicit} that suggests that the goal of the gradient update on the singular vectors $\bU(k, t)$ and $\bV(k, t)$ of the features $\bW_\phi(k, t) = \bU(k, t) \bS(k, t) \bV(K, t)^T$, is to align these spaces with $\bW_{N}(k, t)^T \frac{d L_{N, k}(\bW_{K, t})}{d \bW}$. 

\subsection{Explaining Rank Decrease Based on Singular Value Evolution}
\label{app:rank_decrease_sl}

\begin{figure}[H]
    \vspace{-5pt}
    \centering
    \includegraphics[width=\linewidth]{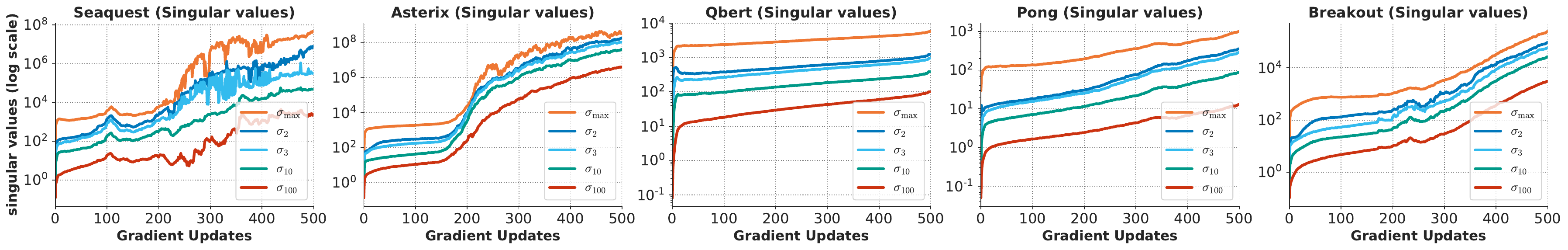}
    \caption{\textbf{Evolution of singular values of $\Phi$ on Atari}. The larger singular values of the feature matrix $\Phi$ grow at a disproportionately higher rate than other smaller singular values
    as described by \eqref{eqn:sing_val_diff_eqn} .}
    \label{fig:sing_val_all}
    \vspace{-10pt}
\end{figure}
In this section, we discuss why the evolution of singular values discussed in Equation~\ref{eqn:sing_val_diff_eqn} indicates a decrease in the rank of the feature matrix \emph{within} a fitting iteration $k$. To see this, let's consider the case when gradient descent has been run long enough (\ie the data-efficient RL case) for the singular vectors to stabilize, and consider the evolution of singular values post timestep $t \geq t_0$ in training. First of all, when the singular vectors stabilize, we obtain that $\bu_r(k, t)^T \bW_{N}(k, t)^T \frac{d L_{N, k}(\bW_{K, t})}{d \bW} \bv_r(k, t)$ is diagonal (extending result of Corollary 1 from \citet{arora2019implicit}). Thus, we can assume that
\begin{equation*}
    \bu_r^T(k, t) \bW_{N}(k, t)^T \frac{d L_{N, k}(\bW_{K, t})}{d \bW} \bv_r(k, t) = \mathbf{f}(k, t) \cdot \be_r \cdot \mathbf{d}_r,
\end{equation*}
where $\be_r$ is given by the unit basis vector for the standard orthonormal basis, $\mathbf{f}$ is a shorthand for the gradient norm of the loss function pre-multiplied by the transpose of the last layer weight matrix, and $\mathbf{d}_r$ denotes the singular values of the state-action input matrix, $[\states, \actions]$. In this case, we can re-write Equation~\ref{eqn:sing_val_diff_eqn} as:
\begin{equation}
\label{eqn:final_diff}
    \dot \sigma_r(k, t) = -N \left(\sigma^2_r(k, t)\right)^{1 - 1/N} \cdot \mathbf{f}(k, t) \cdot \be_r \cdot \mathbf{d}_r.
\end{equation}
Note again that unlike \citet{arora2019implicit}, the expression for $\mathbf{f}(k, t)$ is different from the gradient of the loss, since the weight matrix $\bW_N(k,t)$ is multiplied to the expression of the gradient in our case. By using the fact that the expression for $\mathbf{f}(k, t)$ is shared across all singular values, we can obtain differential equations that are equal across different $\sigma_{r}(k, t)$ and $\sigma_{r'}(k, t)$. Integrating, we can show that depending upon the ratio $\frac{\be_r \cdot \mathbf{d}_r}{\be_{r'} \cdot \mathbf{d}_{r'}}$, and the value of $N$, the singular values $\sigma_{r}(k, t)$ and $\sigma_{r'}(k, t)$ will take on different magnitude values, in particular, they will grow at disproportionate rates. By then appealing to the result from Proposition~\ref{thm:self_distillation} in Section~\ref{sec:self_distill} (kernel regression argument), we can say that disproportionately growing singular values would imply that the value of $\srank_\delta(\bW_\phi(k, t))$ decreases within each iteration $k$ as we will discuss next.

{\textbf{Interpretation of the abstract rank-regularized objective.}
We now discuss the form for the abstract rank-regularized objective in Equation~\ref{eqn:closed_form} that we utilize in Section~\ref{sec:grad_descent}. Intuitively, the justification for Equation~\ref{eqn:closed_form} is that larger singular values grow much faster than small singular values, which dictates how the ratio of singular values evolves through time, thereby reducing effective rank.}

{The discussion in the previous subsection shows that the effective rank, $\srank_\delta(\bW_\phi(k, t))$ decreases within each fitting iteration (\ie decreases over the variable $t$). Now, we will argue that for a given value of effective rank at $t = 0$ at a fitting iteration $k$ (denoted $\srank_\delta(\bW_\phi(k, 0))$, the effective rank after $T$ steps of gradient descent, $\srank_\delta(\bW_\phi(k, T))$, is constrained to be equal to some constant $\varepsilon_k < \srank_\delta(\bW_\phi(k, 0))$ due to the implicit regularizaton effect of gradient descent. To see this, note that if $\sigma_i(k, 0) > \sigma_j(k, 0)$, then $\sigma_i(k, t)$ increases at a much faster rate than $\sigma_j(k, t)$ with increasing $t$ (see Equation~\ref{eqn:final_diff} and \citet{arora2019implicit}, Equation 11)  and hence, the value of $h_{k, t}(i)$ is increasing over $t$:
\begin{equation}
\label{eqn:h_k_i_lower}
    h_{k, t}(i) = \sum_{j=1}^i \frac{\sigma_j(\bW_\phi(k, t))}{\sum_{l} \sigma_l(\bW_\phi(k, t))} = \frac{1}{1 + \frac{\sum_{j=i+1}^n \sigma_j(\bW_\phi(k, t))}{\sum_{j=1}^i \sigma_j(\bW_\phi(k, t))}} \geq \frac{1}{1 + \sum_{j=i+1}^{n} \frac{\sigma_{j}(\bW_\phi(k, t))}{\sigma_1(\bW_\phi(k, t))}},
\end{equation}
as the ratio between larger and smaller singular values increases. Using Equation~\ref{eqn:h_k_i_lower}, we note that the value of $h_{k, t}(i)$ approaches $1$ as the ratio $\sigma_j/\sigma_1$ decreases (in the limit, the ratio is zero). An increase in the value of $h_k(i)$ implies a decrease in $\srank_\delta$, which can be expressed as:
\begin{equation*}
    \srank_\delta(\bW_\phi(k, t)) = \min_{j}\left\{~ h_k(j) \geq 1 - \delta ~\right\}.
\end{equation*}
Thus, by running gradient descent for $T_0 = T(\varepsilon_k)$ (a function of $\varepsilon_k$) steps within a fitting iteration, $k$, the effective rank of the resulting $\bW_\phi(k, T_0)$ satisfies $\srank_\delta(\bW_\phi(k, T_0)) = \varepsilon_k $.}
{Now, since the solution $\bW_\phi(k, T_0)$ satisfies the constraint $\srank_\delta(\bW_\phi(k, T_0)) = \varepsilon_k < \srank_\delta(\bW_\phi(k, 0)) $, we express this solution using the solution of a penalized optimization problem (akin to penalty methods) that penalizes $\srank_\delta(\bW_\phi)$ for a suitable coefficient $\lambda_k$, which is a function of $\varepsilon_k$. We assume $\lambda_k > 0$ since when running gradient descent long enough (\ie for $T_0$ steps in the above discussion), we can reduce $\srank_\delta(\bW_\phi)$ from its initial value (note that $\varepsilon_k < \srank_\delta(\bW_\phi(k, 0))$), indicating a non-zero value of regularization.}

{
\subsection{Proof of Proposition~\ref{thm:compounding_main_paper}: Compounding Rank Drop in Section~\ref{sec:grad_descent}}
\label{app:rank_drop_compounds}
In this section, we illustrate that the rank drop effect compounds due to bootstrapping and prove Proposition~\ref{thm:compounding_main_paper} formally. Our goal is to demonstrate the compounding effect: a change in the rank at a given iteration gets propagated into the next iteration. We first illustrate this compounding effect and then discuss Theorem~\ref{thm:near_optimality} as a special case of this phenomenon.}

{\textbf{Assumptions.} We require two assumptions pertaining to closure of the function class and the change in effective rank due to reward and dynamics transformations, to be able for this analysis. We assume that the following hold for any fitting iteration $k$ of the algorithm.}
{\begin{assumption}[Closure]
\label{assumption:closed}
The chosen function class $\bW_1 \bW_2 \cdots \bW_N$ is closed under the Bellman evaluation backup for policy $\pi$. That is, if $\bQ_{k-1}^T = \bW_N(k-1) \bW_\phi(k - 1) [\mathcal{S}; \mathcal{A}]^T$, then there exists an assignment of weights to the deep linear net such that the corresponding target value $\bR + \gamma P^\pi \bQ_k$ can be written as $\left(\bR + \gamma P^\pi \bQ_{k-1} \right)^T =  \bW_N^P(k) \bW_\phi^P(k) [\mathcal{S}; \mathcal{A}]^T$. This assumption is commonly used in analyses of approximate dynamic programming with function approximation.
\end{assumption}}
\vspace{-15pt}
{\begin{assumption}[Change in $\srank_\delta$]
\label{assumption:c_k}
For every fitting iteration $k$, we assume that the the difference between $\srank_\delta(\bW^P_\phi(k))$, equal to the effective rank to feature matrix obtained when $\bR + \gamma P^\pi \bQ_{k-1}$ is expressed in the function class, and $\srank_\delta(\bW_\phi(k-1))$ by an iteration specific threshold $c_k$, \ie
\begin{equation*}
   \srank_\delta(\bW^P_\phi(k)) \leq \srank_\delta(\bW_\phi(k-1)) + c_k.
\end{equation*}
\end{assumption}
We will characterize the properties of $c_k$ in special cases by showing how Theorem~\ref{thm:near_optimality} is a special case of our next argument, where we can analyze the trend in $c_k$.}

{To first see how a change in rank indicated by $c_k$ in Assumption~\ref{assumption:c_k} propagates through bootstrapping, note that we can use Assumption~\ref{assumption:c_k} and \ref{assumption:closed} to observe the following relation:
\begin{align}
    \srank_\delta(\bW_\phi(k)) ~&\leq \srank_\delta(\bW^P_\phi(k)) - \frac{||\bQ_k - \by_k||}{\lambda_k}~~~ \text{(since we can copy over the weights)} \nonumber\\
    &\leq \srank_\delta(\bW_\phi(k-1)) + c_k - \frac{||\bQ_k - \by_k||}{\lambda_k}~~~~~~~ \text{(using Assumption~\ref{assumption:c_k})} \nonumber\\
    & \leq \srank_\delta (\bW_\phi(k-2)) + c_{k-1} - \frac{||\bQ_{k-1} - \by_{k-1}||}{\lambda_{k-1}} + c_k - \frac{||\bQ_k - \by_k||}{\lambda_k} \nonumber\\
    \srank_\delta(\bW_\phi(k)) & \leq \boxed{\srank_\delta (\bW_\phi(0)) + \sum_{j=1}^k c_j - \sum_{j=1}^k \frac{||\bQ_j - \by_j||}{\lambda_j}}
    \label{eqn:final_compounding}
\end{align}
\textbf{Derivation of the steps:} The first inequality holds via the weight-copying argument: since the weights for the target value $\bR + \gamma P^\pi \bQ_{k-1}$ can be directly copied to obtain a zero TD error feasible point with an equal $\srank$, we will obtain a better solution at the optimum with a smaller $\srank$ in Equation~\ref{eqn:closed_form}. We then use Assumption~\ref{assumption:c_k} to relate $\srank_\delta(\bW^P_\phi(k))$ to the effective rank of the previous Q-function, $\srank_\delta(\bW_\phi(k-1))$, which gives us a recursive inequality. Finally, the last two steps follow from a repeated application of the recursive inequality in the second step. This proves the result shown in Proposition~\ref{thm:compounding_main_paper}.}

{Now if the decrease in effective rank due to accumulating TD errors is larger than the possible cumulative increase in rank $c_k$, then we observe a rank drop. Also note that $c_k$ need not be positive, $c_k$ can be negative, in which case both terms contribute to a drop in rank. But in the most general case, $c_k$, can also be positive for some iterations. This equation indicates how a change in in rank in one iteration gets compounded on the next iteration due to bootstrapping.}

\begin{wrapfigure}{r}{0.36\linewidth}
\centering
\vspace{-17pt}
\includegraphics[width=0.85\linewidth]{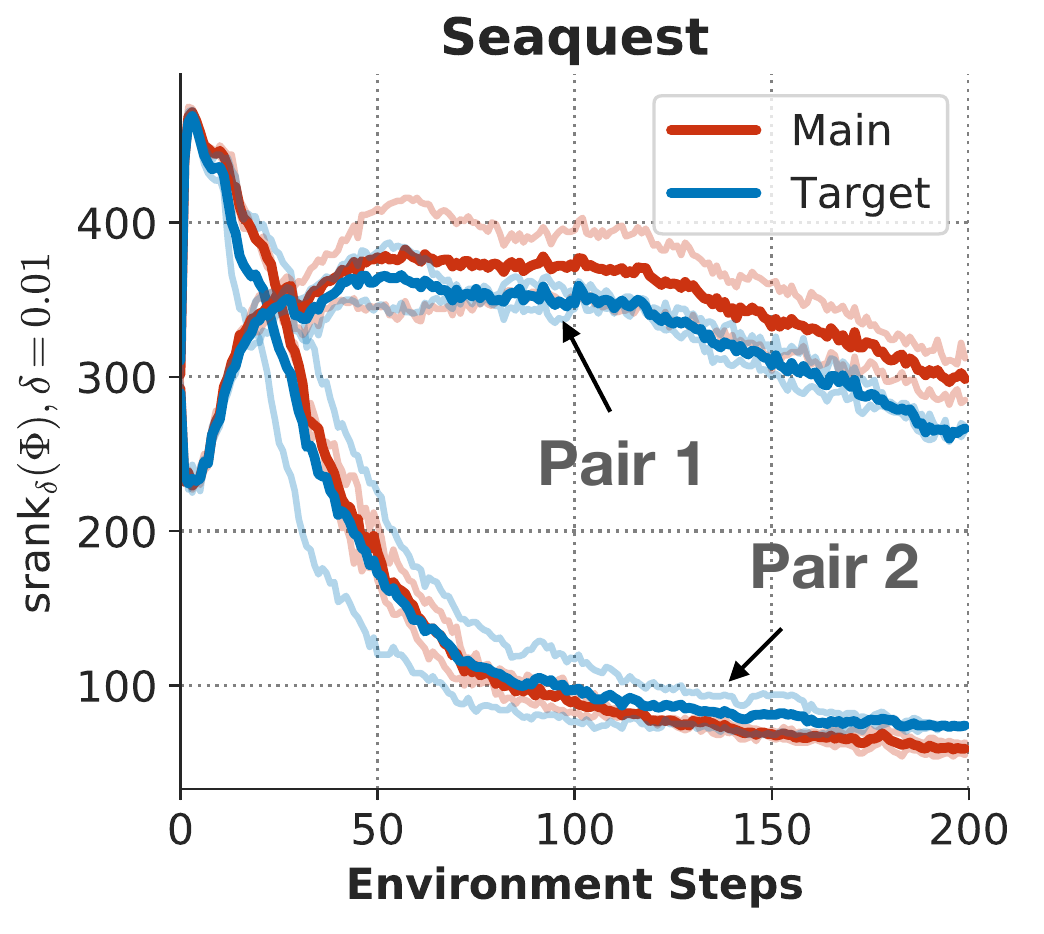}
\vspace{-10pt}
\caption{{Two pairs of runs depicting the trends of target feature (approximately equal to $\srank_\delta(\bW^P_\phi(k))$) (denoted as ``Target'') and $\srank_\delta(\bW_\phi(k-1))$ (denoted as ``Main'') on Seaquest with in the data-efficient online RL setting. Note that the contribution of dynamics transformation in the value of $c_k$ is generally negative for the Pair 1 and is small/positive for Pair 2.}}
\label{fig:target_rank}
\vspace{-25pt}
\end{wrapfigure}
\textbf{What can we comment about the general value of $c_k$?} In the general case, $c_k$ can be positive, for example when the addition of the reward and dynamics projection increases $\srank$. Our current set of assumptions are insufficient to guarantee something about $c_k$ in the general case and more assumptions are needed. For instance, as shown in Theorem~\ref{thm:near_optimality}, under Assumption~\ref{assumption:closed}, we obtain that if we are in a small neighborhood around the fixed point, we will find that $c_k \leq 0$.

\textbf{Empirical verification:} While our analysis does not dictate the value of $c_k$, we performed an experiment on Seaquest to approximately visualize $\srank_\delta(\bW^P_\phi(k))$ by computing the effective rank for the feature matrix $\Phi$ obtained when $P^\pi \bQ_{k-1}$ is expressed in the Q-function class. We do not include the reward function in this experiment, but we believe that a ``simple'' reward function, \ie it takes only three possible values in Atari, and hence it may not contribute too much to $\srank_\delta(\bW^P_\phi(k))$.
As shown in Figure~\ref{fig:target_rank}, the value of the target value feature effective rank decreases, indicating that the coefficient $c_k$ is expected to be bounded in practice.

\subsection{Proof for Theorem~\ref{thm:near_optimality}: Rank Collapse Near a Fixed Point}
\label{app:proof_subopt}
Now, we will prove Theorem~\ref{thm:near_optimality} by showing that when the current Q-function iterate $\bQ_k$ is close to a fixed point but have not converged yet, \ie, when $||\bQ_k - (\bR + \gamma P^\pi \bQ_k)|| \leq \mathbf{\varepsilon}$, then rank-decrease happens. {We will prove this as a special case of the discussion in Section~\ref{app:rank_drop_compounds}, by showing that for any iteration $k$, when the current Q-function iterate $\bQ_k$ is close to a fixed point, the value of $c_k$ can be made close to 0. While the compounding argument (Equation~\ref{eqn:final_compounding}) may not directly hold here since we cannot consistently comment on whether making an update on the Q-function pushes it out of the $\varepsilon$-ball near the fixed point, so we instead prove a ``one-step'' result here.}  

To prove Theorem~\ref{thm:near_optimality}, we evaluate the (infinitesimal) change in the singular values of the features of $\bW_\phi(k, t)$ as a result of an addition in the value of $\varepsilon$. In this case, the change (or differential) in the singular value matrix, $\bS(k, T)$, ($\bW_\phi(k, T) = \bU(k, T) \bS(k, T) \bV(k, T)^T$) is given by:
\begin{equation}
    d\bS(k, T) = \bI_{d} \circ \Big[ \bU(k, T)^T \cdot d \bW_\phi(k, t) \cdot \bV(k, T) \Big], 
\end{equation}
using results on computing the derivative of the singular value decomposition~\citep{townsend2016differentiating}. 

\textbf{Proof strategy:} Our high-level strategy in this proof is to design a form for $\varepsilon$ such that the value of effective rank for feature matrix obtained when the resulting target value $\by_{k} = \bQ_{k-1} + \varepsilon$ is expressed in the deep linear network function class, let's say with feature matrix, $\bW_\phi'(k, T)$ and the last layer weights, $\bW_N(k, T) = \bW_N(k-1, T)$, then the value of $\srank_\delta(\bW'_\phi(k, T))$ can be larger than $\srank_\delta(\bW_\phi(k-1, T))$ by only a bounded limited amount $\alpha$ that depends on $\varepsilon$. More formally, we desire a result of the form 
\begin{align}
    \srank_\delta(\bW'_\phi(k, T)) &\leq \srank_\delta(\bW_\phi(k-1, T)) + \alpha, \nonumber \\ ~~&\text{where}~~ ||\epsilon|| \ll ||\bQ_{k-1}||, ~~\text{and}~~ \by_{k}(\bs, \ba) = \bW_N(k-1, T) \bW'_\phi(k, T) [\bs; \ba]. \label{eqn:desired_result}
\end{align}
Once we obtain such a parameterization of $\by_k$ in the linear network function class, using the argument in Section~\ref{sec:grad_descent} about self-training, we can say that just copying over the weights $\bW_\phi'(k, T)$ is sufficient to obtain a bounded increase in $\srank_\delta(\bW_\phi(k, T))$ at the cost of incurring zero TD error (since the targets can be exactly fit). As a result, the optimal solution found by Equation~\ref{eqn:closed_form} will attain lower $\srank_\delta(W_\phi(k, T))$ value if the TD error is non-zero. Specifically,
\begin{equation}
    \srank_\delta(\bW'_\phi(k, T)) + \frac{||Q_{k} - \by_{k}||}{\lambda_k} \le \srank_\delta(\bW_\phi(k-1, T)) + \alpha
\end{equation}

As a result we need to examine the factors that affect $\srank_\delta(W_\phi(k, T))$ -- \textbf{(1)} an increase in $\srank_\delta(\bW_\phi(k, T))$ due to an addition of $\alpha$ to the rank, and \textbf{(2)} a decrease in $\srank_\delta(\bW_\phi(k, T))$ due to the implicit behavior of gradient descent.
 
\textbf{Proof:} In order to obtain the desired result (Equation~\ref{eqn:desired_result}), we can express $\varepsilon(\bs, \ba)$ as a function of the last layer weight matrix of the \textit{current} Q-function, $\bW_N(k-1, T)$ and the state-action inputs, $[\bs; \ba]$. Formally, we assume that $\varepsilon(\bs, \ba) = \bW_N(k-1, T) {\zeta [\bs; \ba]}$, where $\zeta$ is a matrix {of the same size as $\bW_\phi$} such that $||\zeta||_\infty \ll \tau = \mathcal{O}(||\bW_\phi(k, T)||_\infty)$, \ie $\zeta$ has entries with ignorable magnitude compared the actual feature matrix, $\bW_\phi(k, T)$. {This can be computed out of the closure assumption (Assumption~\ref{assumption:closed}) from Section~\ref{app:rank_drop_compounds}.}

Using this form of $\varepsilon(\bs, \ba)$, we note that the targets $\by_k$ used for the backup can be written as
\begin{align*}
    \by_{k} = \bQ_{k-1} + \varepsilon ~&= \bW_N(k-1, T) \bW_\phi(-1, T) [\states; \actions] + \bW_N(k-1, T) \zeta [\states; \actions]\\
    &= \bW_N(k-1, T) \cdot \underbrace{\left( \bW_\phi(k-1, T) + \zeta \right)}_{\bW'_\phi(k, T)} [\states; \actions]
\end{align*}
Using the equation for sensitivity of singular values due to a change in matrix entries, we obtain the maximum change in the singular values of the resulting ``effective'' feature matrix of the targets $\by_k$ in the linear function class, denoted as $\bW'_\phi(k, T)$, is bounded: 
\begin{equation}
    d \bS'(k, T) = \bI_d \circ \Big[ \bU(k, T) \cdot d \bW_\phi(k, T) \cdot \bV(k, T)^T \Big] ~~ \implies \left\vert\left\vert d \bS'(k, T) \right\vert\right\vert_\infty  \leq \zeta.
\end{equation}
Now, let's use this inequality to find out the maximum change in the function, $h_k(i)$ used to compute $\srank_\delta(\bW_\phi)$: $h_k(i) = \frac{\sum_{j=1}^i \sigma_j(\bW_\phi(k, T))}{\sum_{j=1}^N \sigma_j(\bW_\phi(k, T))}$  ~~~\mbox{$(\srank_\delta(\bW_\phi(k, T)) = \min\left\{j~: h_k(j) \geq 1 - \delta \right\})$}.
\begin{align*}
    |d h_k(i)| ~&= \left\vert\frac{\sum_{j=1}^i d \sigma_j(\bW_\phi)}{\sum_{j=1}^N \sigma_j(\bW_\phi)} - \left( \frac{\sum_{j=1}^i \sigma_j(\bW_\phi)}{\sum_{j=1}^N \sigma_j(\bW_\phi)} \right) \left( \frac{\sum_{j=1}^N  d \sigma_j(\bW_\phi)}{\sum_{j=1}^N \sigma_j(\bW_\phi)} \right) \right\vert\\
     ~&\leq \frac{i \zeta}{\sum_{j=1}^N \sigma_j(\bW_\phi)} + \underbrace{\left( \frac{\sum_{j=1}^i \sigma_j(\bW_\phi)}{\sum_{j=1}^N \sigma_j(\bW_\phi)} \right)}_{\leq 1} \frac{N \zeta}{\sum_{j=1}^N \sigma_j(\bW_\phi)} \\
     &\leq \frac{(i + N) \zeta}{\sum_{j=1}^N \sigma_j(\bW_\phi)}.
\end{align*}
The above equation implies that the maximal change in the effective rank of the feature matrix generated by the targets, $\by_k$ (denoted as $\bW_\phi'(k, T)$) and the effective rank of the features of the current Q-function $\bQ_k$ (denoted as $\bW_\phi(k, T)$) are given by:
\begin{align*}
    \srank_\delta(\bW'_\phi(k, T)) - \srank_\delta(\bW_\phi(k, T)) ~&\leq \alpha 
\end{align*}
where $\alpha$ can be formally written by the cardinality of the set:
\begin{align}
    \alpha = \left\vert \left\{ j:~~ h_k(j) - \frac{(j+N) \zeta}{\sum_{j=1}^N \sigma_j(\bW_\phi)} \geq (1 - \delta), ~~ h_k(j) \leq (1 - \delta) \right\} \right\vert.
\end{align}
Note that by choosing $\varepsilon(\bs, \ba)$ and thus $\zeta$ to be small enough, we can obtain $\bW'_\phi(k, T)$ such that $\alpha = 0$. Now, the self-training argument discussed above applies and gradient descent in the next iteration will give us solutions that reduce rank.

Assuming $r > 1$ and $\srank_\delta(\bW_\phi(k, T)) = r$, we know that $h_k(r-1) < 1 - \delta$, while $h_k(r) \ge 1 - \delta$. Thus, $\srank_\delta(\bW_\phi^\prime(k, T))$ to be equal to $r$, it is sufficient to show that $h_k(r) - |dh_k(r)| \ge 1 - \delta$ since both $h_k(i)$ and $dh_k(i)$ are increasing for $i=0, 1, \cdots, r$. Thus, $\srank_\delta(\bW_\phi^\prime(k, T)) = r \, \forall \zeta$, whenever
\begin{align*}
    \zeta &\le  \frac{\sum_{j=1}^{N} \sigma_j(\bW_\phi)}{r + N} \Big(h_k(r) - 1 - \delta \Big) \\
   & = \frac{\sum_{j=1}^{r} \sigma_j(\bW_\phi) -  (1 - \delta)\sum_{j=1}^{N} \sigma_j(\bW_\phi)}{r + N}
\end{align*}

This implies that $\boxed{\varepsilon \le ||W_{N}(k,T)||_\infty\frac{\sum_{j=1}^{r} \sigma_j(\bW_\phi) -  (1 - \delta)\sum_{j=1}^{N} \sigma_j(\bW_\phi)}{r + N}}$.

\textbf{Proof summary:} We have thus shown that there exists a neighborhood around the optimal fixed point of the Bellman equation, parameterized by $\varepsilon(\bs, \ba)$ where bootstrapping behaves like self-training. In this case, it is possible to reduce $\srank$ while the TD error is non-zero. And of course, this would give rise to a rank reduction close to the optimal solution. 

\end{document}